\documentclass[10pt]{article} 
\usepackage[preprint]{tmlr}

\usepackage{amsmath,amsfonts,bm}

\def\eqref#1{equation~\ref{#1}}

\def\1{\bm{1}}

\def\vg{{\bm{g}}}

\def\vw{{\bm{w}}}
\def\vx{{\bm{x}}}

\def\mX{{\bm{X}}}

\DeclareMathAlphabet{\mathsfit}{\encodingdefault}{\sfdefault}{m}{sl}
\SetMathAlphabet{\mathsfit}{bold}{\encodingdefault}{\sfdefault}{bx}{n}

\def\gL{{\mathcal{L}}}

\def\gR{{\mathcal{R}}}

\def\gW{{\mathcal{W}}}
\def\gX{{\mathcal{X}}}
\def\gY{{\mathcal{Y}}}

\newcommand{\E}{\mathbb{E}}

\usepackage{hyperref}
\usepackage{url}

\renewcommand{\bm}[1]{\boldsymbol{#1}}
\newcommand{\mc}[1]{\mathcal{#1}}
\newcommand{\mbb}[1]{\mathbb{#1}}

\usepackage{amssymb}
\usepackage{algorithm}
\usepackage[noend]{algpseudocode}
\usepackage{etoolbox}
\usepackage{amsthm}
\usepackage{multirow}
\usepackage{graphicx}
\usepackage{booktabs}
\usepackage{accessibility}
\usepackage{subcaption}
\usepackage{graphicx}
\usepackage{booktabs}
\usepackage{caption}

\theoremstyle{plain}
\newtheorem{theorem}{Theorem}[section]
\newtheorem{lemma}[theorem]{Lemma}

\newtheorem{assumption}[theorem]{Assumption}

\theoremstyle{remark}
\newtheorem{remark}[theorem]{Remark}

\usepackage{xcolor}
\definecolor{darkblue}{RGB}{25, 25, 112}
\definecolor{accent}{RGB}{255, 87, 51}
\definecolor{subtle}{RGB}{117, 117, 117}

\title{Low-rank Momentum Factorization for Memory Efficient Training}

\author{\name Pouria Mahdavinia \email pxm5426@psu.edu \\
      \addr Department of Computer Science and Engineering\\
      The Pennsylvania State University 
      \AND
      \name Mehrdad Mahdavi \email mzm616@psu.edu \\
      \addr Department of Computer Science and Engineering\\
      The Pennsylvania State University }

\def\month{07}  
\def\year{2025} 
\def\openreview{\url{https://openreview.net/forum?id=W3D3TVo9a3}} 

\begin{document}

\maketitle
\begin{abstract}
Fine-tuning large foundation models presents significant memory challenges due to stateful optimizers like AdamW, often requiring several times more GPU memory than inference. While memory-efficient methods like parameter-efficient fine-tuning (e.g., LoRA) and optimizer state compression exist, recent approaches like GaLore bridge these by using low-rank gradient projections and subspace moment accumulation. However, such methods may struggle with fixed subspaces or computationally costly offline resampling (e.g., requiring full-matrix SVDs). We propose Momentum Factorized SGD (MoFaSGD), which maintains a dynamically updated low-rank SVD representation of the first-order momentum, closely approximating its full-rank counterpart throughout training. This factorization enables a memory-efficient fine-tuning method that adaptively updates the optimization subspace at each iteration. Crucially, MoFaSGD leverages the computed low-rank momentum factors to perform efficient spectrally normalized updates, offering an alternative to subspace moment accumulation. We establish theoretical convergence guarantees for MoFaSGD, proving it achieves an optimal rate for non-convex stochastic optimization under standard assumptions. Empirically, we demonstrate MoFaSGD's effectiveness on large language model alignment benchmarks, achieving a competitive trade-off between memory reduction (comparable to LoRA) and performance compared to state-of-the-art low-rank optimization methods. Our implementation is available at \url{https://github.com/pmahdavi/MoFaSGD}.
\end{abstract}
\section{Introduction}
Advancements in AI have been propelled by initial scaling laws, which boost pre-training performance via ever-larger models and datasets \citep{kaplan2020scaling}. However, adapting these large foundation models to downstream tasks, such as instruction or preference tuning \citep{ouyang2022training}, often requires several times more GPU memory than inference. This memory burden largely stems from storing optimizer states (e.g., momentum terms) required by ubiquitous methods like AdamW. Various strategies have emerged to alleviate these costs.

Parameter-Efficient Fine-Tuning (PEFT) methods provide an architectural solution, introducing a small set of new, trainable parameters called adapters. For example, the popular LoRA technique \citep{hu2021lora} restricts model updates to these low-rank adapters, drastically reducing memory overhead by training only a small fraction of parameters. Variants such as DoRA \citep{liu2024dora}, AdaLoRA \citep{zhang2023adalora}, VeRA \citep{kopiczko2023vera} and rsLoRA \citep{kalajdzievski2023rank} have further refined this approach to low-rank parameterization. Other approaches focus on compressing optimizer states directly; methods like AdaFactor \citep{shazeer2018adafactor}, SM3 \citep{anil2019memory}, and quantization techniques \citep{microadam_modoranu_2024,feinberg2024sketchy} factorize or reduce the precision of moments. Alternatively, stateless methods like signSGD \citep{bernstein2018signsgd} and SWAN \citep{swan_ma_2024} avoid momentum accumulation by carefully normalizing gradients each iteration. 

A complementary line of research, \textbf{low-rank subspace optimization}, projects gradients onto a smaller, dynamic subspace to reduce memory. This allows full-parameter updates while operating in a lower-dimensional space. Key examples include GaLore \citep{zhao2024galore}, Flora \citep{hao2024flora}, and ReLoRA \citep{lialin2023relora}. However, applying these dynamic subspace techniques effectively presents difficulties.

\textbf{Challenges in Online Subspace Updates.} While promising, dynamically updating the optimization subspace online, as in GaLore, faces key obstacles. First, managing optimizer states (particularly momentum) during abrupt subspace changes is non-trivial; existing methods may reset states \citep{lialin2023relora}, transform them \citep{hao2024flora}, or leave them unchanged \citep{hao2024flora}. Second, determining the new subspace, often based on gradients, can incur high computational costs (e.g., full SVD on large gradient matrices \citep{zhao2024galore}), hindering scalability. These limitations highlight the need for more efficient, iteration-level subspace adaptation strategies.

\subsubsection*{Contributions} Motivated by the success of gradient-based techniques for memory-efficient fine-tuning, the advantages of dynamic projection subspaces over static ones, and the potential for online subspace updates to better track full-rank optimization dynamics, we address three research questions:
\begin{itemize}
    \item[\textbf{(Q1)}] Can we mitigate the challenges of online subspace resampling through a computationally efficient update strategy while rigorously tracking projection residuals?
    \item[\textbf{(Q2)}] Can we directly estimate full-rank adaptive optimizer trajectories using low-rank factors, instead of relying on compressed accumulation of subspace moments?
    \item[\textbf{(Q3)}] Can we match the theoretical per-iteration complexity of standard adaptive optimizers while maintaining LoRA-level memory savings?
\end{itemize}
We answer these questions affirmatively by introducing MoFaSGD, which maintains a dynamically updated low-rank SVD factorization of the first-order momentum. We define this factorization as:
\begin{equation}
\label{eq:momentum_factorization}
\hat{\bm M}_{t} = \bm U_{t+1} \operatorname{Diag}( \bm \sigma_{t+1}) \bm V_{t+1}^{\top} \approx \beta \hat{\bm M}_{t-1} + (1-\beta) \bm G_{t} \quad \text{(First-Momentum Factorization)}
\end{equation}
Here, $\hat{\bm{M}}_{t}$ denotes the approximate first-order momentum at iteration $t$, $\bm{G}_t$ is the gradient at iteration $t$, and $\beta$ is the momentum decay factor. The low-rank SVD factors $\bm{U}_{t+1} \in \mathbb{R}^{m \times r}$ (left singular vectors), $\bm{\sigma}_{t+1} \in \mathbb{R}^{r}$ (singular values), and $\bm{V}_{t+1} \in \mathbb{R}^{n \times r}$ (right singular vectors) serve as the optimizer state in our algorithm.

Motivated by our observation that the exponential moving average (EMA) of gradients, defined as $\sum_{i=1}^t \beta^{t-i} \bm G_i$, often exhibit low-rank structure, MoFaSGD applies a specific low-rank approximation strategy: maintaining an online, low-rank SVD factorization of the \emph{first-order momentum}. This differs from state compression methods that typically target second moments \citep{shazeer2018adafactor} and avoids the complex factored approximations of curvature used in methods like Shampoo or KFAC \citep{gupta2018shampoo, martens2015optimizing}. The factorization is updated efficiently at each iteration using tangent space projections, bypassing costly offline resampling \citep{zhao2024galore}. Crucially, MoFaSGD leverages the \emph{same} continuously updated low-rank momentum factors ($\bm U_{t+1}, \bm V_{t+1}$) to perform spectral normalization ($\bm W_{t+1} \leftarrow \bm W_t - \eta \bm U_{t+1} \bm V_{t+1}^\top$). This integrated approach, inspired by methods like Shampoo \citep{gupta2018shampoo} and Muon \citep{jordan2024muon}, provides adaptive step directions without requiring separate, computationally intensive matrix operations (e.g., SVD, matrix roots, or Newton-Schulz iterations) at each step.

MoFaSGD achieves LoRA-like memory savings through a unified low-rank momentum representation, while also enabling efficient, adaptive full-parameter updates via spectral normalization. This positions our method as a low-rank variant of Muon. Unlike MoFaSGD, Muon maintains full-rank momentum buffers ($\mathcal{O}(mn)$ memory) and applies full-rank updates after normalization.

\section{Related Work}

To improve memory efficiency for training large-scale neural networks, various strategies have emerged, including parameter-efficient fine-tuning, subspace optimization, optimizer state compression, stateless approaches, and partial update techniques. Moreover, second-order and preconditioning methods such as Shampoo~\citep{gupta2018shampoo} have been gaining popularity due to their faster convergence compared to AdamW, albeit with even higher memory requirements, as studied in detail in~\citet{kasimbeg2025accelerating}.

Below, we detail relevant methods like subspace optimization, optimizer state compression, stateless approaches, and gradient spectral normalization; other approaches are deferred to Appendix ~\ref{sec:app-rw}.

\subsubsection*{Subspace Optimization} Rather than introducing additional parameters through PEFT adapters, subspace optimization methods focus on reducing the memory footprint of gradients and optimizer states directly by projecting gradients onto a low-rank subspace and performing moment accumulation in the subspace as well \citep{zhao2024galore,gur2018gradient, gressmann2020improving, yang2023spectral, vogels2020practical}. GaLore \citep{zhao2024galore} projects gradients onto a subspace defined by either the left or right singular vectors of the gradients, computed via SVD, and accumulates optimizer states within this subspace, thereby reducing the required optimizer memory. Flora \citep{hao2024flora} periodically resamples the subspace projection matrix using a multivariate Gaussian distribution for low-rank gradient projections. 

AdaRankGrad \citep{refael2024adarankgrad} dynamically adjusts gradient rank during training. LDAdam \citep{ldadam_robert_2024} and APOLLO \citep{zhu2024apollo} maintain optimizer states in low-dimensional representations, adopting similar strategy to subspace moment accumulation as in GaLore~\citep{zhao2024galore}. Subspace management is a key challenge. Computing it can require costly operations like full SVD~\citep{zhao2024galore}, and its update frequency creates a trade-off: infrequent updates can lead to stale information, while frequent updates can disrupt optimizer state accumulation. Managing optimizer states across subspace changes also requires careful strategies \citep{lialin2023relora, hao2024flora}.

\noindent\textbf{Optimizer State Compression.} This approach directly reduces the memory footprint of existing optimizer states, primarily the second moments used in adaptive methods. AdaFactor \citep{shazeer2018adafactor} and SM3 \citep{anil2019memory} achieve this through factorization techniques, reducing memory complexity from $\mathcal{O}(mn)$ to $\mathcal{O}(m+n)$. Quantization methods like 8-bit Adam \citep{dettmers20218} and 4-bit variants \citep{li2023memory} reduce the precision of stored moments, achieving near full-precision performance with significant memory savings. CAME \citep{luo2023came} and Adapprox \citep{zhao2024adapprox} refine these compression techniques for potentially better accuracy. 

\noindent\textbf{Stateless and Gradient Normalization Methods.} 
Eliminating optimizer states entirely provides maximal memory efficiency. signSGD~\citep{bernstein2018signsgd} achieves this by using only the sign of the gradient. Lion~\citep{chen2024symbolic} similarly reduces the memory footprint of AdamW by discarding the second-moment estimate. Gradient normalization methods such as SWAN~\citep{swan_ma_2024} and Muon~\citep{jordan2024muon} emulate the behavior of adaptive optimizers without storing second-moment statistics. Muon, in particular, applies momentum followed by approximate orthogonalization of the momentum matrix ($\bm{U}_{\bm{M}}\bm{V}_{\bm{M}}^\top$, where $\bm{M} = \bm{U}_{\bm{M}} \bm{\Sigma}_{\bm{M}} \bm{V}_{\bm{M}}^\top$) using efficient Newton-Schulz iterations~\citep{jordan2024muon}, drawing connections to accumulation-free variants of Shampoo~\citep{jordan2024muon, old_bernstein_2024}. However, Muon still retains the Nesterov momentum, and thus is not considered fully stateless.

\noindent\textbf{Positioning MoFaSGD.} The landscape of memory-efficient optimization presents a spectrum of trade-offs. Architectural-based solutions such as PEFT methods offer high memory savings but restrict fine-tuning updates~\citep{hu2021lora} to a low-rank space. Subspace methods enable full-rank parameter updates but encounter challenges such as costly or potentially disruptive subspace resampling and moment accumulation~\citep{zhao2024galore, hao2024flora}; moreover, subspace techniques are mostly applied to AdamW, while their application to non-diagonal-based preconditioning methods (such as Shampoo and its variants~\citep{gupta2018shampoo,george2018fast,jordan2024muon}) remains largely unexplored. Optimizer state compression techniques focus on reducing the storage cost of second moments~\citep{shazeer2018adafactor,dettmers20218}. Stateless methods are memory-efficient but discard historical information~\citep{bernstein2018signsgd,swan_ma_2024}. This work introduces MoFaSGD to navigate these challenges.
\section{Preliminaries}
\label{sec:prelim}
This section introduces notation and provides background on subspace and adaptive optimization methods.

\noindent\textbf{Notation and Conventions.}
We denote matrices with bold capital letters (e.g., $\mX$), vectors with bold lowercase letters (e.g., $\vx$), and scalars with non-bold letters (e.g., $x$). The Frobenius inner product for matrices is denoted as $\langle \cdot, \cdot \rangle$ (e.g., $\langle \bm{X}, \bm{Y} \rangle = \operatorname{Tr}(\bm{X}^{\top} \bm{Y})$). Norms are denoted by $\|\cdot\|$, with specific norms indicated using subscripts (e.g., $\|\cdot\|_{\mathrm{F}}$ for the Frobenius norm). The loss function is represented by $\mc{L}(\cdot)$. Without loss of generality, we denote the model parameters as $\bm{W}_t  \in \mbb{R}^{m \times n}$, and the full-rank gradient at $\bm{W}_t$ is denoted as $\bm{G}_t = \frac{\partial \mc{L}(\bm{W})}{\partial \bm{W}}|_{\bm W_t} \in \mbb{R}^{m \times n}$. The subscript index $t$ denotes the iteration step of any optimizer for each weight matrix. Gradients are defined based on the specific batches used in the corresponding iteration. To simplify notation, we use the same symbols as previously defined where applicable. For any matrix $\bm{X} \in \mbb{R}^{m_1 \times m_2}$, the vectorization operator $\operatorname{Vec}(\cdot)$ stacks its columns into a single vector, denoted as $\bm{x} = \operatorname{Vec}(\bm{X}) \in \mbb{R}^{m_1 m_2 \times 1}$. The vectorized version of a matrix is always denoted by the corresponding bold lowercase letter. The Kronecker product is denoted as $\otimes$, such that for $\bm{X} \in \mbb{R}^{m_1 \times n_1}$ and $\bm{Y} \in \mbb{R}^{m_2 \times n_2}$, the Kronecker product $\bm{X} \otimes \bm{Y} \in \mbb{R}^{m_1 m_2 \times n_1 n_2}$. The symbol $\odot$ denotes the element-wise (Hadamard) product between two matrices of the same dimensions. See Appendix~\ref{appx:note} for further notational details.

\subsection*{Adaptive methods with Switched-off Momentums}
Without first-order momentum, adaptive methods can be framed as a form of online mirror descent:
\begin{equation}
\label{eqn:omd}
\begin{split}
    \bm w_{t+1} = \underset{\bm w}{\arg \min}  \; \mc L(\bm w_t) + \langle \bm w - \bm w_t , \bm g_t \rangle + \frac{1}{2 \eta} \| \bm w - \bm w_t\|_{\bm P_t} = \bm w_t - \eta \bm P_t^{-1} \bm g_t
\end{split}
\end{equation}
where the preconditioner $\bm P_t$ can be any positive semidefinite matrix, and $\mc L(\bm w_t) + \langle \bm w_{t+1} - \bm w_t , \bm g_t \rangle$ is the local approximation of the loss function. For instance, if we let $\bm P_{t,\text{AdaGrad}} =  \left [ \sum_{i=1}^t \bm g_i \bm g_i^{\top}  \right ]^{\frac{1}{2}}$, we exactly recover full-matrix Adagrad~\citep{Duchi2011Adagrad}. 

Similarly, if we define $\bm L_t = \sum_{i=1}^t \bm G_i \bm G_i^{\top}$ and $ \bm R_t = \sum_{i=1}^t \bm G_i^{\top}  \bm G_i$,  then letting $\bm P_{t,\text{Shampoo}} = (\bm R_t \otimes \bm L_t )^{\frac{1}{4}}$ recovers the Shampoo updates for matrices~\citep{gupta2018shampoo}. Furthermore, if we consider that $\bm g_t \bm g_t^{\top}$ closely approximates the Gauss-Newton components of the true Hessian~\citep{morwani2024new}, this motivates switching off covariance momentum (second moment) and considering the following preconditioners based on Adagrad and Shampoo: $\bm P_{t,1} = \left[ \bm g_t \bm g_t^{\top} \right]^{\frac{1}{2}}$ and $\bm P_{t,2} = ( \bm G_t^{\top} \bm G_t \otimes \bm G_t \bm G_t^{\top})^{\frac{1}{4}}$. 

A simple derivation shows that using the diagonal version of $\bm P_{t,1}$ in Equation~\ref{eqn:omd} recovers Signed-SGD~\citep{bernstein2018signsgd}. Moreover, if we consider the reduced SVD of the gradient $\bm G_t = \bm U_{\bm G_t} \bm \Sigma_{\bm G_t} \bm V_{\bm G_t}^{\top}$, then we can write $\bm P_{t,2} = (\bm V_{\bm G_t} \bm \Sigma_{\bm G_t}^2 \bm V_{\bm G_t}^{\top} \otimes \bm U_{\bm G_t} \bm \Sigma_{\bm G_t}^2 \bm U_{\bm G_t}^{\top})^{\frac{1}{4}} $. Using Lemma~\ref{lemma:kron-base} in the appendix, we have:
\begin{equation}
\begin{split}
     \bm P_{t,2} = (\bm V_{\bm G_t} \bm \Sigma_{\bm G_t}^2 \bm V_{\bm G_t}^{\top} \otimes \bm U_{\bm G_t} \bm  \Sigma_{\bm G_t}^2 \bm U_{\bm G_t}^{\top})^{\frac{1}{4}} = (\bm V_{\bm G_t} \otimes \bm U_{\bm G_t}) (\bm \Sigma_{\bm G_t}^{\frac{1}{2}} \otimes  \bm \Sigma_{\bm G_t}^{\frac{1}{2}})(\bm V_{\bm G_t} \otimes \bm U_{\bm G_t})^{\top}
 \end{split}
\end{equation}
Furthermore, we can write $\bm g_t = \operatorname{Vec}(\bm G_t) = (\bm V_{\bm G_t} \otimes \bm U_{\bm G_t}) \operatorname{Vec}(\bm \Sigma_{\bm G_t})$. Substituting the preconditioner $\bm P_{t,2}$ and the vectorized gradient into Equation~\ref{eqn:omd} yields:
\begin{equation}
\begin{split}
    \bm w_{t+1} &= \bm w_t - \eta  (\bm V_{\bm G_t} \otimes \bm U_{\bm G_t}) (\bm \Sigma_{\bm G_t}^{-\frac{1}{2}} \otimes  \bm \Sigma_{\bm G_t}^{-\frac{1}{2}})(\bm V_{\bm G_t} \otimes \bm U_{\bm G_t})^{\top} \bm g_t \\
    &=  \bm w_t  - \eta (\bm V_{\bm G_t} \otimes \bm U_{\bm G_t}) (\bm \Sigma_{\bm G_t}^{-\frac{1}{2}} \otimes  \bm \Sigma_{\bm G_t}^{-\frac{1}{2}}) \operatorname{Vec}(\bm \Sigma_{\bm G_t}) \\
    &= \bm w_t - \eta (\bm V_{\bm G_t} \otimes \bm U_{\bm G_t}) \operatorname{Vec}(\bm I_r) = \bm w_t - \eta \bm U_{\bm G_t} \bm V_{\bm G_t}^\top
\end{split}
\end{equation}
Thus, using $\bm P_{t,2}$ recovers spectrally normalized updates, sometimes referred to as gradient whitening~\citep{jordan2024muon,swan_ma_2024}. For a similar derivation, see~\citet{old_bernstein_2024}. Additional background on the historical development of adaptive optimization methods is provided in Appendix~\ref{sec:add-pre}.

\subsection*{Subspace Optimization Methods} Many studies support the conjecture that gradients during the training of deep neural networks exhibit a low-rank structure, lying in a low-dimensional subspace~\citep{gur2018gradient,gressmann2020improving,yang2023spectral}. This property has been studied both theoretically and empirically and has been leveraged to improve optimization algorithms ranging from communication-efficient distributed training~\citep{vogels2020practical} to efficient fine-tuning~\citep{hu2021lora,lialin2023relora,hao2024flora,zhao2024galore}. \citet{zhao2024galore} aim to explicitly leverage this property to perform subspace training, where gradients are projected and accumulated in a low-rank subspace, leading to significant memory footprint reduction. 

We briefly clarify subspace methods mathematically, following the formulation of GaLore~\citep{zhao2024galore}. Let $\bm W_t \in \mbb R^{m \times n}$ represent the model parameters, and let $\bm Q_t \in \mbb R^{m \times r} $ be the projection matrix defining the subspace at iteration $t$. Then, GaLore performs the following update:
\begin{equation}
\begin{array}{rlrl}
\bm G_{t,r} &= \bm Q_t^{\top} \bm G_t \in \mathbb{R}^{r \times n} & \quad \text{(Subspace projection)} \\
\bm M_{t,r} &= \beta_1 \bm M_{t-1,r} + (1 - \beta_1) \bm G_{t,r} & \quad \text{(First subspace moment)} \\[0.5em]
\bm V_{t,r} &= \beta_2 \bm V_{t-1,r} + (1 - \beta_2) (\bm G_{t,r} \odot \bm G_{t,r}) & \quad \text{(Second subspace moment)} \\
\bm W_{t+1} &= \bm W_t - \eta_t \bm Q_t\left( \frac{\bm M_{t,r}}{\sqrt{\bm V_{t,r}}} \right) & \quad \text{(Project back and update)}
\end{array}
\end{equation}
Moreover, by switching off subspace momentum accumulations, the GaLore update can simply be seen as projected gradient descent: $\bm W_{t+1} = \bm W_t - \eta_t \bm Q_t \bm Q_t^{\top} \bm G_t$. In addition to subspace accumulation, GaLore performs offline updates of the subspace $\bm Q_t$ by updating it as $\bm U_{\bm G_t}^{1:r}$, the top-$r$ left singular vectors of $\bm G_t$, obtained via an SVD operation at predetermined intervals.
\section{Momentum Factorized SGD with Spectral Normalization}
This section introduces MoFaSGD, our memory-efficient optimization method, and its theoretical underpinnings. We begin in Subsection~\ref{sec:core-alg} by detailing the core algorithmic ideas behind MoFaSGD. This includes the motivation derived from the low-rank structure observed in optimizers, the process of maintaining and efficiently updating low-rank momentum factors using tangent space projection, and the subsequent use of these factors for spectrally normalized parameter updates (Algorithm~\ref{alg:mofasgd_complete}). Following the algorithmic description, Subsection~\ref{sec:core-theory} presents the convergence properties and theoretical analysis of MoFaSGD, justifying our design choices and establishing formal performance guarantees.

\subsection{The MoFaSGD Algorithm}
\begin{algorithm}[t]
\caption{\textbf{MoFaSGD}: Momentum Factorized Stochastic Gradient Descent}
\label{alg:mofasgd_complete}
\begin{tabular}{@{}p{0.54\textwidth}@{\hspace{0.8em}}|@{\hspace{0.8em}}p{0.42\textwidth}@{}}
\begin{minipage}[t]{\linewidth}
\begin{algorithmic}[1]
\Require Step size $\eta$, decay $\beta$, rank $r$
\Ensure Optimized weights $\bm{W}_t$
\State \textbf{Initialize:} $\bm{W}_0$
\State $\bm{G}_0 \gets \nabla_{\bm{W}} \mathcal{L}(\bm{W}_0)$
\State Initialize moment factors: $(\bm{U}_0, \bm{\Sigma}_0, \bm{V}_0) \gets \operatorname{SVD}_r(\bm{G}_0)$
\State $t \gets 0$
\State \textbf{\textcolor{darkblue}{repeat}}
\State \quad $\bm{G}_t \gets \nabla_{\bm{W}} \mathcal{L}(\bm{W}_t)$
\State \quad $(\bm{U}_{t+1}, \bm{\Sigma}_{t+1}, \bm{V}_{t+1}) \gets \textcolor{accent}{\textsc{UMF}}(\bm{G}_t, \bm{U}_t, \bm{\Sigma}_t, \bm{V}_t, \beta)$
\State \quad $\bm{W}_{t+1} \gets \bm{W}_t - \eta \bm{U}_{t+1} \bm{V}_{t+1}^\top$
\State \quad $t \gets t + 1$
\State \textbf{\textcolor{darkblue}{until}} convergence criterion is met
\State \textbf{return} $\bm{W}_t$
\end{algorithmic}
\end{minipage}
&
\begin{minipage}[t]{\linewidth}
\vspace{-0.5em}
\textbf{\textcolor{darkblue}{function}} \textcolor{accent}{\textsc{UMF}}($\bm{G}_t$, $\bm{U}_t$, $\bm{\Sigma}_t$, $\bm{V}_t$, $\beta$)
\begin{algorithmic}[1]
\State \textcolor{subtle}{\textit{Compute subspace projections:}}
\State \quad $\bm{G}_t \bm{V}_t$, $\bm{U}_t^{\top} \bm{G}_t$, $\bm{U}_t^{\top} \bm{G}_t \bm{V}_t$
\State Compute QR factors:
\State \quad $(\bm{U}_t', \bm{R}_{\bm{U}_t}) = \textsc{QR}([\bm{U}_t \quad \bm{G}_t \bm{V}_t])$
\State \quad $(\bm{V}_t', \bm{R}_{\bm{V}_t}) = \textsc{QR}([\bm{V}_t \quad \bm{G}_t^{\top} \bm{U}_t])$
\State Construct $2r \times 2r$ matrix:
\State \quad $\bm{S}_t = \bm{R}_{\bm{U}_{t}}\begin{bmatrix} \beta\bm{\Sigma}_{t}-\bm{U}_{t}^{\top}\bm{G}_{t}\bm{V}_{t} & \bm{I}_{r} \\ \bm{I}_{r} & \bm{0}_{r} \end{bmatrix} \bm{R}_{\bm{V}_{t}}^{\top}$
\State Compute rank-$r$ SVD:
\State \quad $\bm{U}_{t}^{\prime \prime} \bm{\Sigma}_{t}^{\prime \prime} (\bm{V}_{t}^{\prime \prime})^{\top} \gets \operatorname{SVD}_{r}(\bm{S}_t)$
\State $\bm{U}_{t+1} \gets \bm{U}_{t}^{\prime} \bm{U}_{t}^{\prime \prime}$
\State $\bm{V}_{t+1} \gets \bm{V}_{t}^{\prime} \bm{V}_{t}^{\prime \prime}$
\State $\bm{\Sigma}_{t+1} \gets \bm{\Sigma}_{t}^{\prime \prime}$
\State \textbf{return} $(\bm{U}_{t+1}, \bm{\Sigma}_{t+1}, \bm{V}_{t+1})$
\end{algorithmic}
\end{minipage}
\end{tabular}
\end{algorithm}

\subsubsection*{Motivation: Low-Rank Momentum Factors}
\label{sec:core-alg}
\citet{feinberg2024sketchy} show that the EMA of gradient covariance $\sum_{i=1}^t \beta^{t-i} \bm G_i \bm G_i^{\top}$ and $\sum_{i=1}^t \beta^{t-i} \bm G_i^{\top} \bm G_i$ maintain low-rank structure and spectral decay properties throughout training, and~\citet{zhao2024galore} argue that the gradients themselves become low-rank during fine-tuning. Building on these observations, we conjecture that the gradient EMAs exhibit low-rank properties. We experimentally evaluate this conjecture in Section~\ref{sec:ablations} and show that the mass of the gradient EMA is largely centered on its top few singular values. Low-rank structure has also been widely leveraged for LLM fine-tuning, e.g., LoRA~\citep{hu2021lora} adapts low-rank matrices during training, while GaLore~\citep{zhao2024galore} leverages the low-rank structure of gradients. Thus, leveraging the low-rank property of the gradient EMA connects the implicit assumptions underlying GaLore and LoRA.

Building on this conjecture, we propose to maintain a low-rank SVD factorized representation of the first momentum. We highlight that this factorization plays two major roles in our algorithm design. First, it is leveraged for online subspace sampling, and second, it provides a low-rank estimation of the first momentum (gradient EMA). Formally, we define the low-rank moment factors as:
\begin{equation}
\label{eqn:factored-moment}
\hat{\bm M}_t \triangleq \bm U_{t+1} \bm \Sigma_{t+1} \bm V_{t+1}^{\top} \approx \sum_{i=1}^t \beta^{t-i} \bm G_i 
\tag{Low-rank Moment Factors}
\end{equation}
where $\{\bm G_i\}_{i=1}^t$ are the observed gradients until iteration $t$. The left moment factor $\bm U_{t+1} \in \mbb R^{m \times r}$ and the right moment factor $\bm V_{t+1} \in \mbb R^{n \times r}$ are orthogonal matrices, while $\bm \Sigma_{t+1} \in \mbb R^{r \times r}$ is diagonal.

Our method maps the full-rank gradient $\bm G_t$ to the tangent space of the previous moment factor representation $(\bm U_t, \bm \Sigma_t, \bm V_t)$, to ensure a smooth adaptation of the subspace when a new gradient arrives. Formally, we leverage the tangent space of the previous iteration $\mc T_{(\bm U_t, \bm \Sigma_t, \bm V_t)}$ as the new subspace for projection, which we denote as $\mc T_t$:
\begin{equation}
\begin{split}
 \mc T_t &= \left \{ \bm U_t \bm M \bm V_t^{\top} + \bm U_p \bm V_t^{\top} + \bm U_t \bm V_p^{\top} | \bm M \in \mbb R^{r \times r} , \bm U_p \in \mbb R^{m \times r} , \bm V_p \in \mbb R ^{n \times r}, \bm U_t^{\top} \bm U_p= \bm 0, \bm V_t^{\top} \bm V_p = 0 \right\}
\end{split}
\end{equation}
and the projection to this subspace can be derived as follows:
\begin{equation}
\label{eqn:proj-tang}
 \hat{\bm G}_t \triangleq \operatorname{Proj}_{\mc T_t} (\bm G_t) = \bm U_t \bm U_t^{\top} \bm G_t + \bm G_t \bm V_t \bm V_t^{\top} - \bm U_t \bm U_t^{\top} \bm G_t \bm V_t \bm V_t^{\top}
\tag{Online subspace projection}
\end{equation}
where the definition of projection to the tangent space is $\operatorname{Proj}_{\mc T_t} (\bm G_t) = \underset{\bm G \in \mc T_t }{\arg \min} \, \| \bm G - \bm G_t\|_{\mathrm{F}}$.

Projecting onto the tangent subspace of previous gradients as shown in Theorem~\ref{thm:main-1}, results in a lower compression error compared to the left, right, or two-sided subspace projections used in GaLore~\citep{zhao2024galore}.

\subsubsection*{Efficient Momentum Factor Updates}
The second component of our approach is to efficiently approximate the current moment factor $\hat{\bm{M}}_t$ given the projected gradient $\hat{\bm{G}}_t$ and the previous moment factor $\hat{\bm{M}}_{t-1}$. A naive update, $\hat{\bm{M}}_t = \operatorname{SVD}_r \bigl(\hat{\bm{G}}_t + \beta \hat{\bm{M}}_{t-1} \bigr)$ where $\operatorname{SVD}_r(\cdot)$ denotes the rank-$r$ truncated SVD, involves a computationally expensive SVD operation, which we aim to avoid. Since both $\hat{\bm M}_{t-1}$ and $\hat{\bm G}_t$ are rank-$r$, their sum has a rank of at most $2r$. This observation allows us to approximate $\hat{\bm M}_{t}$ in $\mathcal{O}((m+n)r^2)$, which is far more efficient than a full-matrix SVD. Let $(\bm U'_t ,\bm R_{\bm U_t}) = \operatorname{QR} \bigl( \begin{bmatrix} \bm U_t & \bm G_t \bm V_t \end{bmatrix} \bigr) $, and $(\bm V'_t ,\bm R_{\bm V_t}) = \operatorname{QR} \bigl(\begin{bmatrix} \bm V_t & \bm G_t^{\top} \bm U_t \end{bmatrix} \bigr)$, where $\operatorname{QR}$ stands for the QR decomposition, and $\bm U'_t \in \mbb R^{m \times 2r}$, $\bm V'_t \in \mbb R^{n \times 2r}$ and $\bm R_{\bm U_t},\bm R_{\bm V_t} \in \mbb R^{2r \times 2r}$. Then we can write:
\begin{equation}
\label{eqn:moment-update-1}
\begin{split}
&\hat{\bm G}_t + \beta \hat{\bm M}_{t-1} = \bm U_t \bm U_t^{\top} \bm G_t + \bm G_t \bm V_t \bm V_t^{\top} + \bm U_t (\beta \bm \Sigma_t- \bm U_t^{\top} \bm G_t \bm V_t) \bm V_t^{\top} \\
&= \begin{bmatrix}
\bm U_t & \bm G_t \bm V_t
\end{bmatrix}
\begin{bmatrix}
\beta \bm \Sigma_t- \bm U_t^{\top} \bm G_t \bm V_t & \bm I_r \\
\bm I_r & \bm 0_r
\end{bmatrix}
\begin{bmatrix}
\bm V_t^{\top} \\
\bm U_t^{\top} \bm G_t
\end{bmatrix} = \bm U_t' \bigl( \bm R_{\bm U_t} \begin{bmatrix}
\beta \bm \Sigma_t- \bm U_t^{\top} \bm G_t \bm V_t & \bm I_r \\
\bm I_r & \bm 0_r
\end{bmatrix} \bm R_{\bm V_t}^\top \bigr) \bm V'^{\top}_t
\end{split}
\end{equation}
Let $ \bm U''_t \bm \Sigma^{''}_t \bm V_t''^{\top} = \operatorname{SVD}_r \bigl( \bm R_{\bm U_t} \begin{bmatrix}
\beta \bm \Sigma_t- \bm U_t^{\top} \bm G_t \bm V_t & \bm I_r \\
\bm I_r & \bm 0_r
\end{bmatrix} \bm R_{\bm V_t}^\top \big)$, and note that the inner matrix has rank at most $r$, and $\bm U_t'', \bm V_t'' \in \mbb R^{2r \times r}$. We can finally write our momentum factor update rule as:
\begin{equation}
\label{eqn:moment-update-2}
 \bm U_{t+1} = \bm U'_t \bm U_t'' \quad , \quad \bm V_{t+1} = \bm V'_t \bm V_t'' \quad , \quad \bm \Sigma_{t+1} = \bm \Sigma_t''
\end{equation}
The computation complexity of our approach includes two QR decompositions, $\mathcal{O}((m+n)r^2)$, one full SVD on a $2r \times 2r$ matrix, $\mathcal{O}(r^3)$, and hence the total complexity is $\mathcal{O}((m+n)r^2 + r^3)$.

\subsubsection*{From Momentum Factors to Spectrally Normalized Updates} Inspired by the connection between spectrally normalized gradient updates and effective non-diagonal preconditioning methods like Shampoo~\citep{gupta2018shampoo}, and motivated by the strong empirical performance of Muon~\citep{jordan2024muon}, which applies spectral normalization to gradient momentum, we leverage our low-rank momentum factorization (Equation~\ref{eqn:factored-moment}) for the main optimizer step. The MoFaSGD update rule is:
\begin{equation}
\label{eqn:main_update_step}
 \bm W_{t+1} = \bm W_t - \eta \bm U_{t+1} \bm V_{t+1}^\top
\end{equation}
Here, $\bm U_{t+1} \in \mathbb{R}^{m \times r}$ and $\bm V_{t+1} \in \mathbb{R}^{n \times r}$ represent the left and right singular vectors derived from the efficiently computed low-rank approximation of the first-order momentum, $\hat{\bm M}_t$. MoFaSGD (summarized in Algorithm~\ref{alg:mofasgd_complete}) contrasts with Muon, which operates on the full-rank momentum $\bm M_t$ (requiring $\mathcal{O}(mn)$ memory) and uses Newton-Schulz iterations to approximate $\bm U_{\bm M_t} \bm V_{\bm M_t}^\top$ for its update step. MoFaSGD can thus be viewed as a memory-efficient, low-rank variant of Muon.

Furthermore, we highlight the key distinctions between MoFaSGD and GaLore~\citep{zhao2024galore}. GaLore employs a two-stage process: 1) Projecting gradients onto a low-rank subspace defined by the singular vectors of the \textit{gradient} itself, updated periodically (referred to as offline subspace resampling), and 2) Accumulating first and second moments (akin to Adam) within this low-rank subspace.

MoFaSGD adopts different strategies for both stages. Firstly, regarding the subspace definition, MoFaSGD performs gradient low-rank projection, but crucially, onto the \textit{tangent space} defined by the singular vectors of the \textit{gradient momentum} $(\hat{\bm M}_{t})$. This online, per-iteration subspace adaptation contrasts with GaLore's offline resampling based on single gradients. The choice of the tangent space projection is theoretically motivated by its optimality in minimizing projection residuals (Theorem~\ref{thm:main-1}), while using the momentum's singular vectors aims for a more stable subspace that evolves smoothly, mitigating potential noise in individual gradients.

Secondly, concerning the optimizer update, MoFaSGD deliberately avoids accumulating moments \textit{within} the subspace, unlike GaLore, thereby avoiding potential error propagation from stale subspaces. MoFaSGD directly uses the computed momentum factors $(\bm U_{t+1}, \bm V_{t+1})$ to perform the spectrally normalized update in Equation~\ref{eqn:main_update_step}. This design choice aims to circumvent potential errors arising from subspace moment accumulation, particularly when the subspace changes frequently (i.e., near-online updates, or small subspace update intervals in GaLore). As empirically supported in Section~\ref{sec:ablations}, frequent subspace updates can indeed negatively impact GaLore's performance, suggesting that subspace moment accumulation errors might increase with the frequency of subspace changes.

MoFaSGD's \textit{novelty} lies in its unique combination of: 1) Projecting gradients onto the dynamically updated tangent space derived from the low-rank \textit{momentum} factors, and 2) Utilizing these factors directly for spectrally normalized updates, thereby bypassing the potential pitfalls of subspace moment accumulation inherent in methods like GaLore.

\subsection{Convergence and Theoretical Analysis}
\label{sec:core-theory}

Our analysis addresses the non-convex optimization problem:
\begin{equation}
\min_{\bm{W} \in \mathbb{R}^{m \times n}} \;\mathcal{L}(\bm{W})
\;=\; \mathbb{E}_{\xi}\bigl[\mathcal{L}(\bm{W}, \xi)\bigr]
\end{equation}
where we assume access to an unbiased, variance-bounded stochastic gradient oracle $\nabla \mathcal{L}(\bm{W}, \xi)$. Below, we first introduce the necessary definitions and assumptions that underpin our theoretical results.

\subsubsection*{Definitions and Assumptions}

For any optimization iterate $\bm W_i$, we denote the full-batch gradient by $\Bar{\bm G}_i = \nabla \mc L (\bm W_i)$ and the stochastic gradient by $\bm G_i = \nabla \mc L (\bm W_i , \xi_i)$. Formally, we leverage the following standard assumptions throughout our analysis:

\begin{assumption}
\label{assmp-main:1}
$\mc L(.)$ is $L$-smooth with respect to the nuclear norm $\| .\|_*$. In other words, for any two arbitrary $\bm W_1 , \bm W_2 \in \mbb R^{m \times n}$, we have: $\|\nabla \mc L(\bm W_1) - \nabla \mc L(\bm W_2)\|_* \le L \| \bm W_{1} - \bm W_2\|_2 $
\end{assumption}
This assumption naturally generalizes the typical smoothness condition from vector optimization and has been previously utilized in the literature~\citep{large2025scalable,old_bernstein_2024}. For further details, please see Appendix~\ref{sec:app-extra}. Additionally, we assume the availability of a stochastic gradient oracle satisfying standard properties:
\begin{assumption}
\label{assmp-main:2}
For any model parameter $\bm W$, we have access to an unbiased and variance-bounded stochastic oracle, as follows: $\mbb E_{\xi} [\nabla \mc L(\bm W , \xi)] = \nabla \mc L (\bm W)$, and $\mbb E_{\xi} [ \|\nabla \mc L(\bm W , \xi) - \nabla \mc L (\bm W)\|_*] \le \sigma$
\end{assumption}

With these definitions and assumptions clarified, we now present an intuitive overview of our main theoretical results, which highlight the strengths and optimality of our proposed MoFaSGD algorithm.

\noindent\textbf{Optimality of Tangent Space Projection (Theorem~\ref{thm:main-1}).}
We establish that projecting each gradient $\bm{G}_t$ onto the tangent space defined by its singular vectors achieves the minimal projection residual error among a broad class of low-rank projection schemes, such as projection onto the left or right singular vector subspaces.
\noindent\textbf{Optimal $O(1/\sqrt{T})$ Convergence Rate (Theorem~\ref{thm:main-conv}).}
Under Assumptions~\ref{assmp-main:1} and~\ref{assmp-main:2}, MoFaSGD achieves convergence to a stationary point at the optimal $O(1/\sqrt{T})$ rate. Critically, the factorization of momentum does not degrade the asymptotic convergence rate.

\subsubsection*{Proof Outline} Our proof structure involves three primary steps. First, we decompose the momentum low-rank approximation error by defining the full-rank momentum as $\bm{M}_t = \sum_{i=0}^t \beta^{t-i} \bm{G}_i$, and breaking down the approximation error into two key components: $\|\hat{\bm{M}}_t - \Bar{\bm{G}}_t\|_* \le \|\bm{M}_t - \Bar{\bm{G}}_t\|_* + \|\hat{\bm{M}}_t - \bm{M}_t\|_*$. These terms are individually controlled via exponential averaging and tangent-space projections. Second, we rigorously prove the optimality of the tangent-space projection (Theorem~\ref{thm:main-1}), demonstrating that $\hat{\bm G}_t = \bm U_t \bm U_t^\top \bm G_t + \bm G_t \bm V_t \bm V_t^\top - \bm U_t \bm U_t^\top \bm G_t \bm{V}_t \bm{V}_t^\top$ minimizes the residual $\|\bm{G}_t - \hat{\bm{G}}_t\|_{\mathrm{F}}$, thus ensuring optimal fitting into the evolving momentum subspace. Lastly, by combining the upper bounds on the aforementioned terms with a standard descent lemma under nuclear-norm smoothness, we derive our final convergence behavior. By carefully bounding the approximation terms from the prior steps and leveraging them in our derived descent lemma under nuclear-norm smoothness, we arrive at our main convergence result.

The low-rank factorization of gradient momentum as described by Equation~\ref{eqn:factored-moment} is the cornerstone of our proposed method. The quality of the momentum approximation plays a key role in the effectiveness of our approach. We provide an intuitive sketch of the theoretical analysis to bound the factorization residual $\|\hat{\bm{M}}_t - \bm{M}_t\|_* = \|\sum_{i=1}^t \beta^{t-i} \bm G_i - \bm U_{t+1} \bm \Sigma_{t+1} \bm V_{t+1}^{\top}  \|_{\mathrm{F}}^2$. By recursively bounding this term, we can show that it is sufficient to bound the term $\| \bm U_{t+1} \bm \Sigma_{t+1} \bm V_{t+1}^{\top} - \beta \bm U_{t} \bm \Sigma_{t} \bm V_{}^{\top} - \bm G_t\|_{\mathrm{F}}$, which can itself be shown to be bounded by $\| \operatorname{Proj}_{\mc T_t}(\bm G_t) - \bm G_t\|_{\mathrm{F}}$. 
Thus, the quality of the factored momentum approximation is directly related to the residual of the gradient's low-rank projection.

\subsubsection*{Results}
The choice of tangent space projection is optimal in the sense of minimizing the term $\| \operatorname{Proj}_{\mc T_t}(\bm G_t) - \bm G_t\|_{\mathrm{F}}$, as demonstrated in the following theorem.
\begin{theorem}
\label{thm:main-1}
 Let $\bm L \in \mbb R^{m \times r}$, $\bm R \in \mbb R^{n \times r}$ be any arbitrary sketching matrices, and $(\alpha_1,\alpha_2,\alpha_3)$ be any arbitrary triple of scalars. Let $\operatorname{Proj}_{(\bm L, \bm R)} (\bm G) = \alpha_1 \bm L \bm L^{\top} \bm G + \alpha_2 \bm G \bm R \bm R^{\top} + \alpha_3 \bm L \bm L^{\top} \bm G \bm R \bm R^{\top}$. Then, the projection residual is minimized when $(1) (\alpha_1,\alpha_2,\alpha_3) = (1 , 1 , -1)$ and $(2) \bm L^{\top} \bm L = \bm R^{\top} \bm R = \bm I_r$. In this case, the residual norm is $\| \operatorname{Proj}_{(\bm L, \bm R)}(\bm G_t) - \bm G_t\| = \|(\bm I - \bm L \bm L^{\top}) \bm G_t (\bm I - \bm R \bm R^{\top})\|$.
\end{theorem}

 \begin{remark}
If we let $\bm L = \bm U^{1:r}_{\bm G_t}$ and $ \bm R = \bm V^{1:r}_{\bm G_t}$, then we can conclude that the residual error would be upper bounded by $\sigma^{r+1 : \min(m,n)}_{\bm G_t} = \sum_{i = r+1}^{\min(m,n)} \sigma^i_{\bm G_t}$, which, considering the low-rank property of the gradient, we expect it to be small, or even zero if the rank of the full gradient is less than $r$. Note that the subspace projection in GaLore~\citep{zhao2024galore} is actually equivalent to letting either $(\alpha_1,\alpha_2,\alpha_3) = (0 , 0 , 1)$ or $(\alpha_1,\alpha_2,\alpha_3) = (1 , 0 , 0)$.
\end{remark}

We now present our main convergence bound for Algorithm~\ref{alg:mofasgd_complete}.

\begin{theorem}
\label{thm:main-conv}
 Let Assumptions~\ref{assmp-main:1} and~\ref{assmp-main:2} hold. Moreover, assume $\operatorname{rank}(\bm G_0) \le r$. By letting $\beta \le \frac{1}{3}$ and $\eta \le 1$, the iterates of Algorithm~\ref{alg:mofasgd_complete} satisfy the following:
\begin{equation}
 \frac{1}{T} \sum_{t=0}^T \mbb E [ \|\nabla \mc L (\bm W_t)\|_*] \le \mc O \left ( \frac{ \mc L(\bm W_0) - \mbb E [\mc L(\bm W_{T+1})]}{\eta T} + \eta L + \frac{ \sigma}{\sqrt{T}} \right )
\end{equation}
Moreover, if we set $ \eta = \Theta \left(\sqrt{\frac{\mc L(\bm W_0) - \mbb E [\mc L(\bm W_{T+1})]}{T L}}\right)$, we can derive the following simplified bound as:
\begin{equation}
 \frac{1}{T} \sum_{t=0}^T \mbb E [ \|\nabla \mc L (\bm W_t)\|_*] \le \mc O \left ( \frac{\bigl(\mc L(\bm W_0) - \mbb E [\mc L(\bm W_{T+1})]\bigr)^{\frac{1}{2}}\sqrt{L} + \sigma}{\sqrt{T}} \right )
\end{equation}
\end{theorem}
\begin{remark} The convergence metric used for Theorem~\ref{thm:main-conv} is the stationary point with respect to the nuclear norm, which is averaged over all gradients. When setting $\sigma = 0$, Algorithm~\ref{alg:mofasgd_complete} achieves the rate of $\mc O(\frac{1}{\sqrt{T}})$, which is known to be optimal in the sense of non-convex stochastic optimization under smoothness and an unbiased, bounded stochastic gradient oracle~\citep{arjevani2023lower}. \end{remark}

\section{Experiments}
\label{sec:exp}

We evaluate MoFaSGD's effectiveness and efficiency across three large language modeling setups: pre-training, natural language understanding (NLU) fine-tuning, and instruction-tuning. These setups allow us to assess MoFaSGD's performance across different training regimes and task complexities. 

\subsection{Pre-training setup: NanoGPT Speedrun}

We first evaluate MoFaSGD in a pre-training context using the Modded NanoGPT benchmark~\citep{modded_nanogpt_2024}. This benchmark focuses on training a GPT-2 architecture on a subset of the FineWeb dataset~\citep{penedo2025fineweb} and measures performance using validation perplexity on a held-out partition of FineWeb. The benchmark's default optimizer is Muon~\citep{jordan2024muon}, which holds current training speed records for this task and serves as a strong, competitive baseline.

We compare MoFaSGD against full fine-tuning baselines AdamW~\citep{loshchilov2017fixing} and Muon, as well as the low-rank baseline GaLore~\citep{zhao2024galore}. Following the standard NanoGPT speedrun setup, we use the hyperparameters tuned for Muon. We tune the learning rates for AdamW, GaLore, and MoFaSGD, along with GaLore's SVD frequency and MoFaSGD's momentum decay ($\beta$) via grid search, selecting the best configuration based on final validation perplexity (details in Appendix~\ref{sec:app-nano}). We evaluate ranks $r \in \{16, 32, 128\}$, common choices for low-rank methods in pre-training setup.

Our primary experiment uses a budget of $0.73$ billion tokens from FineWeb, aligning with the budget used by Muon to reach the target perplexity of $3.27$. Figure~\ref{fig:gpt2_short} shows the validation perplexity curves. MoFaSGD consistently outperforms GaLore across all tested ranks, achieving lower final perplexity. The performance advantage is particularly noticeable at lower ranks ($r=16$). This suggests MoFaSGD's dynamic subspace tracking is more effective at capturing important momentum directions than GaLore's infrequent updates, especially under strict rank constraints. Both low-rank methods underperform the full-rank AdamW and Muon baselines within this specific token budget. This is likely because full-rank methods have more degrees of freedom, and the $0.73B$ token budget, optimized for Muon's convergence speed, might be insufficient for low-rank methods to fully match their performance.

To assess longer-term performance, we conduct an extended run for $10,000$ steps ($\sim5.3B$ tokens) using rank $r=32$ for both MoFaSGD and GaLore. As shown in Figure~\ref{fig:gpt2_long}, MoFaSGD maintains its performance advantage over GaLore, indicating that the benefits of its momentum factorization approach persist during longer training phases.

\subsubsection*{Ablation: Convergence vs. Efficiency}
We conduct a detailed ablation study on the effect of the low-rank parameter $r \in \{16, 32, 128\}$ during NanoGPT pre-training for both GaLore and MoFaSGD. As illustrated in Figure~\ref{fig:rank_ablation_steps} and Figure~\ref{fig:rank_ablation_time}, higher ranks consistently improve convergence speed and final validation loss. MoFaSGD achieves smoother loss curves and stronger performance across all ranks, particularly under tight memory budgets (e.g., $r=16$), where GaLore exhibits noticeable instability. This supports our claim that MoFaSGD's tangent-space subspace tracking better preserves optimizer continuity under aggressive compression.

In terms of runtime (Table~\ref{tab:rank_ablation_results}), GaLore shows minimal runtime variation across ranks due to dominant offline SVD costs. In contrast, MoFaSGD’s runtime scales more significantly with rank, reflecting its per-step online factorization cost. Nonetheless, MoFaSGD remains faster than GaLore at $r=32$ and achieves superior final loss, highlighting a favorable trade-off between expressivity and computational efficiency.

\begin{table}[h]
\centering
\caption{Comparison of MoFaSGD and GaLore across different ranks during NanoGPT pre-training. Best values per row are in \textbf{bold}.}
\label{tab:rank_ablation_results}
\begin{tabular}{c|cc|cc|cc}
\toprule
\multirow{2}{*}{\textbf{Rank}} & \multicolumn{2}{c|}{\textbf{Final Val Loss}} & \multicolumn{2}{c|}{\textbf{Runtime (s)}} & \multicolumn{2}{c}{\textbf{Throughput}} \\
& MoFaSGD & GaLore & MoFaSGD & GaLore & MoFaSGD & GaLore \\
\midrule
16 & \textbf{3.8981} & 4.0773 & \textbf{2156} & 3755 & \textbf{338{,}450} & 194{,}395 \\
32 & \textbf{3.7208} & 3.8953 & 3972 & \textbf{3911} & 183{,}770 & \textbf{186{,}619} \\
128 & \textbf{3.5700} & 3.6561 & 4817 & \textbf{3839} & 151{,}527 & \textbf{190{,}131} \\
\bottomrule
\end{tabular}
\end{table}

\begin{figure}[h]
\centering
\includegraphics[width=0.45\textwidth]{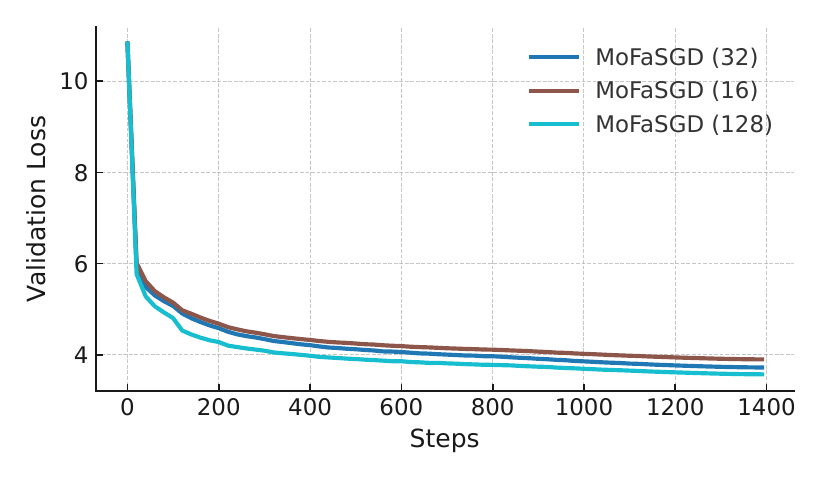}
\includegraphics[width=0.45\textwidth]{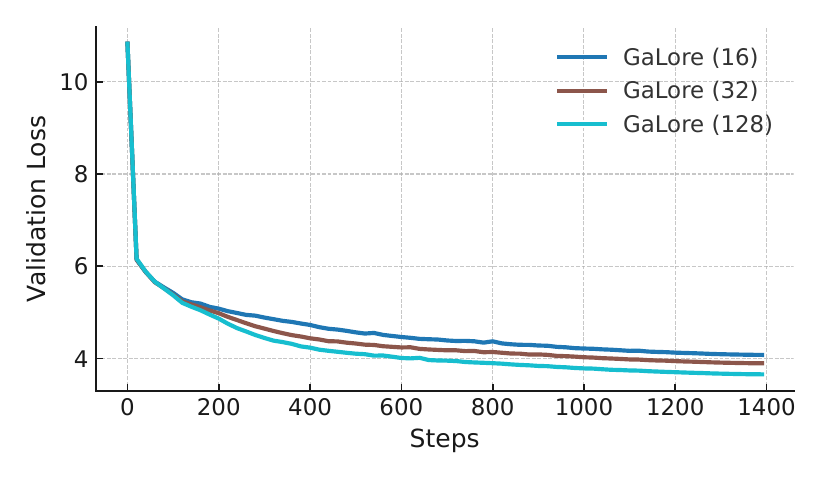}
\caption{Validation loss vs. training steps for MoFaSGD (left) and GaLore (right) across ranks $r \in \{16, 32, 128\}$. MoFaSGD shows smoother and faster convergence.}
\label{fig:rank_ablation_steps}
\end{figure}

\begin{figure}[h]
\centering
\includegraphics[width=0.45\textwidth]{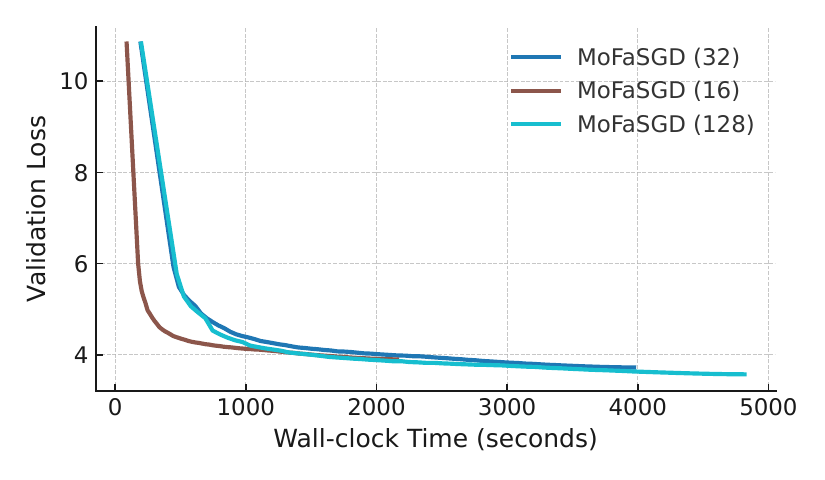}
\includegraphics[width=0.45\textwidth]{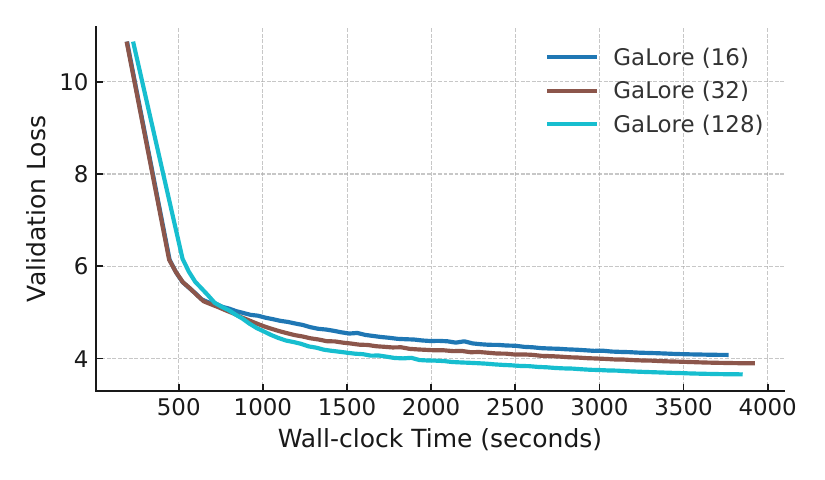}
\caption{Validation loss vs. wall-clock time across ranks. MoFaSGD scales better in convergence and runtime efficiency.}
\label{fig:rank_ablation_time}
\end{figure}

\begin{figure}[ht]
\centering
\begin{subfigure}[b]{0.51\textwidth}
\centering
\includegraphics[width=\linewidth]{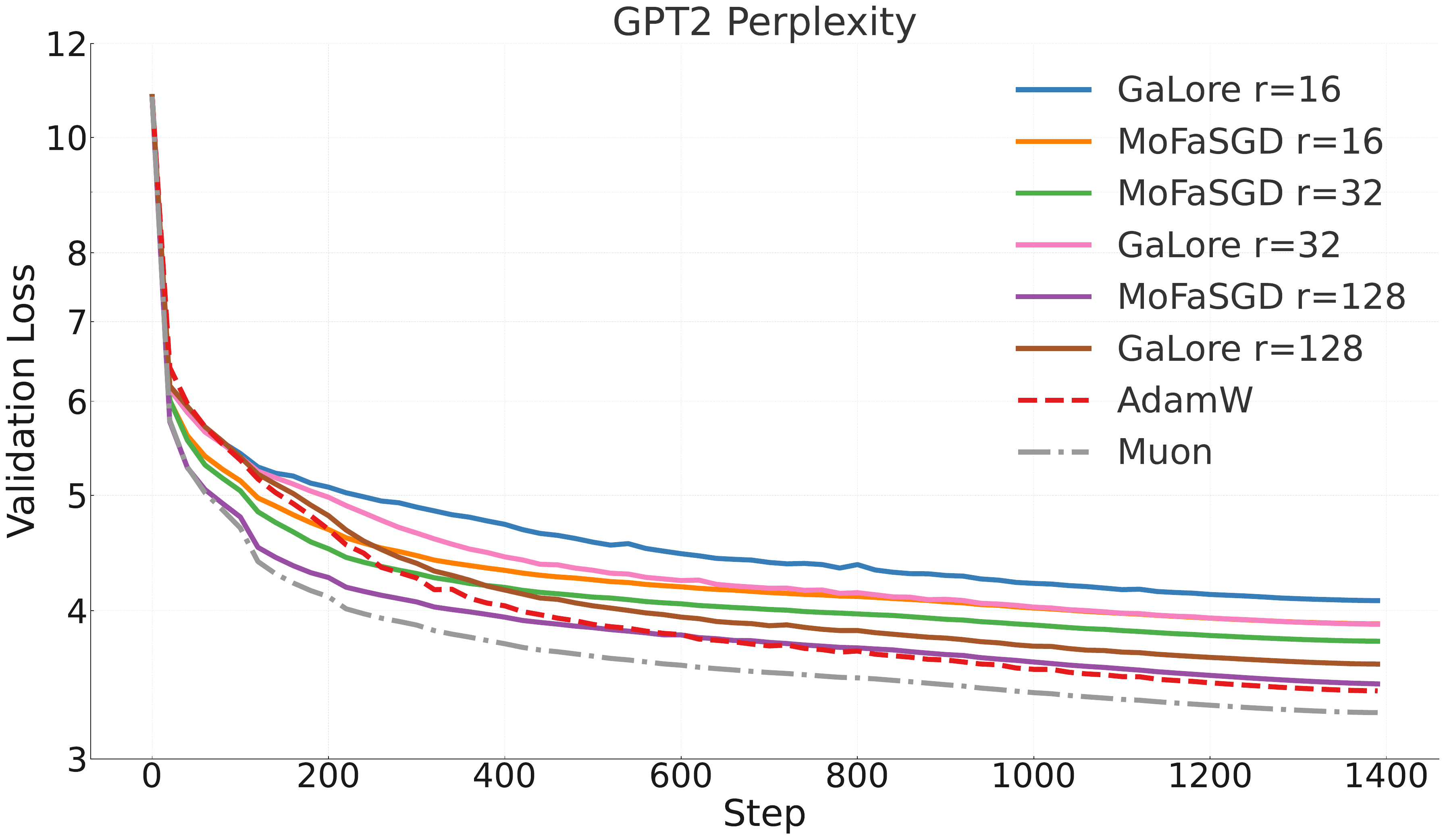}
\caption{GPT-2 Perplexity over $\sim 1.39k$ steps.}
\label{fig:gpt2_short}
\end{subfigure}
\hfill
\begin{subfigure}[b]{0.47\textwidth}
\centering
\includegraphics[width=\linewidth]{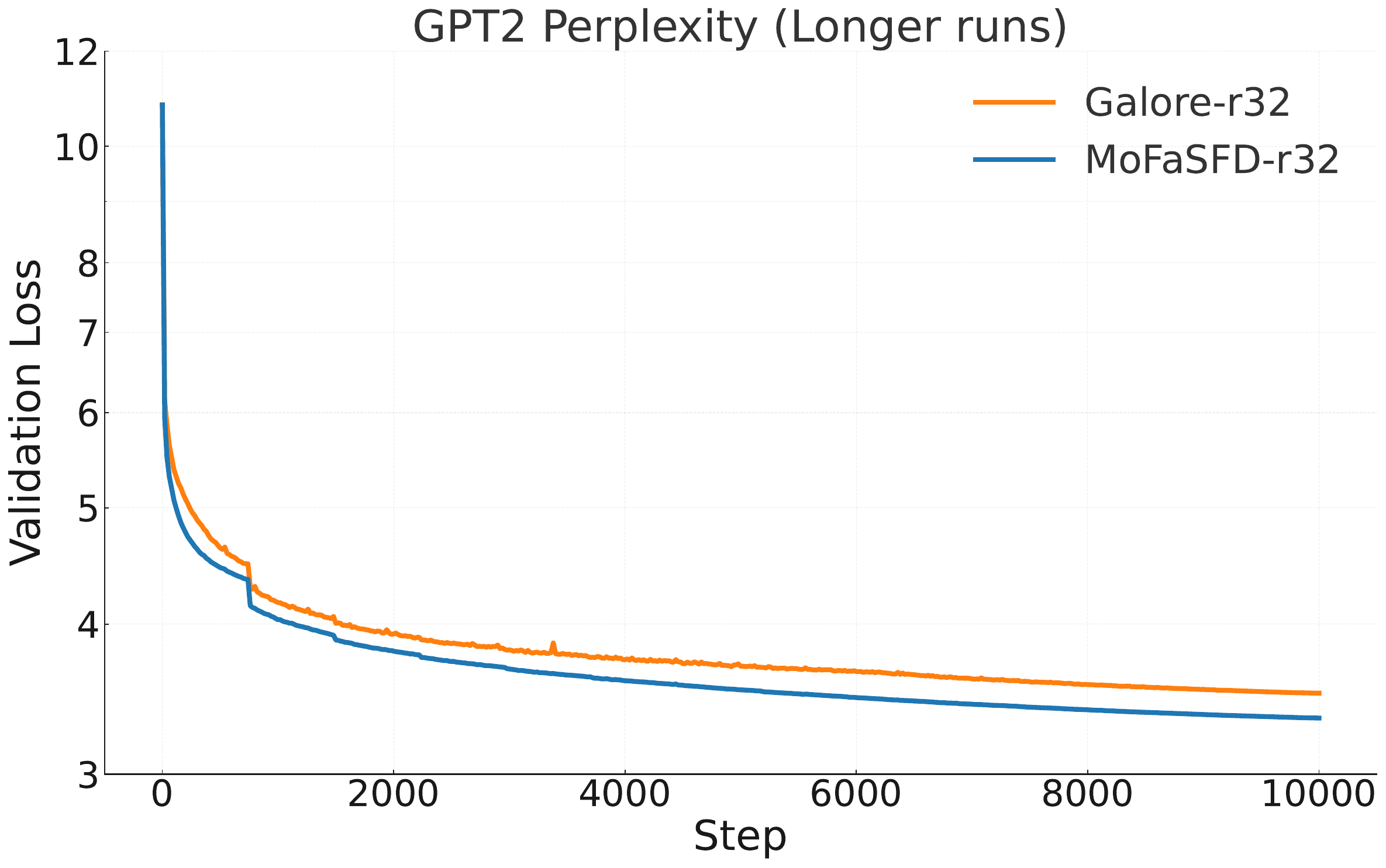}
\caption{GPT-2 Perplexity over $10k$ steps}
\label{fig:gpt2_long}
\end{subfigure}
\caption{Comparison of optimizer performance on GPT-2 using validation perplexity loss.}
\label{fig:gpt2_comparison}
\end{figure}

\subsection{Post-training Evaluation: Setups and Results}
\label{sec:exp-setup}

We evaluate MoFaSGD in two post-training scenarios: standard NLU fine-tuning on the GLUE benchmark and large-scale instruction-tuning on the Tulu3 dataset.

\noindent\textbf{GLUE Benchmark.} We follow the experimental setup detailed in~\citet{zhao2024galore} for fine-tuning a RoBERTa-Base model ($125M$ parameters) on seven diverse NLU tasks from the GLUE benchmark~\citep{wang2018glue}: MNLI, QQP, SST-2, MRPC, COLA, QNLI, and RTE. We use the same hyperparameters as~\citet{zhao2024galore} for comparison baselines to ensure fairness. For MoFaSGD (with ranks $r=4$ and $r=8$), we tune the learning rate for each task based on validation accuracy, keeping the momentum decay $\beta$ fixed at $0.95$ (details in Appendix~\ref{sec:app-glue}). Final validation accuracy and loss are used for evaluation.

\noindent\textbf{Tulu3 Instruction Tuning.}
To assess performance on more complex alignment tasks, we fine-tune the LLaMA-3.1 8B model on the \texttt{tulu-3-sft-mixture} dataset, a large (\(\sim 900\text{K}\) samples) and diverse instruction-tuning dataset~\citep{lambert2024t}. We adopt most hyperparameters from the Tulu3 setup~\citep{lambert2024t}, training for one epoch with an effective batch size of 128; however, we use only a subsample of 200K examples to reduce the training budget. We perform a grid search over learning rates for each optimizer. For the low-rank methods (MoFaSGD, LoRA, GaLore), we use rank \(r = 8\). Key hyperparameters for MoFaSGD (\(\beta\)) and the baselines (GaLore SVD frequency, LoRA alpha) are set as specified in Appendix~\ref{sec:app-tulu}. We evaluate the final checkpoints using the OLMES evaluation framework~\citep{gu2024olmes}, benchmarking across MMLU, TruthfulQA, BigBenchHard, GSM8K, and HumanEval.

\subsubsection*{Performance Analysis} In these post-training tasks, we compare MoFaSGD against LoRA~\citep{hu2021lora} optimized with AdamW, and GaLore~\citep{zhao2024galore}, which are prevalent methods for memory-efficient fine-tuning. We also include results for full-parameter fine-tuning with AdamW as a performance ceiling reference. Table~\ref{table:optimizer_comparison} compares the theoretical complexities, highlighting MoFaSGD's comparable memory footprint to LoRA alongside efficient online subspace updates, while Figure~\ref{fig:mem-break-bar} comprehensively details the memory usage of MoFaSGD compared to other baselines in our Tulu3 instruction-tuning setup.
\begin{figure*}[t]
 \centering
 \begin{minipage}[c]{0.50\textwidth} 
 \centering
 \small
 \setlength{\tabcolsep}{5pt}
 \renewcommand{\arraystretch}{1.2}
 \begin{tabular}{@{}lcc@{}}
 \toprule
 \textbf{Optimizer} & \textbf{Memory Complexity} & \textbf{Subspace Resampling} \\
 \midrule
 GaLore & $mn + mr + 2nr$ & $\mathcal{O}(m^2 n)$ (offline) \\
 LoRA & $mn + 3mr + 3nr$ & -- \\
 MoFaSGD & $mn + mr + nr + r$ & $\mathcal{O}((m + n)r^3)$ (online) \\
 \bottomrule
 \end{tabular}
 \vspace{0.5em}
 \captionof{table}{Comparison of memory and subspace resampling complexity for low-rank optimizers. Let $\bm W \in \mbb R^{m \times n}$ represent model parameters, and $r$ is the rank of low-rank optimizers (w.l.o.g assume $m \le n$). Note that memory complexity includes model parameters and optimizer states.}
 \label{table:optimizer_comparison}
 \end{minipage}
 \hfill
 \begin{minipage}[c]{0.43\textwidth} 
 \centering
 \includegraphics[width=\linewidth]{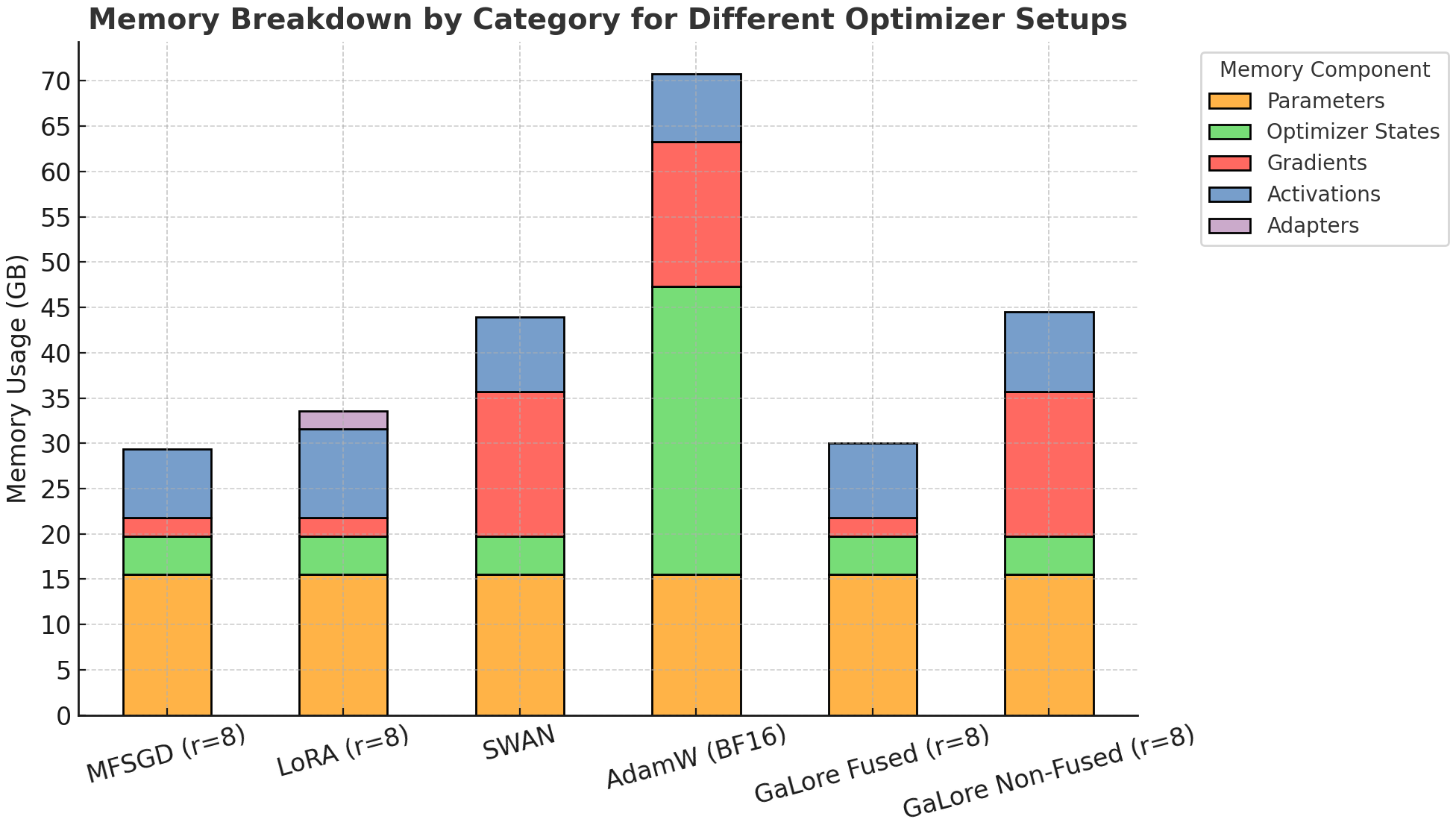}
 \captionof{figure}{Empirical memory breakdown (GB) for LLaMA3.1-8B using different optimizers.}
 \label{fig:mem-break-bar}
 \end{minipage}
\end{figure*}

To illustrate MoFaSGD’s optimization efficiency during instruction tuning on the Tulu3 benchmark, we compare its validation loss trajectory against GaLore and LoRA across both training epochs and wall-clock time. As shown in Figure~\ref{fig:tulu_epoch} and Figure~\ref{fig:tulu_time}, MoFaSGD consistently achieves lower validation loss over the course of training, indicating superior sample efficiency. Notably, when measured against real-world wall-clock time, MoFaSGD converges faster than both GaLore and LoRA, demonstrating improved practical efficiency in addition to theoretical gains. These trends are further supported by our throughput analysis: MoFaSGD reaches 4206 tokens/sec, outperforming GaLore (3214 tokens/sec) and approaching the high throughput of LoRA (4536 tokens/sec). These findings reinforce our earlier conclusion that MoFaSGD’s spectrally normalized updates and dynamic momentum factorization enable more effective fine-tuning under memory constraints. We have also included the training loss curves in Appendix~\ref{app:curve-train}, showcasing MoFaSGD’s convergence behavior on both GLUE and Tulu3 setups.

\begin{figure}[htbp]
 \centering
 \begin{subfigure}[b]{0.49\textwidth}
 \centering
 \includegraphics[width=\linewidth]{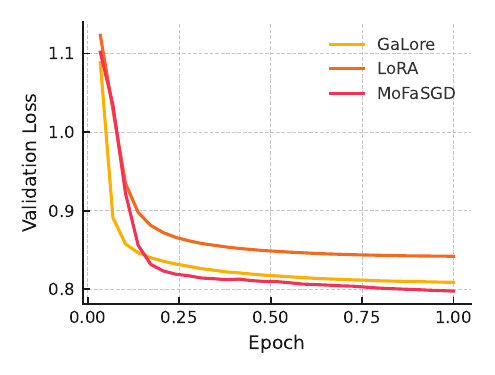}
 \caption{Validation Loss vs. Epoch}
 \label{fig:tulu_epoch}
 \end{subfigure}
 \hfill
 \begin{subfigure}[b]{0.49\textwidth}
 \centering
 \includegraphics[width=\linewidth]{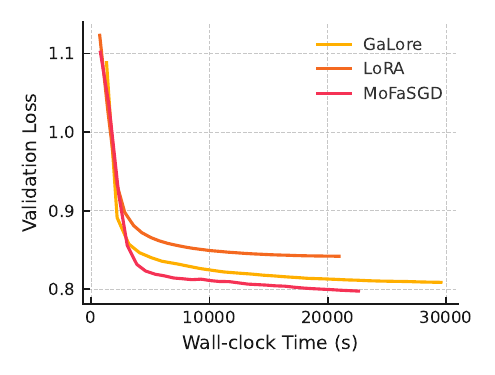}
 \caption{Validation Loss vs. Wall-clock Time}
 \label{fig:tulu_time}
 \end{subfigure}
 \caption{MoFaSGD demonstrates superior sample efficiency and faster wall-clock convergence, achieving lower validation loss than both GaLore and LoRA throughout LLaMA3.1-8B instruction tuning on the Tulu3 benchmark.}
\label{fig:tulu_validation_curves}
\end{figure}

Table~\ref{table:glue} summarizes the final validation accuracies on the GLUE tasks (ranks r=4 and r=8). MoFaSGD achieves performance comparable to, and on average slightly better ($+0.5\%$ for $r=4$, $+0.47\%$ for $r=8$) than both LoRA and GaLore, while using slightly less estimated memory. This demonstrates MoFaSGD's competitiveness and memory efficiency on standard NLU benchmarks.

Table~\ref{table:tulu} presents the results on the challenging instruction-tuning benchmarks using the Tulu3 setup (rank $r=8$). MoFaSGD outperforms both GaLore ($+0.8\%$ avg.) and LoRA ($+2.3\%$ avg.) on the average score across the five benchmarks. This highlights MoFaSGD's potential advantage in complex tasks where accurately capturing the training dynamics over long sequences and diverse instructions is crucial. However, consistent with prior work~\citep{wang2023far}, all low-rank methods exhibit a performance gap compared to full fine-tuning with AdamW (MoFaSGD is $-4.2\%$ avg. vs AdamW). This gap underscores the inherent trade-off between memory efficiency and achievable performance, particularly as task complexity increases and may require capturing subtle, high-rank parameter updates that low-rank approximations inherently miss.

\begin{table*}[htbp]
\centering
\caption{Comparison of final validation accuracies (\%) on seven GLUE tasks (MNLI, QQP, SST-2, MRPC, CoLA, QNLI, RTE)
when fine-tuning a RoBERTa-base model using different optimizers.
For GaLore, LoRA, and MoFaSGD, we report results with rank $r \in \{4,8\}$.
Memory usage is estimated for each method, and the final column shows the average accuracy.
Note that for memory measurement, we include only the parameters and the optimizer states for a fair comparison.}
\label{table:glue}
\resizebox{\textwidth}{!}{
\begin{tabular}{l|ccccccc|c|c}
\toprule
\textbf{Optimizer} & \textbf{MNLI} & \textbf{QQP} & \textbf{SST-2} & \textbf{MRPC} & \textbf{CoLA} & \textbf{QNLI} & \textbf{RTE} & \textbf{Memory} & \textbf{Avg.} \\
\midrule
AdamW (Full-Rank)
& 86.8
& 91.98
& 94.48
& 90.90
& 62.25
& 93.15
& 79.41
& 747M
& 85.57 \\
\midrule
GaLore (\(r=4\))
& \textbf{85.23}
& 89.62
& \textbf{94.17}
& 90.72
& 60.33
& \textbf{93.20}
& 77.22
& 253M
& 84.36 \\
LoRA (\(r=4\))
& 84.25
& 89.73
& 93.59
& 90.53
& 60.42
& 92.91
& 78.59
& 257M
& 84.29 \\
MoFaSGD (\(r=4\))
& 85.12
& \textbf{89.85}
& 94.15
& \textbf{90.78}
& \textbf{61.91}
& 93.10
& \textbf{79.08}
& 251M
& 84.86 \\
\midrule
GaLore (\(r=8\))
& 86.01
& 89.65
& 94.04
& 90.65
& 59.96
& \textbf{93.16}
& 78.14
& 257M
& 84.52 \\
LoRA (\(r=8\))
& 85.18
& 90.15
& 93.87
& \textbf{90.82}
& 61.11
& 93.05
& 78.77
& 264M
& 84.71 \\
MoFaSGD (\(r=8\))
& \textbf{86.32}
& \textbf{90.26}
& \textbf{94.36}
& 90.75
& \textbf{62.16}
& 93.12
& \textbf{79.28}
& 253M
& 85.18 \\
\bottomrule
\end{tabular}
}
\end{table*}

\begin{table*}[htbp]
\centering
\caption{Final scores of Llama-3.1\,8B on the Tulu3-SFT-mixture dataset using four different optimizers.
The table reports performance on MMLU, TruthfulQA, BigBenchHard, GSM8K, and HumanEval,
along with the average of these five benchmarks (Avg.).}
\label{table:tulu}
\resizebox{\textwidth}{!}{
\begin{tabular}{l|ccccc|c}
\toprule
\textbf{Optimizer} & \textbf{MMLU} & \textbf{TruthfulQA} & \textbf{BigBenchHard} & \textbf{GSM8K} & \textbf{HumanEval} & \textbf{Avg.} \\
\midrule
AdamW (Full-Rank)
& 62.8
& 46.5
& 66.7
& 72.7
& 81.0
& 65.9 \\
\midrule
GaLore
& 58.9
& 44.2
& 57.6
& 68.4
& 75.2
& 60.9 \\
LoRA
& 56.1
& 42.6
& 56.5
& 67.9
& 73.9
& 59.4 \\
MoFaSGD
& \textbf{59.4}
& \textbf{45.8}
& \textbf{58.3}
& \textbf{68.4}
& \textbf{76.8}
& \textbf{61.7} \\
\bottomrule
\end{tabular}
}
\end{table*}

\subsection{Ablations}
\label{sec:ablations}

\subsubsection*{Momentum Spectral Analysis} MoFaSGD is motivated by the conjecture that the first moment (the EMA of gradients) preserves a low-rank structure throughout training. This conjecture stems from the GaLore hypothesis on the low-rankness of the gradients themselves~\citep{zhao2024galore}, and, more importantly, from the observation in~\citet{feinberg2024sketchy}, which shows a fast spectral decay in the EMA of the gradient covariance, as discussed in more detail in Section~\ref{sec:core-alg}. To investigate our conjecture, we analyze the first-moment buffer $\bm M_t$ from the AdamW optimizer states generated during the Tulu3 instruction-tuning setup. 

For each relevant parameter matrix $\bm M_t$, we perform SVD and compute the energy ratio captured by the top-$r$ singular values as $\frac{\sum_{i=1}^{r}\sigma_{i,\bm M_{t}}^{2}}{\|\bm M_{t}\|_{\mathrm{F}}^{2}}$. This ratio represents the percentage of the momentum's total energy contained within the top-$r$ subspace. We compute the average of this ratio across all 2D weight matrices in the model at various training steps. Figure~\ref{fig:momentum_ratio} shows the average energy ratio for $r=16$ and $r=32$ throughout training. We observe that the top-$32$ singular values consistently capture around $80\%$ of the momentum's energy, while the top-$16$ capture approximately $75\%$. This persistent and significant concentration of energy in a low-rank subspace strongly supports our hypothesis.

\begin{figure}[ht]
 \centering
 \begin{subfigure}[b]{0.49\linewidth}
 \centering
\includegraphics[width=\linewidth]{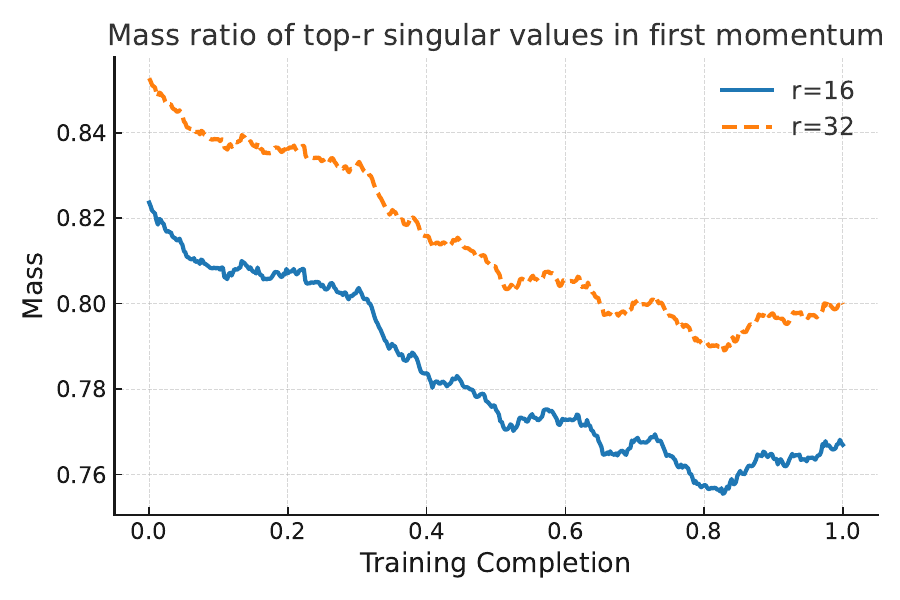}
 \caption{Momentum Spectral Mass Ratio}
 \label{fig:momentum_ratio}
 \end{subfigure}
 \hfill
 \begin{subfigure}[b]{0.49\linewidth}
 \centering
\includegraphics[width=\linewidth]{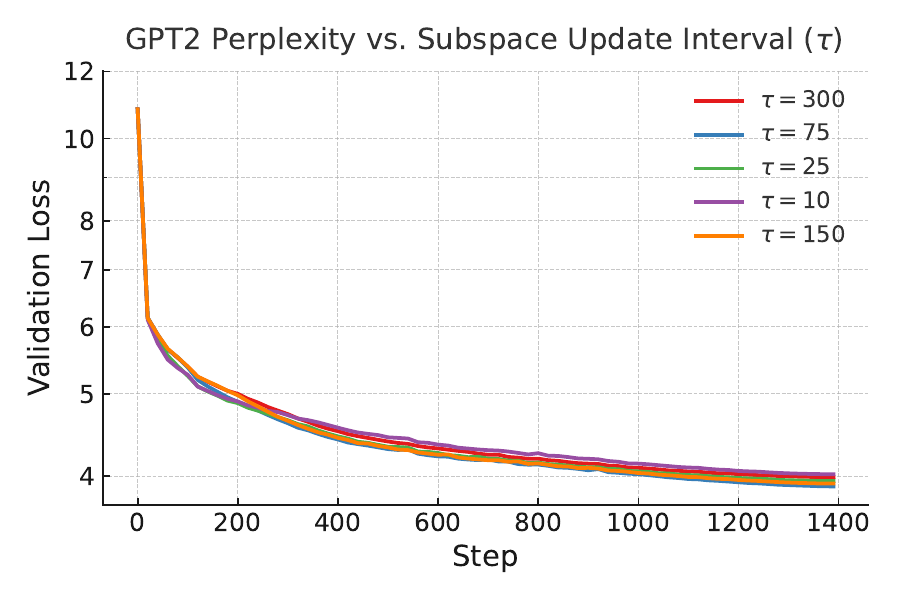}
 \caption{GaLore Update Frequency ($\tau$) Ablation}
 \label{fig:galore_tau_ablation}
 \end{subfigure}
 \caption{(a) Average mass ratio of the AdamW first moment captured by top-$r$ singular vectors during Tulu3 fine-tuning ($r=16, 32$). (b) Validation perplexity vs. subspace update interval ($\tau$) for GaLore ($r=32$) on NanoGPT.}
 \label{fig:ablation_plots}
\end{figure}

\subsubsection*{Impact of GaLore's Subspace Update Frequency} We analyze the impact of subspace update frequency in methods like GaLore compared to MoFaSGD's implicit per-iteration adaptation. MoFaSGD updates its momentum factors at every step, while GaLore performs explicit, costly subspace updates (e.g., full SVD on the gradient) at intervals $\tau$. To investigate whether simply increasing GaLore's update frequency (decreasing $\tau$) matches the benefits of MoFaSGD's online subspace adaptation, we conduct an ablation study on $\tau$ using the NanoGPT pre-training setup ( $0.73B$ token budget with rank $r=32$). We vary $\tau$ across $\{10, 25, 75, 150, 300\}$ steps. Figure~\ref{fig:galore_tau_ablation} shows the validation perplexity curves.

Very frequent updates ($\tau=10$ or $\tau=25$) do not yield the best performance and are slightly worse than less frequent updates (e.g., $\tau=150$). This aligns with GaLore's findings~\citep{zhao2024galore} and suggests that overly frequent subspace changes in GaLore can disrupt optimizer state accumulation. This finding underscores the challenge of online subspace changes. This finding suggests that MoFaSGD's approach of avoiding abrupt subspace changes offers a more stable and efficient way to leverage low-rank structures.

\subsection{Memory Usage Breakdown and Profiling}
\label{sec:memory-breakdown}

We assess MoFaSGD's memory efficiency by decomposing its GPU memory usage across five categories: parameters, optimizer states, gradients, activations, and adapters. Figure~\ref{fig:mem-break-bar} shows a comparative breakdown of memory consumption across six optimizer setups on LLaMA3.1-8B. MoFaSGD achieves total memory usage of $29.4$ GB, competitive with fused GaLore and LoRA, while enabling full-parameter updates. In contrast, AdamW exceeds $70$ GB due to high-cost full-rank momentum buffers and persistent gradient accumulation. SWAN and GaLore (non-fused) similarly suffer from gradient buffers, which dominate their memory footprints.

These savings arise from three key design elements: (i) eliminating second-moment buffers entirely, (ii) maintaining a low-rank SVD factorization of first-order momentum, and (iii) fusing gradient projection and zeroing operations during backpropagation, which prevents gradient accumulation from persisting across steps.

To further validate these results, Figure~\ref{fig:mfsgd-trace} shows the memory trace during MoFaSGD training. We observe a clean separation between parameter storage, a narrow and persistent optimizer state band, and tightly bounded gradient memory. Compared to AdamW (Appendix~\ref{app:mem-profiling}), which shows $16$ GB persistent gradient buffers and $32$ GB optimizer states, MoFaSGD significantly reduces runtime memory pressure. A full quantitative table with GB-level memory usage across all optimizers is provided in Appendix~\ref{app:mem-profiling}.

\begin{figure}[t]
 \centering
\includegraphics[scale=0.25]{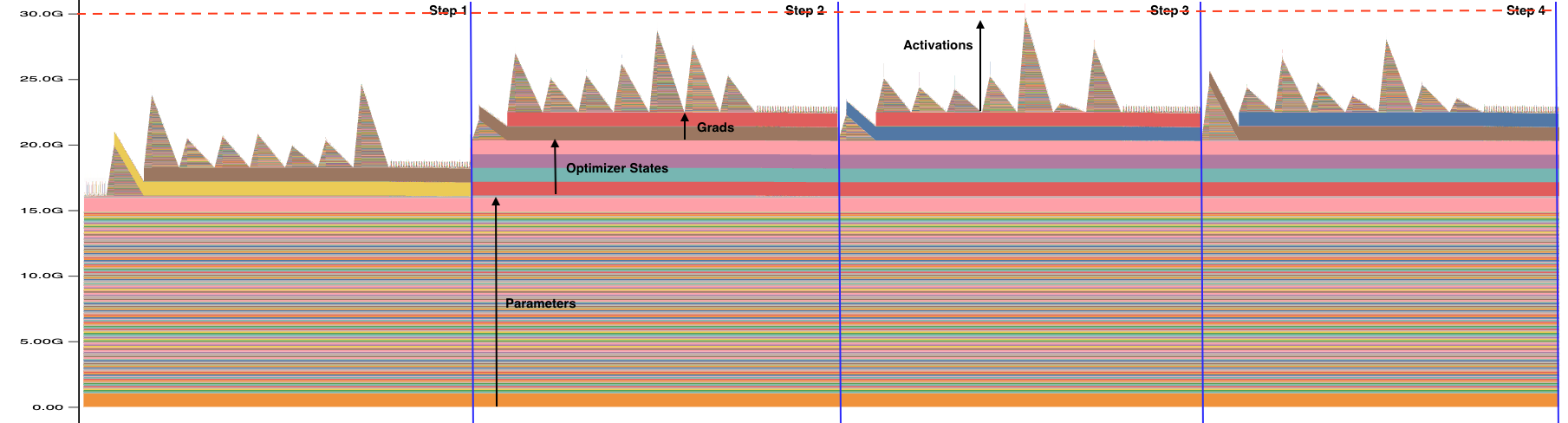}
\caption{GPU memory trace during LLaMA3.1-8B training with MoFaSGD ($r=8$).}
\label{fig:mfsgd-trace}
\end{figure}
\vspace{0.5em}

\subsection{Implementation Details}

\noindent\textbf{Initialization.}
To initialize the momentum factors in MoFaSGD, we perform a full singular value decomposition (SVD) once at the beginning of training, using the gradient from the first step. Specifically, we set $\bm U_0 = \bm U_{\bm G_0}^{1:r}$, $\bm V_0 = \bm V_{\bm G_0}^{1:r}$, and $\bm \Sigma_0 = \bm \Sigma_{\bm G_0}^{1:r}$, corresponding to the top-$r$ components. We apply MoFaSGD exclusively to the linear layers of transformer blocks, as spectral normalization is particularly effective for these layers. Notably, both GaLore~\citep{zhao2024galore} and Muon~\citep{jordan2024muon} adopt similar approach: applying their custom optimizers only to transformer linear layers while using AdamW for embedding weights and 1D layers. We follow this convention and use AdamW (in \texttt{bf16}) for those remaining layers. Consequently, the GPU memory usage for optimizer states—including in fused-GaLore and MoFaSGD—is approximately 4.2\,GB.

\subsubsection*{Gradient Accumulation and Fused Implementation}
Gradient accumulation is essential for training large models with limited hardware memory. For methods such as GaLore~\citep{zhao2024galore} and MoFaSGD, promptly clearing gradient buffers via backward hooks is crucial; otherwise, memory savings akin to LoRA are not realized (see Figure~\ref{fig:non-fused-gal}). To address this, GaLore performs a fused optimizer step, where the gradient is used to update the corresponding parameter immediately during the backward pass, after which the gradient buffer is cleared.

However, this approach is incompatible with gradient accumulation over multiple micro-batches, since the optimizer step must occur only after all gradients are accumulated. To resolve this, we introduce low-rank gradient buffers into the optimizer state and register a backward hook that accumulates low-rank projected gradients instead of performing the optimizer step immediately. For MoFaSGD, updating the momentum factors only requires the low-rank projections $\bm{G}_t \bm{V}_t$, $\bm{U}_t^\top \bm{G}_t$, and $\bm{U}_t^\top \bm{G}_t \bm{V}_t$. We thus implement a temporary low-rank gradient buffer, enabling gradient accumulation while avoiding the need for a persistent full-rank buffer (see memory usage in Figure~\ref{fig:mfsgd-trace}).

A similar strategy is used for GaLore: we implement a gradient-accumulation-friendly fused version by storing full-rank gradients in low-rank form. Since GaLore only needs the projection $\bm Q_t^\top \bm G_t$ to update its subspace momentum, low-rank accumulation suffices.

\subsubsection*{Stateless optimizers}
Recent stateless optimizers such as SWAN~\citep{swan_ma_2024} and SinkGD~\citep{scetbon2025gradient} currently do not have open-source implementations, and we encountered challenges in achieving stable convergence with our preliminary implementations. Nevertheless, given the structural similarity between SWAN and Muon~\citep{jordan2024muon} without momentum, we approximate the memory usage of such approaches by profiling Muon with its momentum buffer disabled. These results are reported in Figure~\ref{fig:mem-break-bar} under the label "SWAN" as a representative proxy for stateless optimizers.

One technical consideration with these methods is their compatibility with gradient accumulation, which is a common practice in low-resource settings. In such scenarios, the fused backward strategy may not be applicable, and persistent gradient buffers are still required in memory (see Figure~\ref{fig:swan}). While stateless optimizers are promising from a conceptual standpoint, this constraint currently poses practical challenges for achieving memory efficiency on par with parameter-efficient fine-tuning approaches like LoRA, or subspace-based approaches like fused GaLore and MoFaSGD.

\section{Conclusion}
We introduced MoFaSGD, a memory-efficient optimizer that uses memory comparable to LoRA-like methods while achieving strong convergence. The core idea is to maintain a low-rank factorization of momentum as the optimizer state and leverage this representation to directly update model parameters at each iteration using spectrally normalized updates, building upon the success of similar approaches in full-training scenarios. We provide a comprehensive theoretical analysis and establish an upper bound on convergence for stochastic non-convex optimization, matching existing lower bounds. Moreover, we empirically demonstrate the effectiveness of our method compared to standard low-rank optimization approaches in both pre-training and post-training setups.

\noindent\textbf{Limitations and Future Directions.} MoFaSGD shows strong empirical performance, but several limitations suggest promising directions for future work. First, as a low-rank subspace optimizer, MoFaSGD may underperform compared to full-rank methods on more complex tasks. Exploring adaptive or dynamic rank selection strategies, beyond the fixed-rank settings used in our study, could help mitigate this gap; for instance, future work could investigate monitoring projection residuals to allocate more rank to layers with higher approximation error or implementing a budget-aware allocation scheme to optimally balance performance and memory. Second, although our theoretical results establish convergence guarantees, they rely on assumptions such as nuclear-norm smoothness. The practical relevance and limitations of such assumptions in deep learning setups remain an open area for study.

\section*{Acknowledgment}
This work was partially supported by NSF CAREER Award \#2239374.




\bibliography{tmlr}

\begin{thebibliography}{66}
\providecommand{\natexlab}[1]{#1}
\providecommand{\url}[1]{\texttt{#1}}
\expandafter\ifx\csname urlstyle\endcsname\relax
  \providecommand{\doi}[1]{doi: #1}\else
  \providecommand{\doi}{doi: \begingroup \urlstyle{rm}\Url}\fi

\bibitem[Anil et~al.(2019)Anil, Gupta, Koren, and Singer]{anil2019memory}
Rohan Anil, Vineet Gupta, Tomer Koren, and Yoram Singer.
\newblock Memory efficient adaptive optimization.
\newblock \emph{Advances in Neural Information Processing Systems}, 32, 2019.

\bibitem[Arjevani et~al.(2023)Arjevani, Carmon, Duchi, Foster, Srebro, and Woodworth]{arjevani2023lower}
Yossi Arjevani, Yair Carmon, John~C Duchi, Dylan~J Foster, Nathan Srebro, and Blake Woodworth.
\newblock Lower bounds for non-convex stochastic optimization.
\newblock \emph{Mathematical Programming}, 199\penalty0 (1):\penalty0 165--214, 2023.

\bibitem[Bernstein \& Newhouse(2024{\natexlab{a}})Bernstein and Newhouse]{bernstein2024modular}
Jeremy Bernstein and Laker Newhouse.
\newblock Modular duality in deep learning.
\newblock \emph{arXiv preprint arXiv:2410.21265}, 2024{\natexlab{a}}.

\bibitem[Bernstein \& Newhouse(2024{\natexlab{b}})Bernstein and Newhouse]{old_bernstein_2024}
Jeremy Bernstein and Laker Newhouse.
\newblock Old optimizer, new norm: An anthology.
\newblock \emph{arXiv.org}, 2024{\natexlab{b}}.
\newblock \doi{10.48550/arxiv.2409.20325}.

\bibitem[Bernstein et~al.(2018)Bernstein, Wang, Azizzadenesheli, and Anandkumar]{bernstein2018signsgd}
Jeremy Bernstein, Yu-Xiang Wang, Kamyar Azizzadenesheli, and Animashree Anandkumar.
\newblock signsgd: Compressed optimisation for non-convex problems.
\newblock In \emph{International Conference on Machine Learning}, pp.\  560--569. PMLR, 2018.

\bibitem[Bernstein et~al.(2023)Bernstein, Mingard, Huang, Azizan, and Yue]{bernstein2023automatic}
Jeremy Bernstein, Chris Mingard, Kevin Huang, Navid Azizan, and Yisong Yue.
\newblock Automatic gradient descent: Deep learning without hyperparameters.
\newblock \emph{arXiv preprint arXiv:2304.05187}, 2023.

\bibitem[Chen et~al.(2024)Chen, Liang, Huang, Real, Wang, Pham, Dong, Luong, Hsieh, Lu, et~al.]{chen2024symbolic}
Xiangning Chen, Chen Liang, Da~Huang, Esteban Real, Kaiyuan Wang, Hieu Pham, Xuanyi Dong, Thang Luong, Cho-Jui Hsieh, Yifeng Lu, et~al.
\newblock Symbolic discovery of optimization algorithms.
\newblock \emph{Advances in neural information processing systems}, 36, 2024.

\bibitem[Dettmers et~al.(2021)Dettmers, Lewis, Shleifer, and Zettlemoyer]{dettmers20218}
Tim Dettmers, Mike Lewis, Sam Shleifer, and Luke Zettlemoyer.
\newblock 8-bit optimizers via block-wise quantization.
\newblock \emph{arXiv preprint arXiv:2110.02861}, 2021.

\bibitem[Dettmers et~al.(2024)Dettmers, Pagnoni, Holtzman, and Zettlemoyer]{dettmers2024qlora}
Tim Dettmers, Artidoro Pagnoni, Ari Holtzman, and Luke Zettlemoyer.
\newblock Qlora: Efficient finetuning of quantized llms.
\newblock \emph{Advances in Neural Information Processing Systems}, 36, 2024.

\bibitem[Duchi et~al.(2011)Duchi, Hazan, and Singer]{Duchi2011Adagrad}
John Duchi, Elad Hazan, and Yoram Singer.
\newblock Adaptive subgradient methods for online learning and stochastic optimization.
\newblock \emph{Journal of Machine Learning Research}, 12:\penalty0 2121--2159, 2011.

\bibitem[Duvvuri et~al.(2024)Duvvuri, Devvrit, Anil, Hsieh, and Dhillon]{duvvuri2024combining}
Sai~Surya Duvvuri, Fnu Devvrit, Rohan Anil, Cho-Jui Hsieh, and Inderjit~S Dhillon.
\newblock Combining axes preconditioners through kronecker approximation for deep learning.
\newblock In \emph{The Twelfth International Conference on Learning Representations}, 2024.

\bibitem[Feinberg et~al.(2024)Feinberg, Chen, Sun, Anil, and Hazan]{feinberg2024sketchy}
Vladimir Feinberg, Xinyi Chen, Y~Jennifer Sun, Rohan Anil, and Elad Hazan.
\newblock Sketchy: Memory-efficient adaptive regularization with frequent directions.
\newblock \emph{Advances in Neural Information Processing Systems}, 36, 2024.

\bibitem[George et~al.(2018)George, Laurent, Bouthillier, Ballas, and Vincent]{george2018fast}
Thomas George, C{\'e}sar Laurent, Xavier Bouthillier, Nicolas Ballas, and Pascal Vincent.
\newblock Fast approximate natural gradient descent in a kronecker factored eigenbasis.
\newblock \emph{Advances in neural information processing systems}, 31, 2018.

\bibitem[Gressmann et~al.(2020)Gressmann, Eaton-Rosen, and Luschi]{gressmann2020improving}
Frithjof Gressmann, Zach Eaton-Rosen, and Carlo Luschi.
\newblock Improving neural network training in low dimensional random bases.
\newblock \emph{Advances in Neural Information Processing Systems}, 33:\penalty0 12140--12150, 2020.

\bibitem[Grosse \& Martens(2016)Grosse and Martens]{grosse2016kronecker}
Roger Grosse and James Martens.
\newblock A kronecker-factored approximate fisher matrix for convolution layers.
\newblock In \emph{International Conference on Machine Learning}, pp.\  573--582. PMLR, 2016.

\bibitem[Gu et~al.(2024)Gu, Tafjord, Kuehl, Haddad, Dodge, and Hajishirzi]{gu2024olmes}
Yuling Gu, Oyvind Tafjord, Bailey Kuehl, Dany Haddad, Jesse Dodge, and Hannaneh Hajishirzi.
\newblock Olmes: A standard for language model evaluations.
\newblock \emph{arXiv preprint arXiv:2406.08446}, 2024.

\bibitem[Gupta et~al.(2018)Gupta, Koren, and Singer]{gupta2018shampoo}
Vineet Gupta, Tomer Koren, and Yoram Singer.
\newblock Shampoo: Preconditioned stochastic tensor optimization.
\newblock In \emph{International Conference on Machine Learning}, pp.\  1842--1850. PMLR, 2018.

\bibitem[Gur-Ari et~al.(2018)Gur-Ari, Roberts, and Dyer]{gur2018gradient}
Guy Gur-Ari, Daniel~A Roberts, and Ethan Dyer.
\newblock Gradient descent happens in a tiny subspace.
\newblock \emph{arXiv preprint arXiv:1812.04754}, 2018.

\bibitem[Hao et~al.(2024)Hao, Cao, and Mou]{hao2024flora}
Yongchang Hao, Yanshuai Cao, and Lili Mou.
\newblock Flora: Low-rank adapters are secretly gradient compressors.
\newblock \emph{arXiv preprint arXiv:2402.03293}, 2024.

\bibitem[Horn \& Johnson(1994)Horn and Johnson]{horn1994topics}
Roger~A Horn and Charles~R Johnson.
\newblock \emph{Topics in matrix analysis}.
\newblock Cambridge university press, 1994.

\bibitem[Houlsby et~al.(2019)Houlsby, Giurgiu, Jastrzebski, Morrone, De~Laroussilhe, Gesmundo, Attariyan, and Gelly]{houlsby2019parameter}
Neil Houlsby, Andrei Giurgiu, Stanislaw Jastrzebski, Bruna Morrone, Quentin De~Laroussilhe, Andrea Gesmundo, Mona Attariyan, and Sylvain Gelly.
\newblock Parameter-efficient transfer learning for nlp.
\newblock In \emph{International conference on machine learning}, pp.\  2790--2799. PMLR, 2019.

\bibitem[Hu et~al.(2021)Hu, Shen, Wallis, Allen-Zhu, Li, Wang, Wang, and Chen]{hu2021lora}
Edward~J Hu, Yelong Shen, Phillip Wallis, Zeyuan Allen-Zhu, Yuanzhi Li, Shean Wang, Lu~Wang, and Weizhu Chen.
\newblock Lora: Low-rank adaptation of large language models.
\newblock \emph{arXiv preprint arXiv:2106.09685}, 2021.

\bibitem[Jordan et~al.(2024{\natexlab{a}})Jordan, Bernstein, Rappazzo, @fernbear.bsky.social, Vlado, Jiacheng, Cesista, Koszarsky, and @Grad62304977]{modded_nanogpt_2024}
Keller Jordan, Jeremy Bernstein, Brendan Rappazzo, @fernbear.bsky.social, Boza Vlado, You Jiacheng, Franz Cesista, Braden Koszarsky, and @Grad62304977.
\newblock modded-nanogpt: Speedrunning the nanogpt baseline, 2024{\natexlab{a}}.
\newblock URL \url{https://github.com/KellerJordan/modded-nanogpt}.

\bibitem[Jordan et~al.(2024{\natexlab{b}})Jordan, Jin, Boza, Jiacheng, Cecista, Newhouse, and Bernstein]{jordan2024muon}
Keller Jordan, Yuchen Jin, Vlado Boza, You Jiacheng, Franz Cecista, Laker Newhouse, and Jeremy Bernstein.
\newblock Muon: An optimizer for hidden layers in neural networks, 2024{\natexlab{b}}.
\newblock URL \url{https://kellerjordan.github.io/posts/muon/}.

\bibitem[Kalajdzievski(2023)]{kalajdzievski2023rank}
Damjan Kalajdzievski.
\newblock A rank stabilization scaling factor for fine-tuning with lora.
\newblock \emph{arXiv preprint arXiv:2312.03732}, 2023.

\bibitem[Kaplan et~al.(2020)Kaplan, McCandlish, Henighan, Brown, Chess, Child, Gray, Radford, Wu, and Amodei]{kaplan2020scaling}
Jared Kaplan, Sam McCandlish, Tom Henighan, Tom~B Brown, Benjamin Chess, Rewon Child, Scott Gray, Alec Radford, Jeffrey Wu, and Dario Amodei.
\newblock Scaling laws for neural language models.
\newblock \emph{arXiv preprint arXiv:2001.08361}, 2020.

\bibitem[Kasimbeg et~al.(2025)Kasimbeg, Schneider, Eschenhagen, Bae, Sastry, Saroufim, Feng, Wright, Yang, Nado, et~al.]{kasimbeg2025accelerating}
Priya Kasimbeg, Frank Schneider, Runa Eschenhagen, Juhan Bae, Chandramouli~Shama Sastry, Mark Saroufim, Boyuan Feng, Less Wright, Edward~Z Yang, Zachary Nado, et~al.
\newblock Accelerating neural network training: An analysis of the algoperf competition.
\newblock \emph{arXiv preprint arXiv:2502.15015}, 2025.

\bibitem[Kingma \& Ba(2015)Kingma and Ba]{Kingma2015Adam}
Diederik~P. Kingma and Jimmy Ba.
\newblock {Adam}: A method for stochastic optimization.
\newblock In \emph{3rd International Conference on Learning Representations (ICLR)}, 2015.

\bibitem[Kopiczko et~al.(2023)Kopiczko, Blankevoort, and Asano]{kopiczko2023vera}
Dawid~J Kopiczko, Tijmen Blankevoort, and Yuki~M Asano.
\newblock Vera: Vector-based random matrix adaptation.
\newblock \emph{arXiv preprint arXiv:2310.11454}, 2023.

\bibitem[Lambert et~al.(2024)Lambert, Morrison, Pyatkin, Huang, Ivison, Brahman, Miranda, Liu, Dziri, Lyu, et~al.]{lambert2024t}
Nathan Lambert, Jacob Morrison, Valentina Pyatkin, Shengyi Huang, Hamish Ivison, Faeze Brahman, Lester James~V Miranda, Alisa Liu, Nouha Dziri, Shane Lyu, et~al.
\newblock T$\backslash$" ulu 3: Pushing frontiers in open language model post-training.
\newblock \emph{arXiv preprint arXiv:2411.15124}, 2024.

\bibitem[Large et~al.(2025)Large, Liu, Huh, Bahng, Isola, and Bernstein]{large2025scalable}
Tim Large, Yang Liu, Jacob Huh, Hyojin Bahng, Phillip Isola, and Jeremy Bernstein.
\newblock Scalable optimization in the modular norm.
\newblock \emph{Advances in Neural Information Processing Systems}, 37:\penalty0 73501--73548, 2025.

\bibitem[Lester et~al.(2021)Lester, Al-Rfou, and Constant]{lester2021power}
Brian Lester, Rami Al-Rfou, and Noah Constant.
\newblock The power of scale for parameter-efficient prompt tuning.
\newblock \emph{arXiv preprint arXiv:2104.08691}, 2021.

\bibitem[Li et~al.(2023)Li, Chen, and Zhu]{li2023memory}
Bingrui Li, Jianfei Chen, and Jun Zhu.
\newblock Memory efficient optimizers with 4-bit states.
\newblock \emph{Advances in Neural Information Processing Systems}, 36:\penalty0 15136--15171, 2023.

\bibitem[Li \& Liang(2021)Li and Liang]{li2021prefix}
Xiang~Lisa Li and Percy Liang.
\newblock Prefix-tuning: Optimizing continuous prompts for generation.
\newblock \emph{arXiv preprint arXiv:2101.00190}, 2021.

\bibitem[Lialin et~al.(2023)Lialin, Muckatira, Shivagunde, and Rumshisky]{lialin2023relora}
Vladislav Lialin, Sherin Muckatira, Namrata Shivagunde, and Anna Rumshisky.
\newblock Relora: High-rank training through low-rank updates.
\newblock In \emph{The Twelfth International Conference on Learning Representations}, 2023.

\bibitem[Liu \& Nocedal(1989)Liu and Nocedal]{liu1989limited}
Dong~C Liu and Jorge Nocedal.
\newblock On the limited memory bfgs method for large scale optimization.
\newblock \emph{Mathematical programming}, 45\penalty0 (1):\penalty0 503--528, 1989.

\bibitem[Liu et~al.(2025)Liu, Su, Yao, Jiang, Lai, Du, Qin, Xu, Lu, Yan, et~al.]{liu2025muon}
Jingyuan Liu, Jianlin Su, Xingcheng Yao, Zhejun Jiang, Guokun Lai, Yulun Du, Yidao Qin, Weixin Xu, Enzhe Lu, Junjie Yan, et~al.
\newblock Muon is scalable for llm training.
\newblock \emph{arXiv preprint arXiv:2502.16982}, 2025.

\bibitem[Liu et~al.(2024)Liu, Wang, Yin, Molchanov, Wang, Cheng, and Chen]{liu2024dora}
Shih-Yang Liu, Chien-Yi Wang, Hongxu Yin, Pavlo Molchanov, Yu-Chiang~Frank Wang, Kwang-Ting Cheng, and Min-Hung Chen.
\newblock Dora: Weight-decomposed low-rank adaptation.
\newblock \emph{arXiv preprint arXiv:2402.09353}, 2024.

\bibitem[Loshchilov et~al.(2017)Loshchilov, Hutter, et~al.]{loshchilov2017fixing}
Ilya Loshchilov, Frank Hutter, et~al.
\newblock Fixing weight decay regularization in adam.
\newblock \emph{arXiv preprint arXiv:1711.05101}, 5, 2017.

\bibitem[Luo et~al.(2024)Luo, Yu, and Li]{luo2024badam}
Qijun Luo, Hengxu Yu, and Xiao Li.
\newblock Badam: A memory efficient full parameter training method for large language models.
\newblock \emph{arXiv preprint arXiv:2404.02827}, 2024.

\bibitem[Luo et~al.(2023)Luo, Ren, Zheng, Jiang, Jiang, and You]{luo2023came}
Yang Luo, Xiaozhe Ren, Zangwei Zheng, Zhuo Jiang, Xin Jiang, and Yang You.
\newblock Came: Confidence-guided adaptive memory efficient optimization.
\newblock \emph{arXiv preprint arXiv:2307.02047}, 2023.

\bibitem[Lv et~al.(2023{\natexlab{a}})Lv, Yan, Guo, Lv, and Qiu]{lv2023adalomo}
Kai Lv, Hang Yan, Qipeng Guo, Haijun Lv, and Xipeng Qiu.
\newblock Adalomo: Low-memory optimization with adaptive learning rate.
\newblock \emph{arXiv preprint arXiv:2310.10195}, 2023{\natexlab{a}}.

\bibitem[Lv et~al.(2023{\natexlab{b}})Lv, Yang, Liu, Gao, Guo, and Qiu]{lv2023full}
Kai Lv, Yuqing Yang, Tengxiao Liu, Qinghui Gao, Qipeng Guo, and Xipeng Qiu.
\newblock Full parameter fine-tuning for large language models with limited resources.
\newblock \emph{arXiv preprint arXiv:2306.09782}, 2023{\natexlab{b}}.

\bibitem[Ma et~al.(2024)Ma, Gong, Scetbon, and Meeds]{swan_ma_2024}
Chao Ma, Wenbo Gong, Meyer Scetbon, and Edward Meeds.
\newblock Swan: Preprocessing sgd enables adam-level performance on llm training with significant memory reduction.
\newblock \emph{arXiv preprint arXiv:2412.13148}, 2024.

\bibitem[Martens(2020)]{martens2020new}
James Martens.
\newblock New insights and perspectives on the natural gradient method.
\newblock \emph{Journal of Machine Learning Research}, 21\penalty0 (146):\penalty0 1--76, 2020.

\bibitem[Martens \& Grosse(2015)Martens and Grosse]{martens2015optimizing}
James Martens and Roger Grosse.
\newblock Optimizing neural networks with kronecker-factored approximate curvature.
\newblock In \emph{International conference on machine learning}, pp.\  2408--2417. PMLR, 2015.

\bibitem[Modoranu et~al.(2024)Modoranu, Safaryan, Malinovsky, Kurtic, Robert, Richtárik, and Alistarh]{microadam_modoranu_2024}
Ionut-Vlad Modoranu, Mher Safaryan, Grigory Malinovsky, Eldar Kurtic, Thomas Robert, Peter Richtárik, and Dan Alistarh.
\newblock Microadam: Accurate adaptive optimization with low space overhead and provable convergence.
\newblock \emph{arXiv.org}, 2024.
\newblock \doi{10.48550/arxiv.2405.15593}.

\bibitem[Morwani et~al.(2024)Morwani, Shapira, Vyas, Malach, Kakade, and Janson]{morwani2024new}
Depen Morwani, Itai Shapira, Nikhil Vyas, Eran Malach, Sham Kakade, and Lucas Janson.
\newblock A new perspective on shampoo's preconditioner.
\newblock \emph{arXiv preprint arXiv:2406.17748}, 2024.

\bibitem[Ouyang et~al.(2022)Ouyang, Wu, Jiang, Almeida, Wainwright, Mishkin, Zhang, Agarwal, Slama, Ray, et~al.]{ouyang2022training}
Long Ouyang, Jeffrey Wu, Xu~Jiang, Diogo Almeida, Carroll Wainwright, Pamela Mishkin, Chong Zhang, Sandhini Agarwal, Katarina Slama, Alex Ray, et~al.
\newblock Training language models to follow instructions with human feedback.
\newblock \emph{Advances in neural information processing systems}, 35:\penalty0 27730--27744, 2022.

\bibitem[Penedo et~al.(2025)Penedo, Kydl{\'\i}{\v{c}}ek, Lozhkov, Mitchell, Raffel, Von~Werra, Wolf, et~al.]{penedo2025fineweb}
Guilherme Penedo, Hynek Kydl{\'\i}{\v{c}}ek, Anton Lozhkov, Margaret Mitchell, Colin~A Raffel, Leandro Von~Werra, Thomas Wolf, et~al.
\newblock The fineweb datasets: Decanting the web for the finest text data at scale.
\newblock \emph{Advances in Neural Information Processing Systems}, 37:\penalty0 30811--30849, 2025.

\bibitem[Rajbhandari et~al.(2020)Rajbhandari, Rasley, Ruwase, and He]{rajbhandari2020zero}
Samyam Rajbhandari, Jeff Rasley, Olatunji Ruwase, and Yuxiong He.
\newblock Zero: Memory optimizations toward training trillion parameter models.
\newblock In \emph{SC20: International Conference for High Performance Computing, Networking, Storage and Analysis}, pp.\  1--16. IEEE, 2020.

\bibitem[Refael et~al.(2024)Refael, Svirsky, Shustin, Huleihel, and Lindenbaum]{refael2024adarankgrad}
Yehonathan Refael, Jonathan Svirsky, Boris Shustin, Wasim Huleihel, and Ofir Lindenbaum.
\newblock Adarankgrad: Adaptive gradient-rank and moments for memory-efficient llms training and fine-tuning.
\newblock \emph{arXiv preprint arXiv:2410.17881}, 2024.

\bibitem[Robert et~al.(2024)Robert, Safaryan, Modoranu, and Alistarh]{ldadam_robert_2024}
Thomas Robert, Mher Safaryan, Ionut-Vlad Modoranu, and Dan Alistarh.
\newblock Ldadam: Adaptive optimization from low-dimensional gradient statistics.
\newblock \emph{arXiv.org}, 2024.

\bibitem[Scetbon et~al.(2025)Scetbon, Ma, Gong, and Meeds]{scetbon2025gradient}
Meyer Scetbon, Chao Ma, Wenbo Gong, and Edward Meeds.
\newblock Gradient multi-normalization for stateless and scalable llm training.
\newblock \emph{arXiv preprint arXiv:2502.06742}, 2025.

\bibitem[Shazeer \& Stern(2018)Shazeer and Stern]{shazeer2018adafactor}
Noam Shazeer and Mitchell Stern.
\newblock Adafactor: Adaptive learning rates with sublinear memory cost.
\newblock In \emph{International Conference on Machine Learning}, pp.\  4596--4604. PMLR, 2018.

\bibitem[Tieleman \& Hinton(2012)Tieleman and Hinton]{Tieleman2012RMSprop}
T.~Tieleman and G.~Hinton.
\newblock Lecture 6.5 - rmsprop: Divide the gradient by a running average of its recent magnitude.
\newblock Coursera: Neural Networks for Machine Learning, 2012.

\bibitem[Vogels et~al.(2020)Vogels, Karimireddy, and Jaggi]{vogels2020practical}
Thijs Vogels, Sai~Praneeth Karimireddy, and Martin Jaggi.
\newblock Practical low-rank communication compression in decentralized deep learning.
\newblock \emph{Advances in Neural Information Processing Systems}, 33:\penalty0 14171--14181, 2020.

\bibitem[Vyas et~al.(2024)Vyas, Morwani, Zhao, Kwun, Shapira, Brandfonbrener, Janson, and Kakade]{vyas2024soap}
Nikhil Vyas, Depen Morwani, Rosie Zhao, Mujin Kwun, Itai Shapira, David Brandfonbrener, Lucas Janson, and Sham Kakade.
\newblock Soap: Improving and stabilizing shampoo using adam.
\newblock \emph{arXiv preprint arXiv:2409.11321}, 2024.

\bibitem[Wang(2018)]{wang2018glue}
Alex Wang.
\newblock Glue: A multi-task benchmark and analysis platform for natural language understanding.
\newblock \emph{arXiv preprint arXiv:1804.07461}, 2018.

\bibitem[Wang et~al.(2023)Wang, Ivison, Dasigi, Hessel, Khot, Chandu, Wadden, MacMillan, Smith, Beltagy, et~al.]{wang2023far}
Yizhong Wang, Hamish Ivison, Pradeep Dasigi, Jack Hessel, Tushar Khot, Khyathi Chandu, David Wadden, Kelsey MacMillan, Noah~A Smith, Iz~Beltagy, et~al.
\newblock How far can camels go? exploring the state of instruction tuning on open resources.
\newblock \emph{Advances in Neural Information Processing Systems}, 36:\penalty0 74764--74786, 2023.

\bibitem[Yang et~al.(2023)Yang, Simon, and Bernstein]{yang2023spectral}
Greg Yang, James~B Simon, and Jeremy Bernstein.
\newblock A spectral condition for feature learning.
\newblock \emph{arXiv preprint arXiv:2310.17813}, 2023.

\bibitem[Zaken et~al.(2021)Zaken, Ravfogel, and Goldberg]{zaken2021bitfit}
Elad~Ben Zaken, Shauli Ravfogel, and Yoav Goldberg.
\newblock Bitfit: Simple parameter-efficient fine-tuning for transformer-based masked language-models.
\newblock \emph{arXiv preprint arXiv:2106.10199}, 2021.

\bibitem[Zhang et~al.(2023)Zhang, Chen, Bukharin, Karampatziakis, He, Cheng, Chen, and Zhao]{zhang2023adalora}
Qingru Zhang, Minshuo Chen, Alexander Bukharin, Nikos Karampatziakis, Pengcheng He, Yu~Cheng, Weizhu Chen, and Tuo Zhao.
\newblock Adalora: Adaptive budget allocation for parameter-efficient fine-tuning.
\newblock \emph{arXiv preprint arXiv:2303.10512}, 2023.

\bibitem[Zhao et~al.(2024{\natexlab{a}})Zhao, Zhang, Chen, Wang, Anandkumar, and Tian]{zhao2024galore}
Jiawei Zhao, Zhenyu Zhang, Beidi Chen, Zhangyang Wang, Anima Anandkumar, and Yuandong Tian.
\newblock Galore: Memory-efficient llm training by gradient low-rank projection.
\newblock \emph{arXiv preprint arXiv:2403.03507}, 2024{\natexlab{a}}.

\bibitem[Zhao et~al.(2024{\natexlab{b}})Zhao, Li, Gu, Zheng, K{\"o}lker, Wang, and Yuan]{zhao2024adapprox}
Pengxiang Zhao, Ping Li, Yingjie Gu, Yi~Zheng, Stephan~Ludger K{\"o}lker, Zhefeng Wang, and Xiaoming Yuan.
\newblock Adapprox: Adaptive approximation in adam optimization via randomized low-rank matrices.
\newblock \emph{arXiv preprint arXiv:2403.14958}, 2024{\natexlab{b}}.

\bibitem[Zhu et~al.(2024)Zhu, Zhang, Cong, Liu, Park, Chandra, Long, Pan, Wang, and Lee]{zhu2024apollo}
Hanqing Zhu, Zhenyu Zhang, Wenyan Cong, Xi~Liu, Sem Park, Vikas Chandra, Bo~Long, David~Z Pan, Zhangyang Wang, and Jinwon Lee.
\newblock Apollo: Sgd-like memory, adamw-level performance.
\newblock \emph{arXiv preprint arXiv:2412.05270}, 2024.

\end{thebibliography}
\bibliographystyle{tmlr}

\clearpage
\appendix

\section{Additional Related Works}
\label{sec:app-rw}
\noindent\textbf{Parameter-Efficient Fine-Tuning (PEFT).} To reduce the memory associated with full-parameter updates, PEFT methods constrain updates to a small subset of parameters or add minimal trainable components. Adapters introduce small bottleneck layers into the model \citep{houlsby2019parameter}, while the popular LoRA technique \citep{hu2021lora} injects trainable low-rank matrices into existing weight layers, drastically reducing memory by training only these adapters \citep{hu2021lora}. Variants include AdaLoRA for adaptive rank allocation \citep{zhang2023adalora}, VeRA using shared low-rank matrices \citep{kopiczko2023vera}, and DoRA decomposing updates into magnitude/direction \citep{liu2024dora}. Other PEFT methods tune only biases (BitFit \citep{zaken2021bitfit}) or learn input embeddings (prompt tuning \citep{li2021prefix, lester2021power}). QLoRA \citep{dettmers2024qlora} further combines LoRA with 4-bit weight quantization. While memory-efficient, PEFT methods generally keep the base model fixed and rely on the capacity of the trained low-rank components.

\subsubsection*{Second-Order and Preconditioning Methods.} These methods aim primarily to accelerate convergence and improve optimization performance by incorporating curvature information, offering potentially more effective descent directions compared to first-order methods \citep{george2018fast,duvvuri2024combining,gupta2018shampoo}. Exact second-order information (e.g., using the full Hessian in Newton's method or the full gradient covariance in full-matrix Adagrad \citep{Duchi2011Adagrad, duvvuri2024combining}) is computationally infeasible for large models due to the prohibitive cost ($\mathcal{O}(d^2)$ memory, $\mathcal{O}(d^3)$ computation) of storing and manipulating the required matrices \citep{george2018fast,duvvuri2024combining, gupta2018shampoo}. Therefore, practical methods rely on approximations to make harnessing second-order information tractable. Quasi-Newton methods like LBFGS \citep{liu1989limited, duvvuri2024combining} build implicit Hessian approximations but can still be memory-intensive \citep{duvvuri2024combining}.

A major family of approximations leverages Kronecker product factorizations. KFAC (Kronecker-Factored Approximate Curvature) \citep{martens2015optimizing} approximates the Fisher Information Matrix (an approximation of the Hessian \citep{martens2020new}) block-diagonally using Kronecker products specific to network layer types \citep{martens2015optimizing,grosse2016kronecker}. Shampoo \citep{gupta2018shampoo}, motivated by full-matrix Adagrad \citep{gupta2018shampoo, anil2019memory}, uses Kronecker products of gradient statistics ($\bm G \bm G^\top, \bm G^\top \bm G$) as a preconditioner \citep{gupta2018shampoo}. While powerful, these factored approximations require computing matrix roots or inverses, which are computationally demanding and often amortized over multiple steps, potentially using stale curvature estimates \citep{gupta2018shampoo, george2018fast}.

Refinements seek to improve accuracy or efficiency. EKFAC \citep{george2018fast} builds on KFAC by performing cheaper diagonal updates within the Kronecker-factored eigenbasis (KFE), yielding a provably better Fisher approximation \citep{george2018fast}. CASPR \citep{duvvuri2024combining} uses Kronecker sums, aiming for a potentially better Adagrad approximation than Shampoo's Kronecker product \citep{duvvuri2024combining}. SOAP \citep{vyas2024soap} runs AdamW within Shampoo's eigenbasis, aiming for improved stability and fewer hyperparameters, especially when the basis update is infrequent \citep{vyas2024soap}. These methods highlight the ongoing effort to balance the power of second-order information with computational feasibility \citep{george2018fast, vyas2024soap}.

\noindent\textbf{On-the-Fly and Partial Updates.} These techniques reduce memory by avoiding large buffers or distributing states. LOMO \citep{lv2023full} applies updates immediately, AdaLOMO \citep{lv2023adalomo} adds adaptivity, BAdam \citep{luo2024badam} uses weight blocks, and ZeRO \citep{rajbhandari2020zero} shards states in distributed settings.

\section{Additional Preliminaries}
\label{sec:add-pre}
\noindent\textbf{Adaptive Optimization Methods.} Adaptive optimization methods have become essential for training deep neural networks. By adjusting learning rates on a per-parameter (or per-dimension) basis using information from past gradients and the loss curvature, they often lead to significantly faster convergence than vanilla stochastic gradient descent (SGD) in practice. This line of work began with Adagrad~\citep{Duchi2011Adagrad}, which proposed leveraging the gradient preconditioner $\bm P_t = \sum_{i=1}^t \bm g_i \bm g_i^{\top} \in \mbb R^{mn \times mn}$ and performing the update rule as $\bm w_{t+1} = \bm w_t - \eta \bm P_t^{-\frac{1}{2}} \bm g_t$, where in our notation $\bm g_i = \operatorname{Vec}(\bm G_i) \in \mbb R^{mn}$ and $\bm w_t = \operatorname{Vec}(\bm W_t)$. Note that considering only the diagonal terms of $\bm P_t$ and using the EMA of gradient covariances (e.g., $\bm P_t = \sum_{i=1}^t \beta_2^{t-i} \bm g_i \bm g_i^{\top}$) led to an early variant of RMSprop~\citep{Tieleman2012RMSprop}, and further using EMA of the gradients themselves, $ \hat{\bm g}_t = \sum_{i=1}^t \beta_1^{t-i} \bm g_t$, instead of the plain $\bm g_t$, yields the well-known Adam~\citep{Kingma2015Adam}. However, the mentioned methods are limited forms of second-order methods, where the preconditioner matrix is restricted to being diagonal. Considering either the original Adagrad~\citep{Duchi2011Adagrad} perspective or second-order optimization methods such as Newton's method, moving beyond diagonal preconditioners is a natural step toward achieving even faster convergence. However, storing a non-diagonal preconditioner is often not feasible due to the size of DNNs, requiring $O(m^2 n^2)$ memory. Shampoo~\citep{gupta2018shampoo} proposed maintaining a Kronecker approximation $\bm P_t \sim \bm R_t \otimes \bm L_t$, where $\bm L_t \in \mbb R^{m \times m}$ and $\bm R_t \in \mbb R^{n \times n}$, thereby reducing the number of parameters required for storing the preconditioning matrix to $O(m^2 + n^2)$. \citet{morwani2024new} showed that Shampoo~\citep{gupta2018shampoo} closely approximates the full Adagrad preconditioner, unified all aforementioned methods, and argued that the gradient covariance $\bm g_t \bm g_t^{\top}$ closely approximates the Gauss-Newton components of the Hessian at $\bm w_t$, thereby drawing an interesting connection between Adagrad~\citep{Duchi2011Adagrad} and second-order optimization methods.

\section{Experimental Details}

This section provides detailed configurations, hyperparameters, tuning procedures, and dataset information for the experiments presented in the main paper, aiming to ensure reproducibility.

\subsection{General Implementation Details}

\begin{itemize}
    \item \textbf{Software:} Experiments were implemented using standard libraries for deep learning, including PyTorch, Hugging Face Transformers, and Accelerate. Specific library versions are detailed in the code repository. 
    \item \textbf{Hardware:} All experiments were conducted on NVIDIA A100 GPUs. The number of GPUs may have varied slightly depending on the specific experimental setup (e.g., $4$ GPUs for Tulu3 tuning).
    \item \textbf{Code:} The implementation for MoFaSGD is available at \url{https://github.com/AnonCode1/MFSGD.git}.
\end{itemize}

\subsection{Pre-training: NanoGPT Speedrun}
\label{sec:app-nano}

\noindent\textbf{Hyperparameters and Tuning.} For this benchmark, Muon hyperparameters were kept at the tuned defaults provided by~\citet{modded_nanogpt_2024}. Learning rates for AdamW, GaLore, and MoFaSGD were tuned via grid search over $\{ 1e-4, 2e-4, 3e-4, 5e-4, 8e-4, 1e-3, 3e-3, 5e-3, 8e-3, 1e-2, 2e-2, 5e-2\}$. MoFaSGD's momentum decay $\beta$ was tuned over $\{ 0.5, 0.85, 0.90, 0.95\}$. GaLore's SVD frequency was tuned over $\{ 10, 25, 75, 150, 300 \}$. The best-performing hyperparameters based on final validation perplexity were selected.

Key hyperparameters selected for the NanoGPT pre-training experiments are summarized in Table~\ref{table:nanogpt_hparams}.

\begin{table}[t]
\centering
\caption{Final Selected Hyperparameters for NanoGPT Pre-training Experiment ($0.73B$ tokens)}
\label{table:nanogpt_hparams}
\begin{tabular}{@{}llccc@{}}
\toprule
Optimizer         & 
Rank& Learning Rate (LR) & SVD Freq. & MoFaSGD $\beta$ \\ \midrule
Muon    & -               & 5e-2               & N/A       & N/A               \\
AdamW           & -               & 2e-3               & N/A       & N/A               \\ \midrule
GaLore          & $r=16$            & 2e-2               & 150       & N/A               \\
                & $r=32$            & 8e-3               & 75        & N/A               \\
                & $r=128$           & 8e-3               & 75        & N/A               \\ \midrule
MoFaSGD (Ours)  & $r=16$            & 1e-3               & N/A       & 0.85              \\
                & $r=32$            & 5e-4               & N/A       & 0.85              \\
                & $r=128$           & 3e-4               & N/A       & 0.85              \\ \midrule
\multicolumn{2}{@{}l}{Batch Size (Tokens)} & \multicolumn{3}{c}{524,288 tokens} \\
\multicolumn{2}{@{}l}{LR Schedule}         & \multicolumn{3}{c}{Stable then linear decay with cool-down of $0.4$} \\ \bottomrule
\end{tabular}
\end{table}

\subsection{NLU Fine-tuning: GLUE Benchmark}
\label{sec:app-glue}

\noindent\textbf{Hyperparameters and Tuning.} The experimental setup, including hyperparameters for baseline optimizers (AdamW, GaLore, LoRA+AdamW), directly follows the configuration reported in~\citet{zhao2024galore} for RoBERTa-Base on GLUE. For our method, MoFaSGD, learning rates were tuned via grid search over $\{ 1e-5, 2e-5, 5e-5, 1e-4, 5e-4  \}$ for each task and rank combination ($r=4$ and $r=8$), and the batch size was fixed to $16$ for all experiments. The MoFaSGD momentum decay $\beta$ was also kept fixed at $0.95$. The learning rate yielding the best validation accuracy on each specific task was selected. The final selected hyperparameters for MoFaSGD across the evaluated GLUE tasks are presented in Table~\ref{table:glue_mofasgd_hparam}.

\begin{table}[htbp]
\centering
\caption{Final Selected MoFaSGD Hyperparameters for GLUE Tasks (RoBERTa-Base).}
\label{table:glue_mofasgd_hparam}
\begin{tabular}{@{}lcccccccc@{}} 
\toprule
Hyperparameter      & Rank     & MNLI   & QQP    & SST-2  & MRPC   & COLA   & QNLI   & RTE    \\ \midrule 
Learning Rate (LR)  & $r=4$    & 1e-4   & 5e-5   & 5e-5   & 1e-4   & 2e-5   & 2e-5   & 5e-5   \\ 
                    & $r=8$    & 5e-5   & 5e-5   & 2e-5   & 5e-5   & 1e-5   & 1e-5   & 5e-5   \\ \midrule 
MoFaSGD $\beta$     & $r=4,\ 8$ & \multicolumn{7}{c}{0.95} \\ \midrule 
Batch Size          & $r=4,\ 8$ & \multicolumn{7}{c}{16} \\ 
Epochs              & $r=4,\ 8$ & \multicolumn{7}{c}{15} \\ \bottomrule
\end{tabular}
\end{table}

\subsection{Instruction Tuning: Tulu3}
\label{sec:app-tulu}

\noindent\textbf{Hyperparameters and Tuning.} The setup largely follows~\citet{lambert2024t}. Learning rates for all optimizers were selected from the grid $\{1e-5, 5e-5, 1e-4 , 5e-4 , 1e-3\}$. The final LR for each method was chosen based on final validation loss on a held-out subset of the tulu-3-sft mixture. We used $5\%$ of the sampled dataset for validation. Final selected hyperparameters for the Tulu3 instruction tuning experiment are summarized in Table~\ref{table:tulu_hparams}.

\begin{table}[htbp] 
\centering
\caption{Final Selected Hyperparameters for Tulu-3 Instruction Tuning}
\label{table:tulu_hparams}
\begin{tabular}{@{}lccc@{}}
\toprule
Hyperparameter        & GaLore     & LoRA      & MoFaSGD (Ours) \\ \midrule
Learning Rate (LR)    & 5e-5       & 5e-5      & 1e-4 \\
Rank ($r$)            & 8          & 8         & 8 \\
LoRA Alpha            & N/A        & 16        & N/A \\
GaLore SVD Freq.      & 200        & N/A       & N/A \\
MoFaSGD $\beta$       & N/A        & N/A       & 0.95 \\
Effective Batch Size  & \multicolumn{3}{c}{128} \\
Epochs                & \multicolumn{3}{c}{1} \\
\bottomrule
\end{tabular}
\end{table}
\vspace{1em}
\noindent\textbf{Momentum Spectral Analysis.}
\vspace{0.5em}
\begin{itemize}
    \item \textbf{Methodology:} Analysis was performed on the Tulu3 instruction-tuning run using the default optimizer AdamW with a learning rate of $1e-5$ and a batch size of $128$ for one epoch. We directly inspected the AdamW state buffer for the first-moment EMA ($\bm M_t$) (\texttt{state[`exp\_avg`]}) after each backward pass periodically. SVD was performed on these buffers, and the energy ratio for rank $r$ was computed as $\frac{\sum_{i=1}^{r}\sigma_{i,\bm M_{t}}^{2}}{\|\bm M_{t}\|_{\mathrm{F}}^{2}}$.
\end{itemize}

\subsection{Training Loss Curves}
\label{app:curve-train}
The training loss curves in Figures~\ref{fig:mnli_loss} and~\ref{fig:tulu_loss} illustrate MoFaSGD's convergence behavior. In both the simpler GLUE fine-tuning and the complex Tulu3 instruction-tuning, MoFaSGD consistently achieves lower training loss compared to LoRA and GaLore. This suggests more effective optimization dynamics, potentially stemming from the synergistic effect of adaptive low-rank momentum factorization and the spectrally normalized updates, which might offer better preconditioning than the standard updates used in LoRA or the subspace accumulation in GaLore.

\begin{figure}[htbp]
    \centering
    \begin{subfigure}[b]{0.495\textwidth}
        \centering
        \includegraphics[width=0.9\linewidth]{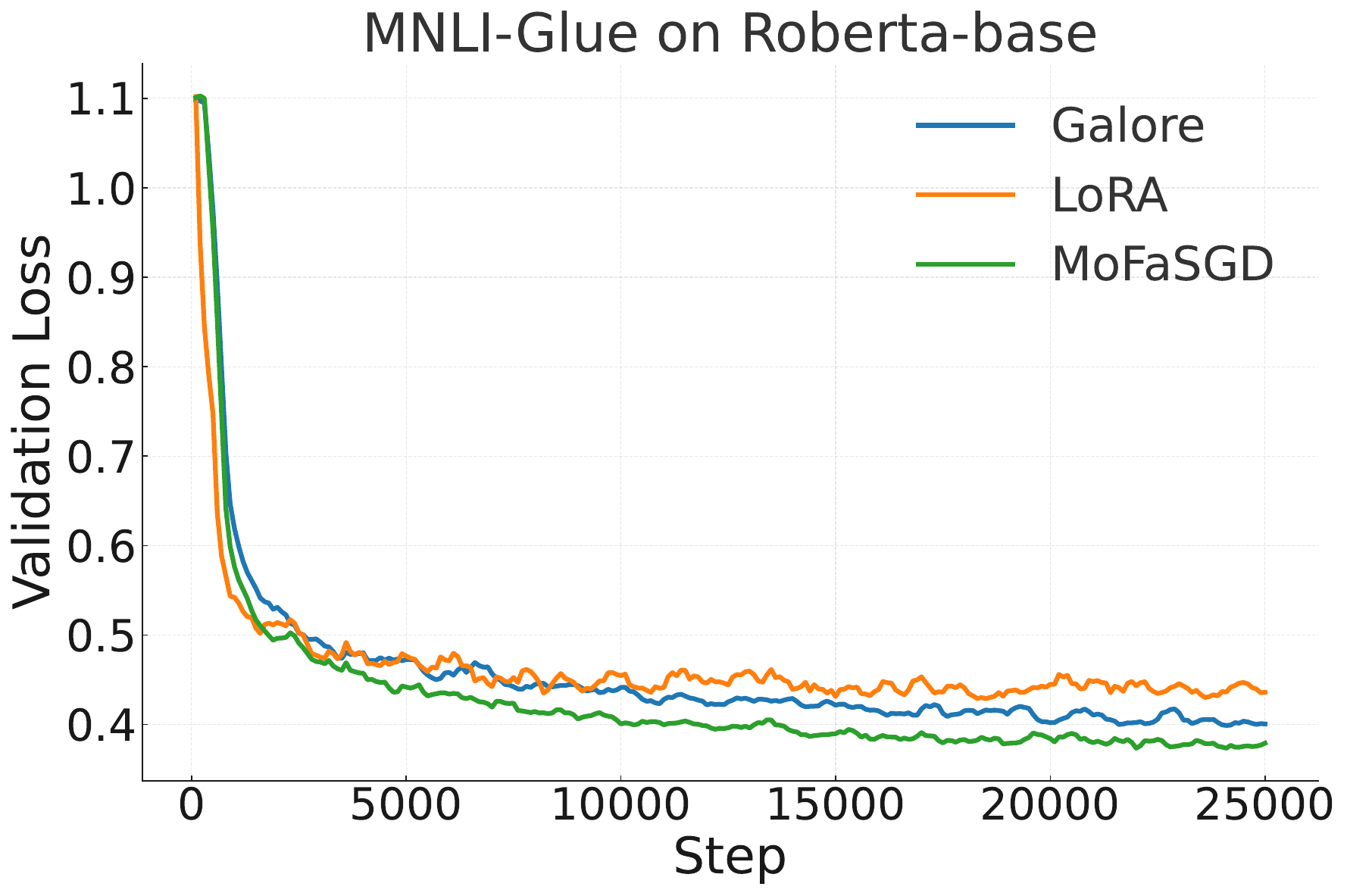}
        \caption{MNLI Training Loss}
        \label{fig:mnli_loss}
    \end{subfigure}
    \hfill
    \begin{subfigure}[b]{0.495\textwidth}
        \centering
        \includegraphics[width=0.9\linewidth]{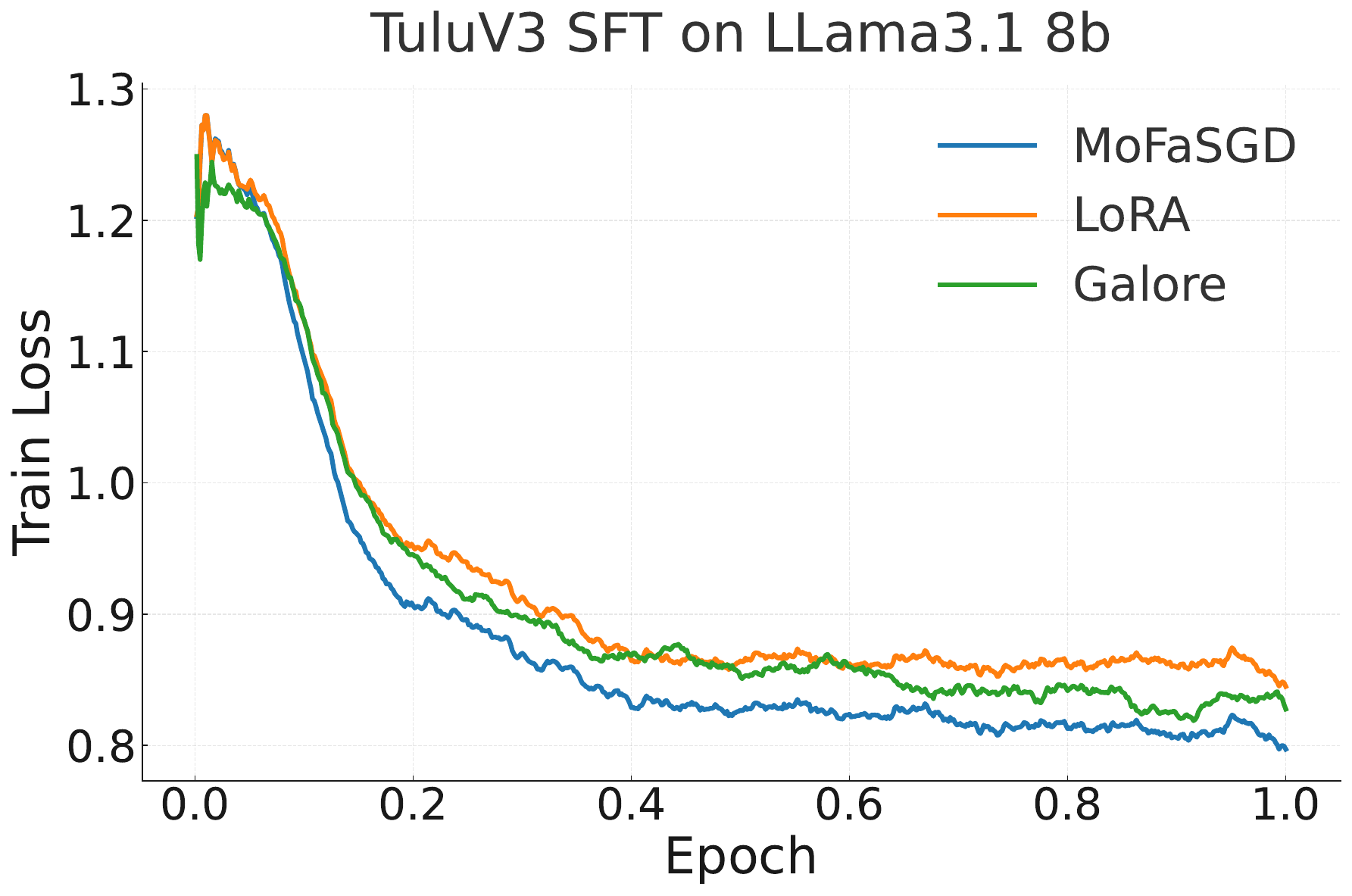}
        \caption{TuluV3 Training Loss}
        \label{fig:tulu_loss}
    \end{subfigure}
    \caption{Training loss curves for post-training setups.}
    \label{fig:combined_three_figures}
\end{figure}

\subsection{Memory Profiling Details}
\label{app:mem-profiling}
We provide full memory trace visualizations and a quantitative breakdown of memory usage across all optimizer configurations (LLaMA3.1-8B, BF16, no activation checkpointing, batch size $1$, gradient accumulation $8$). These results correspond to the experiments summarized in Figure~\ref{fig:mem-break-bar} and Section~\ref{sec:memory-breakdown}.

\textbf{Profiling Setup.}
Memory traces were collected using PyTorch’s native CUDA memory snapshot utility \texttt{torch.cuda.memory\_snapshot()} at each training step. Each trace illustrates the evolution of GPU memory usage over four training steps. Manual annotations highlight key memory components.

\textbf{Observations.}  
As shown in the traces, MoFaSGD maintains a compact and stable memory profile, with minimal transient gradients and compressed optimizer state bands. AdamW and GaLore (non-fused) show persistent high memory usage due to full-rank gradient and optimizer state buffers. LoRA adds moderate adapter overhead, while fused GaLore shows improvements over its non-fused version.
\begin{figure}[H]
    \centering
    \includegraphics[width=\linewidth]{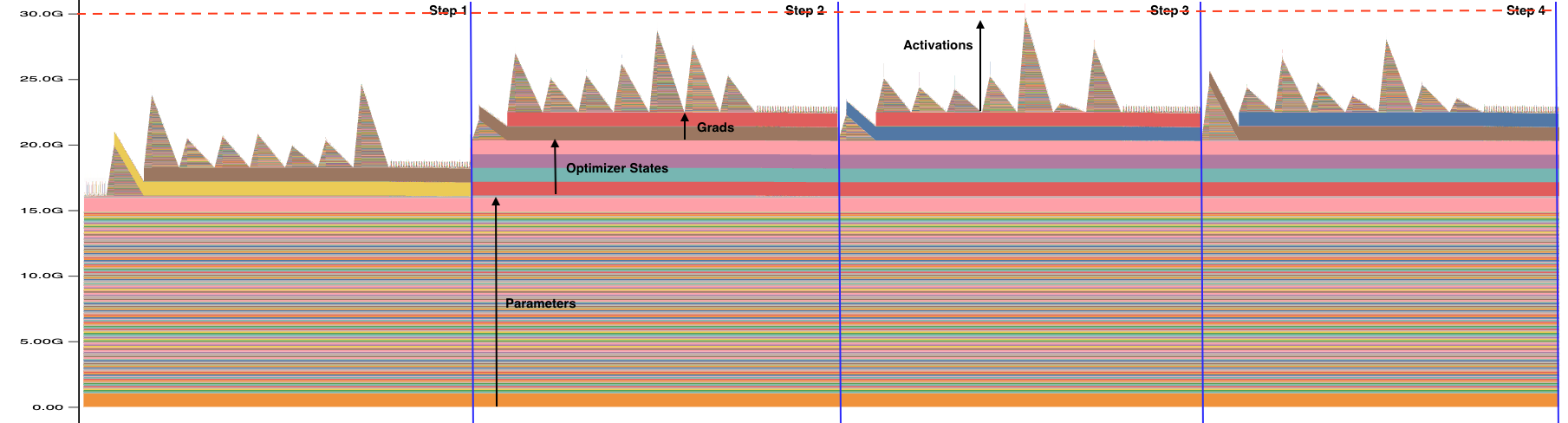}
    \caption{MoFaSGD ($r=8$): Compact optimizer states and minimal gradient spikes. Total memory $\sim$29.4 GB.}
\end{figure}
\begin{figure}[H]
    \centering
    \includegraphics[width=\linewidth]{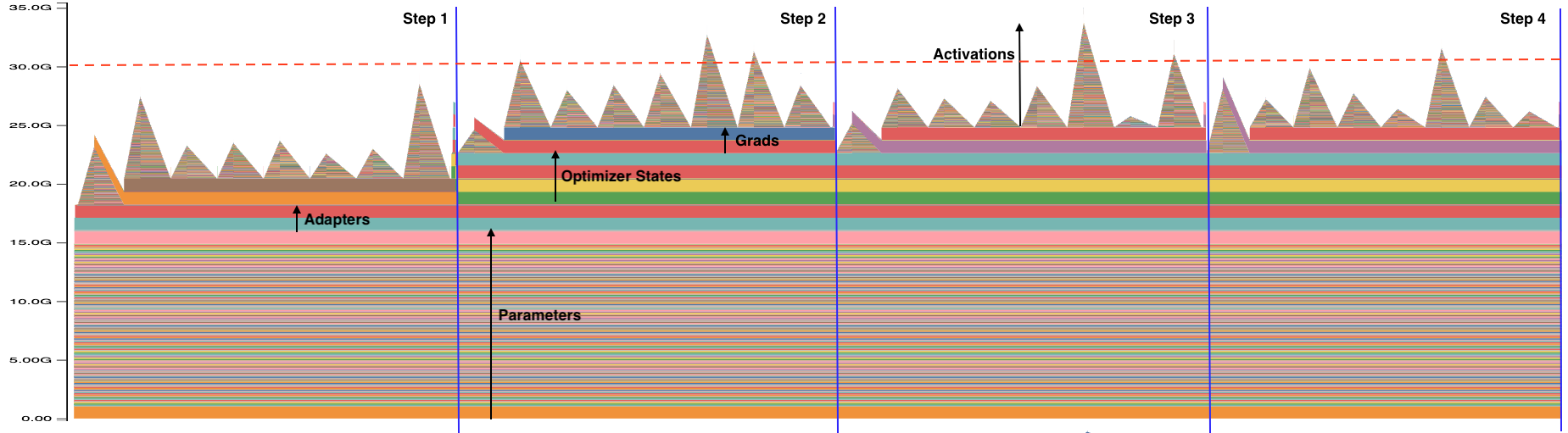}
    \caption{LoRA ($r=8$): Slightly higher activation and adapter memory, total $\sim$33.6 GB.}
\end{figure}
\begin{figure}[H]
    \centering
    \includegraphics[width=\linewidth]{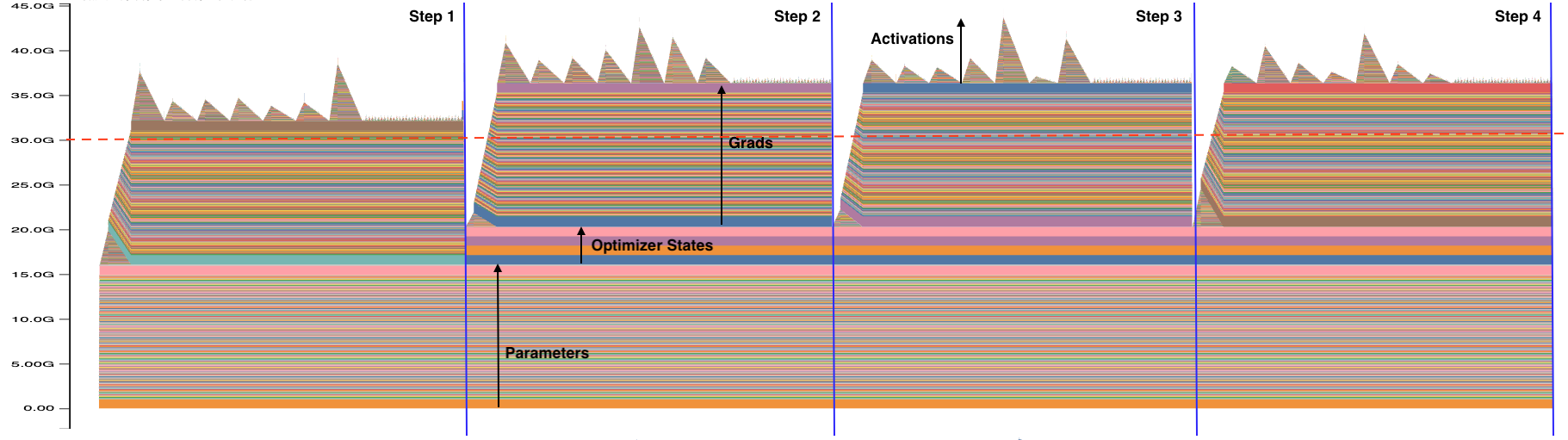}
    \caption{SWAN: Stateless optimizer, but full-sized gradient buffers lead to higher memory usage ($\sim$43.9 GB).}
\label{fig:swan}
\end{figure}
\begin{figure}[H]
    \centering
    \includegraphics[width=\linewidth]{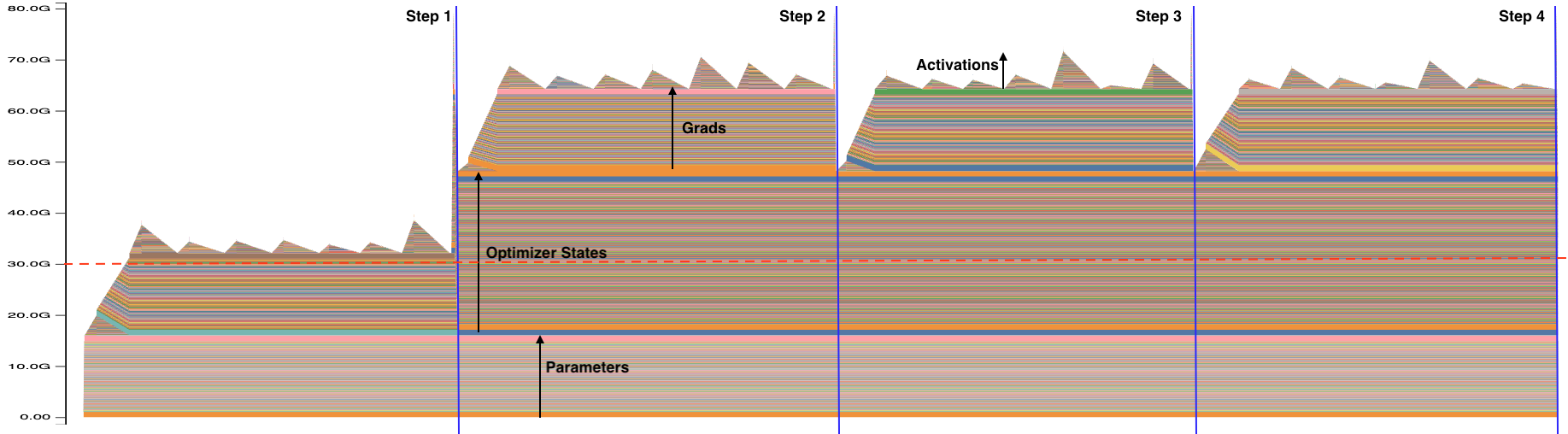}
    \caption{AdamW (BF16): Full-rank moments and gradient buffers dominate memory ($\sim$70.8 GB).}
\end{figure}
\begin{figure}[H]
    \centering
    \includegraphics[width=\linewidth]{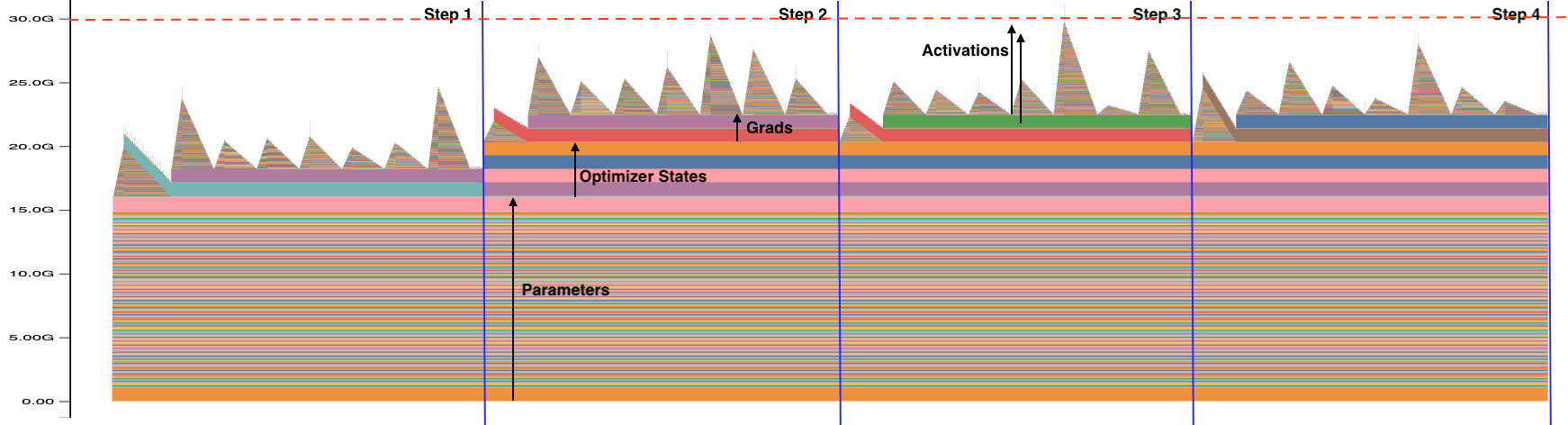}
    \caption{GaLore (Fused, $r=8$): Low gradient and optimizer state memory; total $\sim$30.0 GB.}
\end{figure}
\begin{figure}[H]
    \centering
    \includegraphics[width=\linewidth]{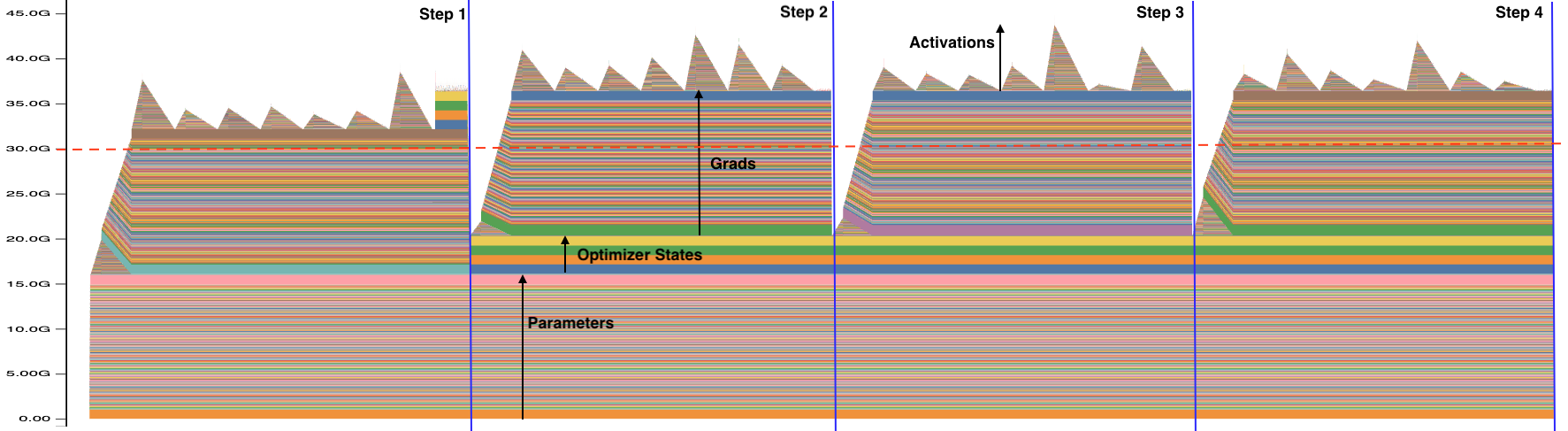}
    \caption{GaLore (Non-Fused, $r=8$): Gradient accumulation inflates memory cost ($\sim$44.5 GB).}
    \label{fig:non-fused-gal}
\end{figure}

\noindent\textbf{Quantitative Breakdown.} The table below summarizes the memory footprint by category for all setups.

\begin{table}[H]
\centering
\caption{Memory breakdown (in GB) for LLaMA3.1-8B fine-tuning.}
\label{tab:appendix-mem-breakdown}
\begin{tabular}{lccccc}
\toprule
\textbf{Optimizer} & \textbf{Params} & \textbf{Opt. States} & \textbf{Gradients} & \textbf{Activations} & \textbf{Adapters} \\
\midrule
MoFaSGD $(r=8)$ & 15.5 & 4.2 & 2.1 & 7.6 & 0 \\
LoRA $(r=8)$ & 15.5 & 4.2 & 2.1 & 9.8 & 2.0 \\
SWAN & 15.5 & 4.2 & 16.0 & 8.2 & 0 \\
AdamW (BF16) & 15.5 & 31.8 & 16.0 & 7.5 & 0 \\
GaLore Fused $(r=8)$ & 15.5 & 4.2 & 2.1 & 8.2 & 0 \\
GaLore Non-Fused $(r=8)$ & 15.5 & 4.2 & 16.0 & 8.8 & 0 \\
\bottomrule
\end{tabular}
\end{table}

\section{Proofs and Technical Details}

\subsection{On the Choice of Nuclear-norm Smoothness}
\label{sec:app-extra}

One of the main assumptions behind our convergence bound in Theorem~\ref{thm:main-conv} is smoothness under the nuclear norm (Assumption~\ref{assmp-main:1}). Here, we provide a detailed argument to justify this assumption. It is primarily motivated by concepts of descent in normed spaces, duality maps, and modular duality in deep learning—advancements detailed in recent works~\citep{bernstein2023automatic,large2025scalable,bernstein2024modular,old_bernstein_2024}. These provide a solid foundation for spectral normalization of the descent update direction, which has been popularized in strong optimizers such as Muon~\citep{jordan2024muon,liu2025muon}.

First, we briefly discuss modular optimization theory, inspired by~\citet{large2025scalable}. The core idea of gradient descent is to optimize a locally linear upper bound of the loss function, as characterized by the gradient. Given a loss function $\gL: \gW \to \mbb R$, we ideally want to find a tight upper bound on the higher-order terms, $\gR(\vw, \Delta \vw)$:

\begin{equation}
    \gL (\vw + \Delta \vw) \le \gL(\vw) + \langle \nabla \gL(\vw), \Delta \vw \rangle + \gR(\vw, \Delta \vw)
\end{equation}

For $\gW \in \mbb R^d$, classical second-order theory estimates the upper bound on higher-order terms using the maximum singular value of the Hessian. This leads to employing the $\ell_2$ norm with the typical upper bound $\frac{L}{2} \| \Delta \vw \|_2^2$, where $L$ is the smoothness constant. However, this assumption often does not hold in deep learning practice and may yield a very loose upper bound. Modular theory~\citep{large2025scalable} aims to go beyond this by introducing an architecture-aware version of the upper bound based on the following conjecture:

\begin{equation}
   \gL (\vw + \Delta \vw) \le \gL(\vw) + \langle \nabla \gL(w), \Delta \vw \rangle + \frac{\lambda}{2} \| \Delta \vw \|^2
\tag{Normed Space Steepest Descent}
\end{equation}

Here, $\gW$ can be any vector space, and $\|.\|:\gW \to \mbb R$ is an arbitrary norm on it. Modular optimization theory aims to find appropriate norms on the weight space of all model parameters $\gW = \gW_1 \times \cdots \times \gW_L$. The modular norm on this parameter product space $\gW$ is defined as the scaled max over individual norms (e.g., $\max \left(s_1 \| \vw_1\|_{\gW_1}, \ldots, \| \vw_L\|_{\gW_L}\right)$). The specific choice of norms for each parameter is based on its role and properties. For instance, for a linear layer $\bm W$, we expect the normalized feature change not to be drastic, which is characterized by the induced operator norm defined as $\| \Delta \bm W\|_{\alpha \to \beta} = \max \frac{ \|\bm W \vx\|_{\beta}}{\|\vx\|_{\alpha}}$. Empirical observations show that for typical gradient updates, the normalized feature change of a linear layer is indeed close to the RMS-RMS operator norm~\citep{large2025scalable}.

\citet{large2025scalable} show that by assigning norms to atomic modules (e.g., RMS $\to$ RMS for Linear layers and $\ell_1 \to$ RMS for embeddings), the modular norm of general architectures like Transformers satisfies the smoothness property. Informally, Proposition 5 of~\citet{large2025scalable} shows that for a general module $\bm M$ over $\gX, \gY, \gW$, and for common loss functions $\gL(\vw ) = \E_{x,y}[\ell(M(\bm w , \bm x))]$ such as cross-entropy, we have:

\begin{equation}
\label{eq:ncn}
   \| \nabla \gL ( \vw + \Delta \vw) - \nabla \gL(\vw)\|^*_{M} \le L_{M} \| \Delta \vw\|_M
\end{equation}

where $\|.\|^*_{M}$ is the dual norm of the modular norm defined earlier. Since our method targets the linear layers of Transformers, and the dual norm of RMS $\to$ RMS is the fan-in/fan-out scaled nuclear norm, we argue that assuming smoothness under the nuclear norm is more natural than under the Frobenius norm. The Frobenius norm assumption essentially treats the parameter space as $\mbb R^d$, ignoring the matrix structure of linear layers and their architectural role.

However, note that our Assumption~\ref{assmp-main:1} is still stronger than Equation~\ref{eq:ncn}. Even if we use the RMS $\to$ RMS operator norm for all linear layers and assume the remaining parameters are fixed, letting the weight space be $\gW_{\texttt{linear}} = (\bm W_1,\ldots,\bm W_{N_l})$, then by Equation~\ref{eq:ncn}, we would have:

\[
\sum_{i=1}^{N_{l}} \| \nabla \gL(\bm W_i + \Delta \bm W_i) - \nabla \gL (\bm W_i)\|_* \le L \max_{i=1}^{N_l} \left(\| \Delta \bm W_i \|_2\right)
,\]

where $\|.\|_*$ denotes the nuclear norm, and $\|.\|_2$ is the spectral norm, i.e., the $\ell_2 \to \ell_2$ operator norm.

\subsection{Mathematical Tools and Notations}
\label{appx:note}

\noindent\textbf{Kronecker Product and Vectorization.} Here, we introduce the background regarding the definition of Kronecker product, Vectorization, and well-known properties of these operators that make it easier to work with them. The Kronecker product is denoted as $\otimes$, and for any arbitrary $\bm X \in \mbb R^{m_1 \times n_1}$ and $\bm Y \in \mbb R^{m_2 \times n_2}$, is defined as follows, where $x_{i,j}$ are the elements on row $i$ and column $j$ of the matrix $\bm X$:
\begin{equation}
\bm X \otimes \bm Y =
\begin{bmatrix}
x_{1,1} \bm Y & x_{1,2} \bm Y & \cdots & x_{1,n_1} \bm Y \\
x_{2,1} \bm Y & x_{2,2} \bm Y & \cdots & x_{2,n_1} \bm Y \\
\vdots & \vdots & \ddots & \vdots \\
x_{m_1,1} \bm Y & x_{m_1,2} \bm Y & \cdots & x_{m_1,n_1} \bm Y
\end{bmatrix}
\end{equation}

Moreover, we define the vectorization operator $\operatorname{Vec}(.)$ that stacks columns of the matrix as a vector. Particularly, we have:
\begin{equation}
    \operatorname{Vec}(\bm{X}) = 
    \begin{bmatrix}
    x_{1,1} \\
    x_{2,1} \\
    \vdots \\
    x_{m_1,1} \\
    x_{1,2} \\
    x_{2,2} \\
    \vdots \\
    x_{m_1,n_1}
    \end{bmatrix}  \in \mbb R^{m_1  n_1}
\end{equation}

The following lemma summarizes basic properties of $\otimes$ and $\operatorname{Vec}(.)$ that we leverage throughout our proofs in the paper. Since this lemma covers basic, well-known properties in Matrix Algebra, we refer the reader to~\citet{horn1994topics} for proofs and details.  
\begin{lemma}
\label{lemma:kron-base}
    Let $A$, $B$, $C$ and $D$ be arbitrary matrices and assume that operations in each property are well-defined with respect to their dimensions. We have: 
    \begin{enumerate}
    \item Scalar Multiplication: $\bm A \otimes (\alpha \bm B) = \alpha (\bm A \otimes \bm B) = (\alpha \bm A) \otimes \bm B$
        \item Associativity: $(\bm A \otimes \bm B) \otimes \bm C = \bm A \otimes (\bm B \otimes \bm C)$
        \item Distributivity over addition: $(\bm A + \bm B) \otimes \bm C = (\bm A \otimes \bm C) + (\bm B \otimes \bm C)$, and $\bm A \otimes (\bm B + \bm C)= (\bm A \otimes \bm B) + (\bm A \otimes \bm C)$
        \item Not necessarily commutative: $(\bm A  \otimes \bm B) \neq ( \bm B \otimes \bm A)$
        \item Transpose: $(\bm A \otimes \bm B)^{\top} = \bm A^{\top} \otimes \bm B^{\top}$
        \item if $\bm A$ and $\bm B$ are positive-semidefinite, we have $(\bm A \otimes \bm B)^{s} = \bm A^s \otimes \bm B^s$ for any positive real $s$, and if $\bm A$ is positive definite, it holds for any real $s$.
        \item Mixed-product property: $(\bm A \otimes \bm B)(\bm C \otimes \bm D) = \bm A \bm C \otimes \bm B \bm D $ 
        \item Outer product: $\bm u \otimes \bm v = \bm u \bm v^{\top}$
        \item Mixed Kronecker matrix-vector product: $(\bm A^{\top} \otimes \bm B) \operatorname{Vec}(\bm C) = \operatorname{Vec}(\bm B \bm C \bm A) $ 
    \end{enumerate}
\end{lemma}

\subsection{Proofs}
\label{appx:proofs}

\begin{proof}[Proof of Theorem~\ref{thm:main-1}]
To prove Theorem~\ref{thm:main-1}, we first introduce the following lemma.

\begin{lemma}
\label{lemma:main-dcomp}
Consider the reduced SVD decomposition of the left and right sketches as follows $\bm L = \bm U_L \bm \Sigma_L \bm V_L^{\top}$ and $ \bm R = \bm U_R \bm \Sigma_R \bm V_R^{\top}$. Then,  the vectorized unsketched representation of any arbitrary matrix $ \bm G \in \mbb R^{m \times n }$ can be decomposed as follows: 
\begin{equation}
    \operatorname{Vec}(\bm L \bm L^{\top} \bm G + \bm G \bm R \bm R^{\top}) = \bm U_{\bm L , \bm R}  \bm \Sigma_{\bm L , \bm R}  \bm U_{\bm L , \bm R}^{\top} \operatorname{Vec}(\bm G)
\end{equation}
where $ \bm U_{\bm L , \bm R} \in \mbb R^{mn \times r (m+n-r)}$ is a semi-orthogonal matrix (i.e., $\bm U_{\bm L , \bm R}^{\top} \bm U_{\bm L , \bm R} = \bm I_{r (m+n-r)}$), and $\bm \Sigma_{\bm L , \bm R} \in \mbb R^{(m+n-r) \times (m+n-r)}$ is a diagonal matrix, defined as below:  
\begin{equation}
\setlength{\arraycolsep}{0.1pt} 
\begin{aligned}
    &\bm U_{\bm L, \bm R} = \begin{bmatrix}
        \bm U_{\bm R} \otimes \bm U_{\bm L} ,  
        & \bm U_{\bm R} \otimes \bm U_{\bm L}^\perp , 
        & \bm U_{\bm R}^\perp \otimes \bm U_{\bm L}
    \end{bmatrix} \\
    & \bm \Sigma_{\bm L , \bm R} = \begin{bmatrix}
         \bm \Sigma_{\bm R}^2 \otimes \bm I_r  + \bm I_r \otimes \bm \Sigma_{\bm L}^2  
         & \bm 0 
         & \bm 0 \\
         \bm 0 
         & \bm \Sigma^2_{\bm R} \otimes \bm I_{n-r} 
         & \bm 0 \\
         \bm 0 
         & \bm 0 
         & \bm I_{m-r} \otimes \bm \Sigma^2_{\bm L}
    \end{bmatrix}
\end{aligned}
\label{eq:block_diagonal_dimensions}
\end{equation}
\end{lemma}

\begin{proof}[Proof of Lemma~\ref{lemma:main-dcomp}]
Recall the left singular vectors of the left and right sketch matrices, $\bm U_{\bm L} \in \mbb R^{m \times r}$ and $\bm U_{\bm R} \in \mbb R^{n \times r}$. We can augment these reduced bases with $m-r$ and $n-r$ orthogonal vectors to get full orthonormal bases as follows: $[ \bm U_{\bm L} , \bm U_{\bm L}^\perp] \in \mbb R^{m \times m} $ and $[ \bm U_{\bm R} , \bm U_{\bm R}^\perp] \in \mbb R^{n \times n} $. Since each of these expanded bases is square orthonormal matrices, we can write $\bm U_{\bm L}^\top \bm U_{\bm L} = \bm I_r$, $(\bm U_{\bm L}^\perp)^\top \bm U_{\bm L}^\perp = \bm I_{m-r} $, and also $\bm U_{\bm L}^\top  \bm U_{\bm L}^\perp = \bm 0$. Similar properties hold for $\bm U_{\bm R}$. Moreover, again based on the orthonormality of the augmented bases, we can write:
\begin{equation}
\label{eqn:app1}
\bm U_{\bm L} \bm U_{\bm L}^\top + \bm U_{\bm L}^\perp (\bm U_{\bm L}^\perp)^\top  = \bm I_{m} \quad , \quad \bm U_{\bm R} \bm U_{\bm R}^\top + \bm U_{\bm R}^\perp (\bm U_{\bm R}^\perp)^\top  = \bm I_{n}
\end{equation}

Leveraging Lemma~\ref{lemma:kron-base} we can write:
\begin{equation}
\label{eq:16}
\begin{split}
&\operatorname{Vec}\bigl(\bm{L}\bm{L}^\top\bm{G}           + \bm{G}\bm{R}\bm{R}^\top\bigr) =
  \operatorname{Vec}\bigl(\bm{L}\bm{L}^\top\bm{G}\bigr)
  +
  \operatorname{Vec}\bigl(\bm{G}\bm{R}\bm{R}^\top\bigr)
\\
& \quad=
    \Bigl(\bm{I}_n \otimes \bm{L}\bm{L}^\top      +
         \bm{R}\bm{R}^\top \otimes \bm{I}_m\Bigr)\operatorname{Vec}(\bm{G})
\\
& \quad=\Bigl(\bm{I}_n \otimes \bm{U_L}\bm{\Sigma_L}^2\bm{U_L}^\top + \bm{U_R}\bm{\Sigma_R}^2\bm{U_R}^\top \otimes \bm{I}_m\Bigr)
         \operatorname{Vec}(\bm{G})
\end{split}
\end{equation}

Using Equation~\ref{eq:16} and properties in Lemma~\ref{lemma:kron-base}, we can further write:
\begin{equation}\label{eq:17}
\begin{split}
& \bm{I}_n \otimes
  \bm{U_L}\bm{\Sigma_L}^2\bm{U_L}^\top
 =
  \Bigl(\bm{U_R}\bm{U_R}^\top      +
         \bm{U_R}^\perp(\bm{U_R}^\perp)^\top\Bigr)
  \otimes
  \Bigl(\bm{U_L}\bm{\Sigma_L}^2\bm{U_L}^\top\Bigr)
\\
& =
  \bm{U_R}\bm{U_R}^\top \otimes
  \bm{U_L}\bm{\Sigma_L}^2\bm{U_L}^\top
  +
  \bm{U_R}^\perp(\bm{U_R}^\perp)^\top \otimes
  \bm{U_L}\bm{\Sigma_L}^2\bm{U_L}^\top
\\
& =
  \bigl(\bm{U_R}\otimes \bm{U_L}\bm{\Sigma_L}\bigr)
  \bigl(\bm{U_R}\otimes \bm{U_L}\bm{\Sigma_L}\bigr)^\top
  +
  \bigl(\bm{U_R}^\perp\otimes \bm{U_L}\bm{\Sigma_L}\bigr)
  \bigl(\bm{U_R}^\perp\otimes \bm{U_L}\bm{\Sigma_L}\bigr)^\top \\
& =
  \bigl(\bm{U_R} \otimes \bm{U_L}\bigr)\bigl(\bm I_r \otimes \bm \Sigma_{\bm L}^2 \bigr)\bigl(\bm{U_R} \otimes \bm{U_L}\bigr)^\top+\bigl(\bm{U_R}^\perp \otimes \bm{U_L}\bigr)\bigl(\bm I_{n-r} \otimes \bm \Sigma_{\bm L}^2 \bigr)\bigl(\bm{U_R}^\perp \otimes \bm{U_L}\bigr)^\top
\end{split}
\end{equation}
Using a similar derivation as in Equation~\ref{eq:17}, we obtain:
\begin{equation}\label{eq:18}
\begin{split}
&\bm{U_R}\bm{\Sigma_R}^2\bm{U_R}^\top 
  \otimes
  \bm{I}_m =
  \bigl(\bm{U_R}\otimes \bm{U_L}\bigr)
  \bigl(\bm{\Sigma_R}^2 \otimes \bm{I}_r\bigr)
  \bigl(\bm{U_R}\otimes \bm{U_L}\bigr)^\top+
  \bigl(\bm{U_R}\otimes \bm{U_L}^\perp\bigr)
  \bigl(\bm{\Sigma_R}^2 \otimes \bm{I}_{m-r}\bigr)
  \bigl(\bm{U_R}\otimes \bm{U_L}^\perp\bigr)^\top
\end{split}
\end{equation}
Plugging Equation~\ref{eq:17} and Equation~\ref{eq:18} back into Equation~\ref{eq:16}, we get:
\begin{equation}
\label{eqn:main-d}
\begin{split}
\operatorname{Vec}(\bm{L}\bm{L}^\top \bm{G}           + \bm{G}\bm{R}\bm{R}^\top) &=\Bigl[
    \bigl(\bm{U_R}\otimes\bm{U_L}\bigr)
    \bigl(\bm{I}_r \otimes \bm{\Sigma_L}^2\bigr)
\bigl(\bm{U_R}\otimes\bm{U_L}\bigr)^\top + \bigl(\bm{U_R}^\perp\otimes\bm{U_L}\bigr)
    \bigl(\bm{I}_{n-r} \otimes \bm{\Sigma_L}^2\bigr)
    \bigl(\bm{U_R}^\perp\otimes\bm{U_L}\bigr)^\top\\
&\quad\quad+
    \bigl(\bm{U_R}\otimes\bm{U_L}\bigr)
    \bigl(\bm{\Sigma_R}^2 \otimes \bm{I}_{r}\bigr)
    \bigl(\bm{U_R}\otimes\bm{U_L}\bigr)^\top
\\
&\quad\quad+\bigl(\bm{U_R}\otimes\bm{U_L}^\perp\bigr)
    \bigl(\bm{\Sigma_R}^2 \otimes \bm{I}_{m-r}\bigr) \bigl(\bm{U_R}\otimes\bm{U_L}^\perp\bigr)^\top
  \Bigr]
  \operatorname{Vec}(\bm{G})
\\
& = \Bigl[
    \bigl(\bm{U_R}\otimes\bm{U_L}\bigr)
    \bigl(\bm{I}_r \otimes \bm{\Sigma_L}^2
         + \bm{\Sigma_R}^2 \otimes \bm{I}_r\bigr)
    \bigl(\bm{U_R}\otimes\bm{U_L}\bigr)^T
\\
&\quad\quad+\bigl(\bm{U_R}^\perp\otimes\bm{U_L}\bigr)
    \bigl(\bm{I}_{n-r} \otimes \bm{\Sigma_L}^2\bigr)
    \bigl(\bm{U_R}^\perp\otimes\bm{U_L}\bigr)^\top
\\
&\quad\quad+\bigl(\bm{U_R}\otimes\bm{U_L}^\perp\bigr)
    \bigl(\bm{\Sigma_R}^2 \otimes \bm{I}_{m-r}\bigr)
    \bigl(\bm{U_R}\otimes\bm{U_L}^\perp\bigr)^\top
  \Bigr]
  \operatorname{Vec}(\bm{G})
\end{split}
\end{equation}
To complete the proof, it remains to show that the block matrix
\begin{equation}
  \bm U_{\bm L , \bm R} =
  \bigl[\bm U_{\bm R} \otimes \bm U_{\bm L}, \bm U_{\bm R}^{\perp} \otimes \bm U_{\bm L} , 
         \bm U_{\bm R}\otimes \bm U_{\bm L}^{\perp}\bigr]
  \in\mathbb{R}^{mn \times r(m+n-r)}
\end{equation}
is semi-orthogonal, i.e.\ 
\(\bm U_{\bm L , \bm R}^\top \bm U_{\bm L , \bm R} = \bm I_{r(m+n-r)}.\)

First, note that pairs of sub-blocks are mutually orthogonal. For instance,
\[
  \bigl(\bm U_{\bm R} \otimes \bm U_{\bm L}\bigr)^\top
  \bigl(\bm U_{\bm R}^{\perp} \otimes \bm U_{\bm L}\bigr)
  =
  \bigl(\bm U_{\bm R}^\top \bm U_{\bm R}^{\perp}\bigr)
  \otimes
  \bigl(\bm U_{\bm L}^\top \bm U_{\bm L}\bigr)
  =\bm 0.
\]
A similar argument applies to all other sub-block pairs.

Moreover, each sub-block is orthonormal in its own right. For example,
\[
  \bigl(\bm U_{\bm R} \otimes \bm U_{\bm L}\bigr)^\top
  \bigl(\bm U_{\bm R} \otimes \bm U_{\bm L}\bigr)
  =
  \bigl(\bm U_{\bm R}^\top \bm U_{\bm R}\bigr)
  \otimes
  \bigl(\bm U_{\bm L}^\top \bm U_{\bm L}\bigr)
  =
  \bm I_r \otimes \bm I_r =\bm I_{r^2}.
\]
Hence, \(\bm U_{\bm L , \bm R}\) is a semi-orthogonal, and the result follows.
\end{proof}

We turn to proving Theorem~\ref{thm:main-1}.

Let the reconstructed gradient after projection be $\Tilde{\bm G} = \operatorname{Proj}_{(\bm L, \bm R)} (\bm G) = \alpha_1 \bm L \bm L^\top \bm G + \alpha_2 \bm G \bm R \bm R^\top + \alpha_3 \bm L \bm L^\top \bm G \bm R \bm R^\top$. We can write: 
\begin{equation}
    \operatorname{Vec}(\Tilde{\bm G}) = \operatorname{Vec}(\alpha_1 \bm L \bm L^\top \bm G + \alpha_2 \bm G \bm R \bm R^\top) + \alpha_3 (\bm R \bm R^\top \otimes \bm L \bm L^\top ) \operatorname{Vec}(\bm G)
\end{equation}
Replacing $\bm L$ and $\bm R$ with their corresponding reduced SVD decompositions, we have:  
\begin{equation}
\label{eqn:main-d2}
\begin{split}
 (\bm R \bm R^\top \otimes \bm L \bm L^\top ) &= \bm U_{\bm R} \bm \Sigma_{\bm R}^2 \bm U_{\bm R}^\top \otimes \bm U_{\bm L} \bm \Sigma_{\bm L}^2 \bm U_{\bm L}^\top = (\bm U_{\bm R} \otimes \bm U_{\bm L}) (\bm \Sigma_{\bm R}^2 \otimes \bm \Sigma_{\bm L}^2) (\bm U_{\bm R} \otimes \bm U_{\bm L})^\top
 \end{split}
\end{equation}
Using Equation~\ref{eqn:main-d} from Lemma~\ref{lemma:main-dcomp} and Equation~\ref{eqn:main-d2}, we can rewrite:  
\begin{equation}
\begin{split}
\operatorname{Vec}(\Tilde{\bm G}) &=\Bigl[
    \bigl(\bm{U_R}\otimes\bm{U_L}\bigr)
    \bigl(\alpha_1 \bm{I}_r \otimes \bm{\Sigma_L}^2
         + \alpha_2 \bm{\Sigma_R}^2 \otimes \bm{I}_r + \alpha_3 \bm \Sigma_{\bm R}^2 \otimes \bm \Sigma_{\bm L}^2\bigr)
    \bigl(\bm{U_R}\otimes\bm{U_L}\bigr)^\top
\\
& \quad\quad+
    \alpha_2 \bigl(\bm{U_R}^\perp\otimes\bm{U_L}\bigr)
    \bigl(\bm{I}_{n-r} \otimes \bm{\Sigma_L}^2\bigr)
    \bigl(\bm{U_R}^\perp\otimes\bm{U_L}\bigr)^\top
\\
& \quad\quad+
    \alpha_1 \bigl(\bm{U_R}\otimes\bm{U_L}^\perp\bigr)
    \bigl(\bm{\Sigma_R}^2 \otimes \bm{I}_{m-r}\bigr)
    \bigl(\bm{U_R}\otimes\bm{U_L}^\perp\bigr)^\top
  \Bigr]
  \operatorname{Vec}(\bm{G}) \\
  &\triangleq \bm P_{\bm L, \bm R}  \operatorname{Vec}(\bm{G})
\end{split}
\end{equation}
where in the last equality, we define $\bm P_{\bm L, \bm R} \in \mbb R^{mn \times mn}$ as the corresponding projection matrix of $\operatorname{Proj}_{(\bm L, \bm R)}$.

We can use Equation~\ref{eqn:app1} to write:  
\begin{equation}
\begin{split}
    \bm I_{mn} &= \bm I_n \otimes \bm I_m = (\bm U_{\bm R} \otimes \bm U_{\bm L}) (\bm U_{\bm R} \otimes \bm U_{\bm L})^\top + (\bm U_{\bm R}^\perp \otimes \bm U_{\bm L})  (\bm U_{\bm R}^\perp \otimes \bm U_{\bm L})^\top + (\bm U_{\bm R} \otimes \bm U_{\bm L}^\perp) (\bm U_{\bm R} \otimes \bm U_{\bm L}^\perp)^\top \\
    &\quad+  (\bm U_{\bm R}^\perp \otimes \bm U_{\bm L}^\perp) (\bm U_{\bm R}^\perp \otimes \bm U_{\bm L}^\perp)^\top
\end{split}
\end{equation}
We now propose our Kronecker decomposition of the subspace projection residual. 
\begin{equation}
\label{eqn:main_residual}
\begin{split}
&\operatorname{Vec}(\Tilde{\bm G} - \bm G) = (\bm I_n \otimes \bm I_m - \bm P_{\bm L, \bm R}) \operatorname{Vec} (\bm G) = \\
&\Bigl[
\bigl(\bm{U_R}\otimes\bm{U_L}\bigr)
    \bigl(\alpha_1 \bm{I}_r \otimes \bm{\Sigma_L}^2
         + \alpha_2 \bm{\Sigma_R}^2 \otimes \bm{I}_r + \alpha_3 \bm \Sigma_{\bm R}^2 \otimes \bm \Sigma_{\bm L}^2-\bm I_r \otimes \bm I_r\bigr)\bigl(\bm{U_R}\otimes\bm{U_L}\bigr)^\top \\
     &\quad+
\bigl(\bm{U_R}^\perp\otimes\bm{U_L}\bigr)
    \bigl(\alpha_2 \bm{I}_{n-r} \otimes \bm{\Sigma_L}^2 - \bm I_{n-r} \otimes \bm I_r \bigr) \bigl(\bm{U_R}^\perp\otimes\bm{U_L}\bigr)^\top
\\
&\quad +
\bigl(\bm{U_R}\otimes\bm{U_L}^\perp\bigr)
    \bigl(\alpha_1 
 \bm{\Sigma_R}^2 \otimes \bm{I}_{m-r} -  \bm I_r \otimes \bm I_{m-r} \bigr)
\bigl(\bm{U_R}\otimes\bm{U_L}^\perp\bigr)^\top
\\
&\quad+\bigl(\bm{U_R}^\perp \otimes\bm{U_L}^\perp\bigr)\bigl(\bm{U_R}^\perp \otimes\bm{U_L}^\perp\bigr)^\top \Bigr]
  \operatorname{Vec}(\bm{G}).
\end{split}
\end{equation}

For simplicity let $\bm A_1 = \bigl(\bm{U_R}\otimes\bm{U_L}\bigr)
    \bigl(\alpha_1 \bm{I}_r \otimes \bm{\Sigma_L}^2
         + \alpha_2 \bm{\Sigma_R}^2 \otimes \bm{I}_r + \alpha_3 \bm \Sigma_{\bm R}^2 \otimes \bm \Sigma_{\bm L}^2-\bm I_r \otimes \bm I_r\bigr)\bigl(\bm{U_R}\otimes\bm{U_L}\bigr)^\top$, $\bm A_2  = \bigl(\bm{U_R}^\perp\otimes\bm{U_L}\bigr)
    \bigl(\alpha_2 \bm{I}_{n-r} \otimes \bm{\Sigma_L}^2 - \bm I_{n-r} \otimes \bm I_r \bigr) \bigl(\bm{U_R}^\perp\otimes\bm{U_L}\bigr)^\top $, $\bm A_3 = \bigl(\bm{U_R}\otimes\bm{U_L}^\perp\bigr)
    \bigl(\alpha_1 
 \bm{\Sigma_R}^2 \otimes \bm{I}_{m-r} -  \bm I_r \otimes \bm I_{m-r} \bigr)
\bigl(\bm{U_R}\otimes\bm{U_L}^\perp\bigr)^\top$ and $\bm A_4 = \bigl(\bm{U_R}^\perp \otimes\bm{U_L}^\perp\bigr)\bigl(\bm{U_R}^\perp \otimes\bm{U_L}^\perp\bigr)^\top $. First note that $\bm A_1, \bm A_2, \bm A_3, \bm A_4 \in \mbb R^{mn \times mn}$. For all distinct $(i,j) \in \{1,2,3,4\}$, their matrix product is zero, $\bm A_i^\top \bm A_j = \bm 0$. Without loss of generality, we show it for $\bm A_2$, and $\bm A_3$ here.

\begin{equation}
\begin{split}
    \bm A_2^\top \bm A_3 &= \bigl(\bm{U_R}^\perp\otimes\bm{U_L}\bigr)
    \bigl(\alpha_2 \bm{I}_{n-r} \otimes \bm{\Sigma_L}^2 - \bm I_{n-r} \otimes \bm I_r \bigr) \bigl(\bm{U_R}^\perp\otimes\bm{U_L}\bigr)^\top \bigl(\bm{U_R}\otimes\bm{U_L}^\perp\bigr)
    \bigl(\alpha_3 
 \bm{\Sigma_R}^2 \otimes \bm{I}_{m-r} -  \bm I_r \otimes \bm I_{m-r} \bigr)
\\
&\quad \quad \bigl(\bm{U_R}\otimes\bm{U_L}^\perp\bigr)^\top \\
&=\bigl(\bm{U_R}^\perp\otimes\bm{U_L}\bigr)
    \bigl(\alpha_2 \bm{I}_{n-r} \otimes \bm{\Sigma_L}^2 - \bm I_{n-r} \otimes \bm I_r \bigr)
    \bigl(\bm{U_R}^{\perp\top} \bm {U_R} \otimes\bm{U_L}^\top \bm{U_L}^\perp \bigr)^\top
    \bigl(\alpha_3 
 \bm{\Sigma_R}^2 \otimes \bm{I}_{m-r} -  \bm I_r \otimes \bm I_{m-r} \bigr)
\\
&\quad \quad \bigl(\bm{U_R}\otimes\bm{U_L}^\perp\bigr)^\top  \\
&=\bigl(\bm{U_R}^\perp\otimes\bm{U_L}\bigr)
    \bigl(\alpha_2 \bm{I}_{n-r} \otimes \bm{\Sigma_L}^2 - \bm I_{n-r} \otimes \bm I_r \bigr)
    \bigl( \bm 0_{n-r,r} \otimes \bm 0_{n-r,r} \bigr)^\top
    \bigl(\alpha_3 
 \bm{\Sigma_R}^2 \otimes \bm{I}_{m-r} -  \bm I_r \otimes \bm I_{m-r} \bigr)
\\
&\quad \quad \bigl(\bm{U_R}\otimes\bm{U_L}^\perp\bigr)^\top  \\
&= \bm 0_{mn}
\end{split}
\end{equation}

Using Equation~\ref{eqn:main_residual}, and above fact yields: 
\begin{equation}
\label{eqn:dmb}
    \|\Tilde{\bm G} - \bm G\|^2_F = \| \operatorname{Vec}(\Tilde{\bm G} - \bm G) \|^2_2 = \| \bm A_1 \vg \|^2_2 + \| \bm A_2 \vg \|^2_2 + \| \bm A_3 \vg \|^2_2 + \| \bm A_4 \vg \|^2_2
\end{equation}
where $\| . \|_2$ here is $\ell_2$ norm of a vector, and $\vg = \operatorname{Vec}(\bm G)$. The equality holds because $(\bm A_i \vg)^\top (\bm A_j \vg) = \bm 0_{mn}$. Based on the above equation, we have the following lower bound on residual $ \|\Tilde{\bm G} - \bm G\|^2_F \ge  \| \bm A_4 \vg \|^2_2 $ for any given $\bm{L}$, and $\bm{R}$. The term $\bm A_4 $ does not depend on the choice of $(\alpha_1,\alpha_2,\alpha_3)$, also $\bm \Sigma_{\bm R}$ and $\bm \Sigma_{\bm L}$. Thus, the lower bound on this residual would be tight if and only if we have $\| \bm A_1 \vg \|_2 =    \| \bm A_2 \vg \|_2 = \| \bm A_3 \vg \|_2 = 0$. By setting $\bm \Sigma_{\bm L} = \bm I_r$, $\bm \Sigma_{\bm R} = \bm I_r$, and then considering $(\alpha_1,\alpha_2,\alpha_3) = (1,1,-1)$, we have $\bm A_1 = \bm A_2 = \bm A_3 = \bm 0_{mn}$.

Thus, we have shown that under conditions in Theorem~\ref{thm:main-1}, the first three terms in Equation~\ref{eqn:dmb} will be exactly zero, leaving the following term as the projection residual:  
\begin{equation}
\operatorname{Vec}(\Tilde{\bm G}_{\text{optimal}} - \bm G) = \bigl(\bm{U_R}^\perp \otimes\bm{U_L}^\perp\bigr)\bigl(\bm{U_R}^\perp \otimes\bm{U_L}^\perp\bigr)^\top \operatorname{Vec}(\bm{G}).
\end{equation}
\end{proof}

\subsubsection{Convergence Proofs}
In this section, we aim to find an upper bound on the iteration complexity of Algorithm~\ref{alg:mofasgd_complete} for finding an $\epsilon$-stationary point defined by the averaged nuclear norm of gradients. First, we recall the necessary assumptions and notations.  

\noindent\textbf{Update Rule.} Let $\{\bm W_t\}_{t=1}^T$ be the iterates of Algorithm~\ref{alg:mofasgd_complete}. Then we have the following equivalent update rule:
\begin{equation}
\label{eqn:app-update}
    \bm W_{t+1} = \bm W_t - \eta \bm U_{t+1} \bm V_{t+1}^\top \quad , \quad \hat{\bm M}_t = \bm U_{t+1} \bm \Sigma_{t+1} \bm V_{t+1}^\top \quad , \quad \bm M_t = \sum_{i=0}^t \beta^{t-i} \bm G_t 
\end{equation}
where $\bm M_t$ represents the full-rank momentum, and $\bm G_t = \nabla \mc L(\bm W_t, \xi_t) \in \mbb R^{m \times n}$ is the stochastic gradient at parameter $\bm W_t$. Moreover, let $ \hat{\bm M}_t$ be the low-rank momentum factorization, where $\bm U_i \in \mbb R^{m \times r}$, $\bm V_i \in \mbb R^{n \times r}$, and $\bm  \Sigma_i \in \mbb R^{r \times r}$.

\noindent\textbf{Definitions.} We recall the necessary definitions. We let the optimization objective be represented as $\mc L(\bm W) = \mbb E_{\xi} [ \mc L(\bm W , \xi) ]$, and moreover assume we have access to an unbiased, variance-bounded ($\sigma$) gradient oracle. For any optimization iterate $\bm W_i$, the full-batch gradient is denoted as $ \Bar{\bm G}_i = \nabla \mc L (\bm W_i)$, and the stochastic gradient is denoted as $ \bm G_i = \nabla \mc L (\bm W_i , \xi_i) $. Based on assumptions on the gradient oracle, we have $\Bar{\bm G}_i = \mbb E [ \bm G_i ]$, and $\mbb E[\| \bm G_i - \Bar{\bm G}_i   \|_*] \le \sigma$. Moreover, we say a function $\mc L (.): \mbb R^{m \times n } \to r$ is $L$-smooth with respect to an arbitrary norm $\| \cdot\|$ if for any two parameters $\bm W_1, \bm W_2$, we have $\| \nabla \mc L(\bm W_1) -  \nabla \mc L(\bm W_2)  \|_* \le L \| \bm W_1 - \bm W_2\|_2$.

Next, we propose our descent lemma.  

\begin{lemma}[Descent lemma]
\label{lemma:des}
    Let $\{ \bm W \}_{t=0}^T$ be the iterates of Algorithm~\ref{alg:mofasgd_complete}, optimized under a loss function that satisfies Assumption~\ref{assmp-main:1}. Then, we have:  
    \begin{equation}
        \mc L(\bm W_{t+1}) \le \mc L (\bm W_t ) - \eta \| \Bar{\bm G}_t\|_* + 2 \eta \| \hat{\bm M}_t - \Bar{\bm G}_t\|_* + \frac{\eta^2 L}{2} 
    \end{equation}
\end{lemma}

\begin{proof}[Proof of Lemma~\ref{lemma:des}]
    First note that $\mc L(.)$ is $L$-smooth with the nuclear norm. Therefore for any $\bm W_1, W_2 \in \mbb R^{m\times n}$ we have: $\| \nabla \mc L(\bm W_1) - \nabla \mc L(\bm W_2) \|_* \le L \| \bm W_1 - \bm W_2 \|_2 $, where $\| . \|_2$ is the spectral norm. Due to this property, we can leverage Proposition $5$ in \citet{large2025scalable} to argue that for any $\bm W_1, \bm W_2$ we have:  
    \begin{equation}
        \mc L ( \bm W_2) \le \mc L(\bm W_1) + \langle  \nabla \mc L(\bm W_1),\bm W_2 - \bm W_1 \rangle_F + \frac{L}{2} \| \bm W_1 - \bm W_2 \|_2 
    \end{equation}
    Thus, considering the update rule as in Equation~\ref{eqn:app-update}, for two consecutive iterates $\bm W_{t}$ and $\bm W_{t+1}$, we can write: 
    \begin{equation}
    \label{eqn:app-dp1}
    \mc L (\bm W_{t+1}) \le \mc L(\bm W_t) - \eta \langle \Bar{\bm G}_t, \bm U_{t+1} \bm V_{t+1}^\top \rangle + \frac{L \eta^2}{2} \| \bm U_{t+1} \bm V_{t+1}^\top\|_2  
    \end{equation}
    Note that if we take $\| .\|_*$ as a function, its subgradient is well-known and can be derived as follows: the subgradient set at $\bm X = \bm U \bm \Sigma \bm V^\top$ is $\partial \| \bm X\|_* = \{ \bm U \bm V^\top + \bm H: \bm U^\top \bm H =0, \bm V^\top \bm H = 0 , \| \bm W\|_2 \le 1\}$. Since the nuclear norm is a convex function, for any $\bm X = \bm U \bm \Sigma \bm V^\top$ and $\bm Y$, we have:  
    \begin{equation}
        \| \bm Y \|_* \ge \| \bm X\|_* + \langle \bm U \bm V^\top , \bm Y - \bm X \rangle  
    \end{equation}
    Replace $\bm Y  = \hat{\bm M}_t - \Bar{\bm G}_t$ and $ \bm X = \hat{\bm M}_t = \bm U_{t+1} \bm \Sigma_{t+1} \bm V_{t+1}^\top $; then we can rewrite Equation~\ref{eqn:app-dp1} as: 
    \begin{equation}
    \label{eqn:app-dp2}
        \mc L (\bm W_{t+1}) \le \mc L(\bm W_t) - \eta \| \hat{\bm M}_t \|_* +  \eta \|  \hat{\bm M}_t - \Bar{\bm G}_t\|_* + \frac{\eta^2}{2}L  
    \end{equation}
    where we also used the fact that $\| \bm U \bm V^\top\|_2 = 1$. Using the triangle inequality, we have $\|\hat{\bm M}_t \|_* \ge \|\Bar{\bm G}_t\|_* - \| \hat{\bm M}_t - \Bar{\bm G}_t \|_*$, and plugging this into Equation~\ref{eqn:app-dp2} completes the proof.  
\end{proof}
As indicated in Lemma~\ref{lemma:des}, the term $\| \hat{\bm M}_t - \Bar{\bm G}_t\|_*$ is the main challenge for deriving the upper bound. Note that we can write $\| \hat{\bm M}_t - \Bar{\bm G}_t\|_* \le \| \bm M_t - \Bar{\bm G}_t\|_* + \| \hat{\bm M}_t - \bm M_t\|_*$. Thus, we aim to bound each decomposed term separately. Note that the first decomposed term $\| \bm M_t - \Bar{\bm G}_t\|_*$ measures the nuclear norm difference between the full-batched gradient and the first momentum. Moreover, the second decomposed term $\| \hat{\bm M}_t - \bm M_t\|_*$ measures the low-rank compression error of the full-rank momentum.

\begin{lemma}
\label{app:lemma1}
 Under Assumptions~\ref{assmp-main:1} and~\ref{assmp-main:1} we have:
 \begin{equation}
\label{eqn:lemma1}       \sum_{t=1}^T \mbb E [ \| \bm M_t - \Bar{\bm G}_t\|_*] \le \frac{\beta}{1 - \beta} \sum_{t=0}^{T-1} \mbb E [ \| \Bar{\bm G}_t\|_* ]  + \frac{T \sigma}{(1-\beta) \sqrt{B}} 
\end{equation}
\end{lemma}

\begin{proof}[Proof of Lemma~\ref{app:lemma1}]
    Let $\bm \Delta_t  =  \bm G_t - \Bar{\bm G}_t$. We can write $\| \bm M_t - \Bar{\bm G}_t \|_* \le \|\bm M_t - \bm G_t\|_* + \| \bm \Delta_t\|_* = \beta \| \bm M_{t-1} \|_* + \| \bm \Delta_t\|_*$. Moreover, we have $ \|\bm M_t \|_* \le \| \bm G_t\|_* + \beta \| \bm M_{t-1}\|_*$, which, upon unrolling, gives $ \| \bm M_t\|_* \le \sum_{i=0}^t \beta^{t-i} \| \bm G_i\|_*$. Thus, we can write $\|\bm M_t - \bm G_t\|_* \le \beta \sum_{i=0}^{t-1} \beta^{t-1-i} \| \bm G_i \|_*$. Finally, since $\sum_{i=0}^j \beta^i = \frac{1 - \beta^{j+1}}{1 - \beta} \le \frac{1}{1-\beta}$, we have:  
\begin{equation}
\label{eqn:app-dp12}
    \sum_{t=1}^T \| \bm M_t - \Bar{\bm G}_t\|_* \le \frac{\beta}{1 - \beta} \sum_{t=0}^{T-1} \| \Bar{\bm G}_t\|_* + \frac{1}{1-\beta} \sum_{t=0}^{T} \| \bm \Delta_t\|_* \end{equation}
Taking the expectation and leveraging Assumption~\ref{assmp-main:1}, we have:  
\begin{equation}
\label{eqn:app-dp17}       \sum_{t=1}^T \mbb E [ \| \bm M_t - \Bar{\bm G}_t\|_*] \le \frac{\beta}{1 - \beta} \sum_{t=0}^{T-1} \mbb E [ \| \Bar{\bm G}_t\|_* ]  + \frac{T \sigma}{(1-\beta) \sqrt{B}} 
\end{equation}
\end{proof} 

We now bound the second term $\| \hat{\bm M}_t - \bm M_t \|_*$. We have the following lemma:  
\begin{lemma}
\label{app:lemma2}
    Let $\| \hat{\bm M}_t - \bm M_t\|_*$ be the compression error of the low-rank momentum estimation at iteration $t$. Then under Assumption~\ref{assmp-main:1} and the condition $\operatorname{rank}(\bm G_0) \le r$, we have the following bound:  
    \begin{equation}
\label{eqn:lemma2}
\sum_{t=1}^T \mbb E \bigl [\| \hat{\bm M}_t - \bm M_t\|_* \bigr] \le    \frac{\eta L T}{1 - \beta}  + \frac{2  \sigma T}{\sqrt{B} (1 - \beta)} 
\end{equation}
\end{lemma}
\begin{proof}[Proof of Lemma~\ref{app:lemma2}]
First note that we have: \begin{equation}
    \| \hat{\bm M}_t - \bm M_t\|_* = \| \beta \bm M_{t-1} + \bm G_t - \hat{\bm M}_t\|_* \le \beta \| \hat{\bm M}_{t-1} - \bm M_{t-1} \|_* + \| \bm G_t + \beta \hat{\bm M}_{t-1} - \hat{\bm M}_t \|_*
\end{equation}
Considering the momentum factor update rule in~\ref{alg:mofasgd_complete}, we can replace $\hat{\bm M}_t$ with $\hat{\bm G}_t + \beta \hat{\bm M}_{t-1}$, where $\hat{\bm G}_t = \bm U_t \bm U_t^\top \bm G_t + \bm G_t \bm V_t \bm V_t^\top - \bm U_t \bm U_t^\top \bm G_t \bm V_t \bm V_t^\top $. Therefore, we can write:  
\begin{equation}
\label{eqn:app-dp9}
    \| \hat{\bm M}_t - \bm M_t\|_* \le \beta \| \hat{\bm M}_{t-1} - \bm M_{t-1} \|_* + \| ( \bm I  - \bm U_t \bm U_t^\top) \bm G_t (\bm I - \bm V_t \bm V_t^\top) \|_* \end{equation}

The term $ \| ( \bm I   - \bm U_t \bm U_t^\top) \bm G_t (\bm I - \bm V_t \bm V_t^\top) \|_*$ represents the low-rank compression error of gradients happening in our Algorithm~\ref{alg:mofasgd_complete}. Thus, we aim to bound this term.  We have:  
\begin{equation}
\label{eqn:app-dp5}
\begin{split}
     &\| ( \bm I  - \bm U_t \bm U_t^\top) \bm G_t (\bm I - \bm V_t \bm V_t^\top) \|_* \le \| ( \bm I  - \bm U_t \bm U_t^\top) (\bm G_t - \bm G_{t-1})  (\bm I - \bm V_t \bm V_t^\top) \|_* + \| ( \bm I  - \bm U_t \bm U_t^\top) \bm G_{t-1} (\bm I - \bm V_t \bm V_t^\top) \|_* \\
      &\le \| ( \bm I  - \bm U_t \bm U_t^\top)( \Bar{\bm G}_t - \Bar{\bm G}_{t-1})(\bm I - \bm V_t \bm V_t^\top)\|_* +  \| ( \bm I  - \bm U_t \bm U_t^\top) \bm G_{t-1} (\bm I - \bm V_t \bm V_t^\top) \|_* + \| \bm \Delta_{t} - \bm \Delta_{t-1} \|_* \\
      &\le  L \| \bm W_{t} - \bm W_{t-1}\|_2  + \| ( \bm I  - \bm U_t \bm U_t^\top) \bm G_{t-1} (\bm I - \bm V_t \bm V_t^\top) \|_* + \| \bm \Delta_{t} - \bm \Delta_{t-1} \|_*
      \\
      &\le  \eta L + \| ( \bm I  - \bm U_t \bm U_t^\top) \bm G_{t-1} (\bm I - \bm V_t \bm V_t^\top) \|_* + \| \bm \Delta_{t} - \bm \Delta_{t-1} \|_*
\end{split}
\end{equation}
where in the last inequality, we used the smoothness property in Assumption~\ref{assmp-main:1}, and also the fact that $\bm W_t - \bm W_{t-1} = - \eta \bm U_t \bm V_t^\top$.

Next, we bound the term $ \| ( \bm I   - \bm U_t \bm U_t^\top) \bm G_{t-1} (\bm I - \bm V_t \bm V_t^\top) \|_*$. First note that based on Equation \ref{eqn:moment-update-1}, we can see that $\operatorname{Range}([ \bm U_{t-1}    \quad \bm G_{t-1} \bm V_{t-1}]) \subseteq \operatorname{Range}(\bm U_t)$, and therefore we have $\operatorname{Range}(\bm U_{t-1}) \subseteq \operatorname{Range}(\bm U_t)$. We can write:
\begin{equation}
\label{eqn:app-dp6}
\begin{split}
    & \| ( \bm I  - \bm U_t \bm U_t^\top) \bm G_{t-1} (\bm I - \bm V_t \bm V_t^\top) \|_* \le \| ( \bm I  - \bm U_t \bm U_t^\top) (\bm G_{t-1}- \hat{\bm G}_{t-1} ) (\bm I - \bm V_t \bm V_t^\top) \|_* \\
    & \quad+ \| ( \bm I  - \bm U_t \bm U_t^\top)  \hat{\bm G}_{t-1}  (\bm I - \bm V_t \bm V_t^\top) \|_* \\
    \quad & \le \| ( \bm I  - \bm U_t \bm U_t^\top) (\bm G_{t-1}- \hat{\bm G}_{t-1}) (\bm I - \bm V_t \bm V_t^\top) \|_* \\
    \quad & \le \|  ( \bm I  - \bm U_{t-1} \bm U_{t-1}^\top)\bm G_{t-1} (\bm I - \bm V_{t-1} \bm V_{t-1}^\top) \|_* \end{split}
\end{equation}
where in the last inequality we leveraged the bound from Theorem~\ref{thm:main-1} that upper bounds the tangent space projection error. Particularly, the fact that $\| \bm G_{t-1} - \hat{\bm G}_{t-1} \|_* = \| ( \bm I   - \bm U_{t-1} \bm U_{t-1}^\top) \bm G_{t-1} (\bm I - \bm V_{t-1} \bm V_{t-1}^\top)\| $. 
Plugging Equation~\ref{eqn:app-dp6} into Equation~\ref{eqn:app-dp5}, we have the following recursive equation:  
\begin{equation}
    \| ( \bm I  - \bm U_t \bm U_t^\top) \bm G_t (\bm I - \bm V_t \bm V_t^\top) \|_* \le \|  ( \bm I  - \bm U_{t-1} \bm U_{t-1}^\top)\bm G_{t-1} (\bm I - \bm V_{t-1} \bm V_{t-1}^\top) \|_* + \eta L +     \| \bm \Delta_{t} - \bm \Delta_{t-1} \|_* \end{equation}
Unrolling this recursive relation, and considering the SVD initialization at the beginning of Algorithm~\ref{alg:mofasgd_complete}, we have:  
\begin{equation}
\label{eqn:app-dp7}
    \| ( \bm I  - \bm U_t \bm U_t^\top) \bm G_t (\bm I - \bm V_t \bm V_t^\top) \|_* \le \| ( I - \bm U_0 \bm U_0^\top) \bm G_0 (I - \bm V_0 \bm V_0^\top) \|_* + \eta L t + \sum_{i=1}^t \| \bm \Delta_{t} - \bm \Delta_{t-1}\|_* \end{equation}
Let $e_0 = \| ( I - \bm U_0 \bm U_0^\top) \bm G_0 (I - \bm V_0 \bm V_0^\top) \|_*$. Note that by our initialization of $\bm U_0$ and $\bm V_0$, $e_0$ is expected to be small. For instance, if we assume $\operatorname{rank}(\bm G_0) \le r$, then we would have $e_0 = 0$. Therefore, for simplicity we can ignore this term.    

Plugging Equation~\ref{eqn:app-dp7} back into Equation~\ref{eqn:app-dp9} we have:  
\begin{equation}
    \mbb E \bigl [ \| \hat{\bm M}_t - \bm M_t\|_* \bigr] \le \beta \mbb E \bigl [ \| \hat{\bm M}_{t-1} - \bm M_{t-1} \|_* \bigr ] + \eta L t + \frac{2 t\sigma }{\sqrt{B}}
\end{equation}

Similar to Equation~\ref{eqn:app-dp12}, we can unroll the above equation to get:  
\begin{equation}
\label{eqn:app-dp18}
\sum_{t=1}^T \mbb E \bigl [\| \hat{\bm M}_t - \bm M_t\|_* \bigr] \le    \frac{\eta L T}{1 - \beta}  + \frac{2  \sigma T}{\sqrt{B} (1 - \beta)} 
\end{equation}
where we also used the fact that $ \hat{ \bm M}_0 = \bm M_0$.  
\end{proof}

\begin{proof}[Proof of Theorem~\ref{thm:main-conv}]
Taking the expectation and unrolling the descent lemma, Lemma~\ref{lemma:des}, we can write:  
\begin{equation}
\label{eqn:app-dp19}
\begin{split}
    \frac{1}{T} \sum_{t=0}^T \mbb E [ \|\Bar{\bm G}_t\|_*]  &\le    \frac{ \mc L(\bm W_0) - \mbb E [\mc L(\bm W_{T+1})]}{\eta T} +   \frac{2}{T} \sum_{t=0}^T\mbb E [ \| \hat{\bm M}_t - \Bar{\bm G}_t \|_*] + \frac{\eta^2 L}{2} \\
     & \le \frac{ \mc L(\bm W_0) - \mbb E [\mc L(\bm W_{T+1})]}{\eta T} + \frac{\eta^2 L}{2} +    \frac{2}{T} \sum_{t=0}^T\mbb E [ \| \hat{\bm M}_t - \bm M_t\|] + \frac{2}{T} \sum_{t=0}^T \mbb E [ \| \bm M_t - \Bar{\bm G}_t \|_*]  
\end{split}
\end{equation}
Plugging Equation~\ref{eqn:app-dp17} (Lemma~\ref{app:lemma1}) and Equation~\ref{eqn:app-dp18} (Lemma~\ref{app:lemma2}) into Equation~\ref{eqn:app-dp19} yields
\begin{equation}
    (1- \frac{2 \beta}{1 - \beta}) \frac{1}{T} \sum_{t=0}^T \mbb E [ \|\Bar{\bm G}_t\|_*] \le \frac{ \mc L(\bm W_0) - \mbb E [\mc L(\bm W_{T+1})]}{\eta T} + \frac{\eta^2 L}{2} + \frac{4 \sigma}{(1-\beta) \sqrt{B}} + \frac{2\eta L}{1 - \beta}
\end{equation}
Let $\beta \in (0,\frac{1}{3})$, $\eta < 1$, and $B = T$; then we can rewrite the above equation as follows:
\begin{equation}
    \frac{1}{T} \sum_{t=0}^T \mbb E [ \|\Bar{\bm G}_t\|_*] \le \mc O \left ( \frac{ \mc L(\bm W_0) - \mbb E [\mc L(\bm W_{T+1})]}{\eta T} +  \eta L + \frac{ \sigma}{\sqrt{T}} \right) 
\end{equation}

Finally, by letting $ \eta = \Theta \left(\sqrt{\frac{\mc L(\bm W_0) - \mbb E [\mc L(\bm W_{T+1})]}{T L}}\right)$, we have the final upper bound as follows: 
\begin{equation}
        \frac{1}{T} \sum_{t=0}^T \mbb E [ \|\Bar{\bm G}_t\|_*] \le \mc O \left ( \frac{\bigl(\mc L(\bm W_0) - \mbb E [\mc L(\bm W_{T+1})]\bigr)^{\frac{1}{2}}\sqrt{L} + \sigma}{\sqrt{T}} \right). 
\end{equation}
\end{proof}

\end{document}